\documentclass[dvipsnames]{article}

\PassOptionsToPackage{numbers, compress}{natbib}
\usepackage{sty/iclr2027/iclr2027_conference,times}
\iclrfinalcopy

\usepackage[utf8]{inputenc} % allow utf-8 input
\usepackage[T1]{fontenc}    % use 8-bit T1 fonts
\usepackage{hyperref}       % hyperlinks
\usepackage{url}            % simple URL typesetting
\usepackage{booktabs}       % professional-quality tables
\usepackage{amsfonts}       % blackboard math symbols
\usepackage{nicefrac}       % compact symbols for 1/2, etc.
\usepackage{microtype}      % microtypography
\usepackage[table]{xcolor}         % colors

\usepackage{subfigure} % subfigure
\usepackage{multirow}
\usepackage{makecell}
\usepackage{algorithm}
\usepackage[noend]{algpseudocode}
\usepackage{seqsplit}
\usepackage{minibox}
\usepackage{booktabs}       % professional-quality tables
\usepackage{amsfonts}       % blackboard math symbols
\usepackage{nicefrac}       % compact symbols for 1/2, etc.
\usepackage{amssymb}
\usepackage{enumitem} % leftmargin
\usepackage{wrapfig}
\usepackage{xspace}

\usepackage[normalem]{ulem}

\usepackage{tikz}
\usetikzlibrary{shapes,backgrounds,arrows,fit,calc,patterns}

\newcommand{\ie}{\emph{i.e.,}~}
\newcommand{\eg}{\emph{e.g.,}~}

\usepackage{pifont}
\newcommand{\cmark}{{\color{ForestGreen}\ding{51}}} % check mark
\newcommand{\xmark}{{\color{red}\ding{55}}} % x mark

\newcommand{\para}[1]{\textbf{#1}}
\usepackage{amsmath} % http://ctan.org/pkg/amsmath
\usepackage{bm} % for bold symbols like \bm{\alpha} or {\bm \alpha\beta}
\usepackage{bbm} % for mathbbm
\usepackage{amsthm}
\usepackage{mathtools}

\newcommand{\ep}{\varepsilon}
\renewcommand{\epsilon}{\ep}

\makeatletter
\@ifundefined{proposition}{%
    \newtheorem{proposition}{Proposition}
}{}
\@ifundefined{definition}{%
    \newtheorem{definition}{Definition}%
}{}
\@ifundefined{lemma}{%
    \newtheorem{lemma}{Lemma}
}{}
\@ifundefined{corollary}{%
    \newtheorem{corollary}{Corollary}
}{}
\@ifundefined{theorem}{%
    
}{}
\@ifundefined{corollary}{%
    
}{}
\@ifundefined{assumption}{%
    
}{}
\@ifundefined{problem}{%
    
}{}
\@ifundefined{problem*}{%
    \newtheorem*{problem*}{Problem}
}{}
\@ifundefined{claim}{%
    
}{}
\@ifundefined{example}{%
    
}{}
\@ifundefined{implication}{%
    
}{}
\@ifundefined{remark*}{%
    \newtheorem*{remark*}{Remark}
}{}
\makeatother

\newcommand{\argmin}{\operatorname*{arg\,min}}
\newcommand{\argmax}{\operatorname*{arg\,max}}

\newtheorem{innercustomgeneric}{\customgenericname}
\providecommand{\customgenericname}{}
\newcommand{\newcustomtheorem}[2]{%
  \newenvironment{#1}[1]
  {%
   \renewcommand\customgenericname{#2}%
   \renewcommand\theinnercustomgeneric{##1}%
   \innercustomgeneric
  }
  {\endinnercustomgeneric}
}
\newcustomtheorem{customclaim}{Claim}

\providecommand{\realnum}					{\mathbb{R}}

\renewcommand{\]}						{\right]}

\providecommand{\Exp}{\mathbbm{E}}

\def\e{\mathbf{e}}

\def\r{\mathbf{r}}

\def\w{\mathbf{w}}
\def\x{\mathbf{x}}
\def\y{\mathbf{y}}

\def\J{\mathbf{J}}

\def\gh{\hat{{g}}}

\def\As{\mathcal{{A}}}

\def\Ds{\mathcal{{D}}}

\def\Ps{\mathcal{{P}}}

\def\Vs{\mathcal{{V}}}
\def\Ws{\mathcal{{W}}}
\def\Xs{\mathcal{{X}}}
\def\Ys{\mathcal{{Y}}}

\def\alphaB{{\bm{\alpha}}}

\def\thetaB{{\bm{\theta}}}

\def\muB{{\bm{\mu}}}


\newcommand{\ours}{\texttt{MAESTRO}\xspace}

\newcommand{\std}[1]{{\scriptsize\color{gray}$\pm#1$}}
\title{Which Preferences to Train On? End-to-End Multi-Objective Alignment with an Adversarial Preference Distribution}
\author{%
  Minjae Lee
  \\
  GSAI
  \\
  POSTECH
  \\
  \texttt{minjae.lee@postech.ac.kr} \\
  \And
  Kyunghyun Cho
  \\
  Center for Data Science
  \\
  New York University
  \\
  \texttt{kyunghyun.cho@nyu.edu} \\
  \And
  Sangdon Park
  \\
  GSAI \& CSE
  \\
  POSTECH
  \\
  \texttt{sangdon@postech.ac.kr} 
}
\begin{document}

\maketitle
\lhead{Preprint} 
% \begin{NoHyper}
% \def\thefootnote{*}\footnotetext{Equal contribution}
% \end{NoHyper}
%\def\thefootnote{\arabic{footnote}}

\begin{abstract}
Aligning large language models (LLMs) with human values is important for safe, efficient, and beneficial AI deployment. However, human values are multifaceted: helpfulness, harmlessness and humor trade off against one another, and different users want different trade-offs.
Multi-objective alignment (MOA) addresses this by training a policy that can provide any point of the Pareto front, but existing methods either train one model per preference, interpolate a few separately aligned experts post hoc, or train a single conditioned model without considering which preferences it should be trained on.
Since the hard regions of the preference simplex depend on the objectives at hand, existing methods leave them under-trained and do not get the most out of a single model.
% Therefore, we propose \ours (\textbf{M}ulti-objective \textbf{A}lignment via \textbf{E}nd-to-end \textbf{ST}eering and \textbf{R}obust \textbf{O}ptimization), which trains a single prompt-conditioned policy end-to-end with RL against an \emph{adversarial} preference distribution: a Dirichlet distribution updated by online mirror descent toward the preferences the current policy serves worst, rather than on a fixed distribution.
Therefore, we propose \ours (\textbf{M}ulti-objective \textbf{A}lignment via \textbf{E}nd-to-end \textbf{ST}eering and \textbf{R}obust \textbf{O}ptimization), which formulates MOA as a minimax problem over preference distributions and trains a single prompt-conditioned policy end-to-end with RL against an \emph{adversarial} preference distribution: a Dirichlet distribution updated by online mirror descent toward the preferences the current policy serves worst, rather than on a fixed one.
% Original: On HH-RLHF, BeaverTails and summarization, with up to three objectives, \ours attains the best Pareto front on most tasks in a single training run, at the lowest training cost among all methods.
On HH-RLHF, BeaverTails and a summarization task, with up to three objectives, \ours attains the best Pareto front on most tasks in a single training run, at the lowest training cost among the compared methods.
The largest margins appear in the hard regions that a fixed preference distribution leaves under-trained, confirming that a single prompt-conditioned model is capable of covering the objective trade-offs on its own.
\end{abstract}

\vspace{-1.5ex}
\section{Introduction}
\vspace{-1.5ex}
\label{sec:introduction}

Aligning large language models (LLMs) with human values is a central challenge for their safe and beneficial deployment.
Reinforcement learning from human feedback (RLHF) \citep{ouyang2022training}, \eg offline methods such as DPO \citep{rafailov2023direct} and online policy-gradient methods such as GRPO \citep{shao2024deepseekmath}, has emerged as the dominant paradigm for this purpose.
% However, human values are inherently multifaceted, \eg helpfulness, harmlessness, humor, and beyond, and they conflict,
% Optimizing a single reward signal fails to capture their full complexity,
However, human values are inherently multifaceted, \eg helpfulness, harmlessness, humor, and beyond, and they conflict \citep{bai2022training}, so optimizing a single reward signal fails to capture their full complexity.
This motivates \emph{Multi-Objective Alignment} (MOA) \citep{li2020deep,zhong2024panacea,zhou2024beyond,yang2024rewards,xie2025bone,shi2024decoding,li2025self}, where multiple objectives represent different facets of human value,
% and the goal is to learn a policy that provides a set of solutions with different preference trade-offs.
% Since providing all possible solutions is infeasible, the ideal solution is a policy that can be steered to any preference vector at inference time, tracing the entire \emph{Pareto front}---the set of trade-offs from which no objective can be improved without sacrificing another.
% Steering here means only that changing the stated preference moves the policy's returns along that front in the intended direction, as in prior multi-objective work \citep{wang2024conditional,dong2023steerlm}; it involves no intervention on activations or weights.
and the goal is to train a policy that can be \emph{steered}\footnote{As in prior multi-objective work \citep{wang2024conditional,dong2023steerlm}, \emph{steering} here is conditioning on a stated preference, not intervening on activations or weights.} at inference time to any preference over objectives, tracing the \emph{Pareto front}---the set of trade-offs from which no objective can be improved without sacrificing another.
% \MJ{need?}
% Original: See Table~\ref{tab:qualitative} for example.
See Table~\ref{tab:qualitative} for an example.

MOA inherits a problem that predates language models: multi-objective optimization \citep{miettinen1999nonlinear} and its sequential-decision form, multi-objective reinforcement learning (MORL) \citep{white1982multi,hayes2022practical}, already separate a policy for a known preference from a coverage set spanning all of them.
Linear scalarization carries this literature into RLHF as \emph{multi-objective RLHF} (MORLHF) \citep{li2020deep}, collapsing the reward vector into a single scalar reward that RLHF then optimizes \citep{wu2023fine}, a formulation later work adopts as the reference for MOA \citep{zhou2024beyond,rame2023rewarded}.
It therefore requires a separate model for every preference vector, which is computationally prohibitive and intractable for LLMs.
% Post-hoc interpolation methods \citep{zhou2024beyond,rame2023rewarded,xie2025bone,shi2024decoding} train $M$ backbone experts and leverage them at inference time, but they are not able to capture intermediate preference regions, preventing coverage of the full Pareto front.
Post-hoc interpolation methods \citep{zhou2024beyond,rame2023rewarded,xie2025bone,shi2024decoding} train $M$ backbone experts and leverage them at inference time, but they rely on an assumed linearity in weight space \citep{rame2023rewarded,xie2025bone} or in logit space \citep{shi2024decoding}, rather than on optimization over the preference space itself.
% Moved to rel.tex (Single-Model Approaches): {\color{red}Conditioning one model on the preference vector, the remaining option, is itself an old idea in MORL, where the weight enters the value function or the approximator's input \citep{castelletti2011multi} and, later, deep MORL and hypernetwork amortizations of the whole front \citep{yang2019generalized,navon2021learning}; our prompt prefix is the language-model form of that input conditioning.}
% Single-model approaches \citep{zhong2024panacea,yang2024rewards,dong2023steerlm} enable zero-shot adaptation via a single model, but often require a specific architecture \citep{zhong2024panacea}, or rely on preference-to-reward mappings \citep{yang2024rewards,dong2023steerlm} that are not jointly optimized with the policy.
% {\color{red}Such mappings also presuppose the range of each reward: RiC \citep{yang2024rewards} needs per-objective bounds to turn a preference into target reward tokens, and MO-ODPO \citep{gupta2025robust} likewise rescales each reward to $[0,1]$ before weighting, so both must be recalibrated whenever an objective or a reward model changes.}
% Single-model approaches \citep{zhong2024panacea,yang2024rewards,dong2023steerlm} enable zero-shot adaptation via a single model, but often require a specific architecture \citep{zhong2024panacea}, or rely on preference-to-reward mappings \citep{yang2024rewards,dong2023steerlm} that presuppose the range of each reward.
Some approaches instead serve the whole simplex with one model: RiC \citep{yang2024rewards} and SteerLM \citep{dong2023steerlm} write reward values into the prompt and fit them by supervised fine-tuning, which presupposes the range of each reward.

\begin{wrapfigure}{r}{0.5\textwidth}
    \centering
    \vspace{-1.2em}
    \subfigure{
        \centering
        \includegraphics[width=0.23\textwidth, clip, trim=0.5cm 0.5cm 0.3cm 0.3cm]{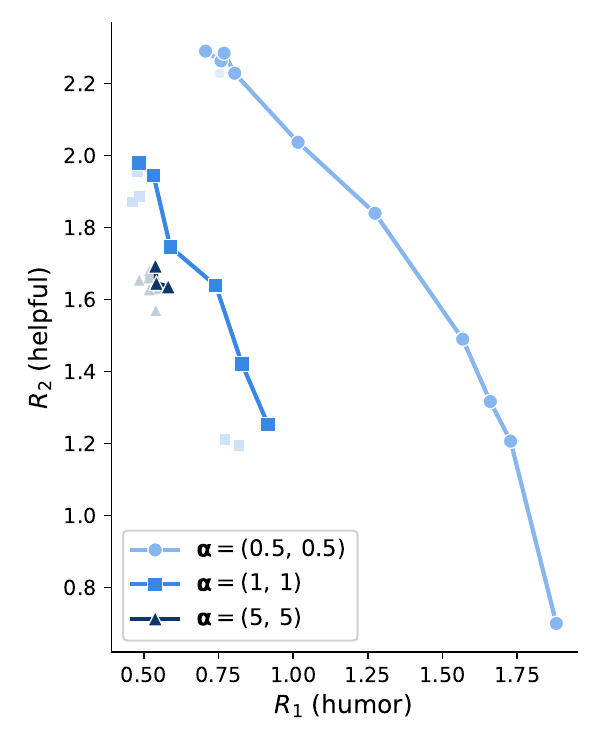}
        % \vspace{-5em}
    }
    % \hfill
    \hspace{-2ex}
    \subfigure{
        \centering
        \includegraphics[width=0.23\textwidth, clip, trim=0.5cm 0.5cm 0.3cm 0.3cm]{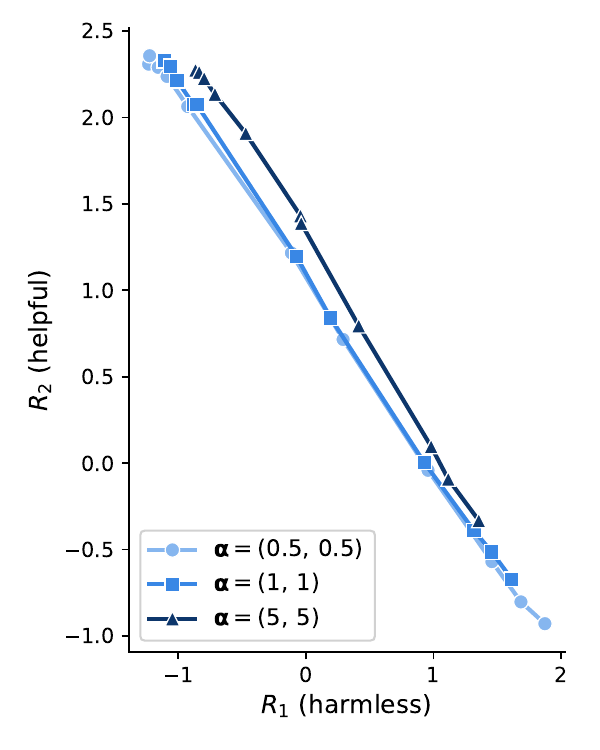}
    }
    \vspace{-1.8ex}
    \caption{
        {Pareto fronts under preference sampling $\w\sim\mathrm{Dir}(\alphaB)$. Center-concentrated sampling ($\alphaB = (5,5)$) suffices on Harmless--Helpful but collapses the front on Humor--Helpful, where edge sampling ($\alphaB = (0.5,0.5)$) recovers it.}
        % \MJ{without sys}
        % \textbf{(Humor--Helpful)}: center-concentrated sampling ($\alpha{=}5$) severely degrades the front, while edged sampling ($\alpha{=}0.5$) recovers it.
        % \textbf{(Harmless--Helpful)}: the exactly opposite pattern holds.
        % \MJ{TODO: change color, wrap, remove number.}
    }
    \label{fig:motivation_sampling}
    \vspace{-2ex}
\end{wrapfigure}
Panacea \citep{zhong2024panacea} and MO-ODPO \citep{gupta2025robust} also train a single model but on the preferences end-to-end, drawing them from a distribution over the simplex.
However, the distribution is fixed in advance, so it leaves the hardest preference regions under-trained, and which regions are hard is not known before training; prompt-conditioned policies are also reported to be sensitive to this choice \citep{gupta2025robust,wang2024conditional}.
% Figure~\ref{fig:motivation_sampling} varies it, sampling $\w$ from the symmetric Dirichlet distribution over the simplex $\mathrm{Dir}(\alphaB)$, whose concentration $\alpha$ places mass near the vertices when $\alpha<1$ and near the center when $\alpha>1$; the effect even reverses between objective pairs, and the hard region moves as the policy improves, so no fixed choice can track it; % Moved to the contributions paragraph below: \ours therefore replaces the fixed distribution with an adversarial one, as in distributionally robust optimization \citep{nemirovski2009robust,sagawa2020distributionally} and in minimax treatments of multi-objective problems whose worst-case solution is itself a Pareto point \citep{martinez2020minimax,park2024the}.
Figure~\ref{fig:motivation_sampling} varies it, sampling $\w$ from the symmetric Dirichlet distribution over the simplex $\mathrm{Dir}(\alphaB)$ with $\alphaB=\alpha\mathbf{1}$,
which places mass near the vertices when $\alpha<1$, the center when $\alpha>1$, and uniformly when $\alpha=1$.
Which concentration works changes between objective pairs, and the hard region would move as the policy improves, so no fixed choice can track it.
% Moved to the contributions paragraph below: \ours therefore replaces the fixed distribution with an adversarial one, as in distributionally robust optimization \citep{nemirovski2009robust,sagawa2020distributionally} and in minimax treatments of multi-objective problems whose worst-case solution is itself a Pareto point \citep{martinez2020minimax,park2024the}.
% }
These gaps raise our central research question: 
\emph{"How can a single preference-conditioned LLM improve Pareto front coverage and robustness across the preference simplex?"}

To mitigate this, we propose \ours, a novel framework for \textbf{M}ulti-objective \textbf{A}lignment via \textbf{E}nd-to-end \textbf{ST}eering and \textbf{R}obust \textbf{O}ptimization.
\ours conditions the policy on the preference vector by writing it into the prompt and trains it end-to-end with RL on the linearly scalarized reward, which requires no extra architecture, no model per preference and no prescribed reward ranges.
% Then, we formulate MOA as a minimax optimization over preference distributions, as in distributionally robust optimization \citep{nemirovski2009robust,sagawa2020distributionally} and in minimax treatments of multi-objective problems whose worst-case solution is itself a Pareto point \citep{martinez2020minimax,park2024the}: rather than fixing one in advance, \ours trains against an adversary, a Dirichlet distribution updated by online mirror descent that shifts training toward the preferences the current policy serves worst, so that a single run optimizes a smoothed worst case over the simplex rather than an average.
Then, we formulate MOA as a minimax optimization over preference distributions \citep{martinez2020minimax,park2024the}, and train the policy against an adversary, a Dirichlet distribution updated by online mirror descent toward the preferences it serves worst, rather than fixing the distribution in advance.
% \MJ{The wording above deliberately distinguishes the ideal minimax payoff from the implemented per-objective rank surrogate; Appendix~\ref{apdx:layer3} already calls the latter a heuristic.}
% , \ours trains against an adversary, a Dirichlet distribution updated by online mirror descent that shifts training toward the preferences the current policy serves worst, so that a single run optimizes a smoothed worst case over the simplex rather than an average.
% Since the adversary's gradient multiplies the reward, rewards on different scales would let one objective dominate its update, so the adversary is given a per-objective rank statistic of the rewards while the policy is still trained on the raw reward.
% Empirically, on HH-RLHF, BeaverTails and a summarization task, \ours leaves the other single-model methods well behind and matches or exceeds the methods that train and keep several models---one policy per preference, or $M$ merged experts---at $8.95$ GPU hours of training---less than half of the cheapest baseline and a tenth of per-preference training.
Finally, we support this formulation with analysis (Appendix~\ref{apdx:theory}) and experiments: on HH-RLHF, BeaverTails and a summarization task, \ours leaves the other single-model methods well behind and matches or exceeds the methods that train and keep several models---one policy per preference, or $M$ merged experts---at $8.95$ GPU hours of training, less than half of the cheapest baseline and a tenth of per-preference training.

\vspace{-1.5ex}
\subsection{Related Work}\label{sec:rel}
\vspace{-1.5ex}

Here, we introduce closely related work on MOA.
See Appendix~\ref{apdx:rel} for additional related work.

\para{Multi-Objective Alignment.}
% MORLHF \citep{li2020deep,wu2023fine} trains one RL policy per preference vector by linearly scalarizing the rewards, .
MORLHF \citep{li2020deep,wu2023fine} trains one RL policy per preference vector by linearly scalarizing the rewards, which needs $N \gg M$ runs to cover the simplex.
% Soup-based methods \citep{rame2023rewarded,jang2023personalized,xie2025bone} and decoding-time mixtures \citep{shi2024decoding} reduce this to $M$ single-objective experts combined at inference, trading training cost for a family of solutions that is linear in the experts' weights or logits; Bone Soups widens that family by extrapolation, at the price of a coefficient selected after training.
Soup-based methods cut this to one expert per objective, keeping $M$ models in memory, and reach the other trade-offs by interpolation, \eg Rewarded Soups \citep{rame2023rewarded} merges the weights of the $M$ single-objective experts in the proportions given by the preference vector, while Bone Soups \citep{xie2025bone} first trains its experts under skewed reward combinations and then turns the preference vector into merging coefficients, so that the merge reaches solutions beyond the experts themselves.
% Original: MOD \citep{shi2024decoding} interpolates in logit space instead, which it derives as an approximation of the scalarized optimum, combining the experts' next-token distributions at inference and so forward passing all $M$ models for each step.
MOD \citep{shi2024decoding} interpolates in logit space instead, which it derives as an approximation of the scalarized optimum, thereby requiring a forward pass through all $M$ models at every decoding step.
Both are far cheaper to train than one policy per preference, but the intermediate trade-offs are produced by a combination rule at inference rather than optimized during training.
By contrast, \ours trains a single prompt-conditioned policy on those preferences directly, so they are optimized rather than assembled.
% Moved to Appendix~\ref{apdx:rel}: Outside language modeling, conditioning one policy on the preference is not the default either---C-MORL \citep{liu2024c} trains a population of policies instead, on the grounds that conditioning scales poorly as the state and preference dimensions grow---but an instruction-following LLM already turns a preference stated in the prompt into a change in behavior.
% {\color{red}Multi-objective optimization and MORL predate all of these; Appendix~\ref{apdx:rel} summarizes that lineage.}

\para{Single-Model Approaches.}
% Moved to Appendix~\ref{apdx:rel}: Conditioning one policy on the preference is an old idea in MORL, where the weight enters the value function or the approximator's input \citep{castelletti2011multi} and is later amortized over the whole front by deep MORL and hypernetworks \citep{yang2019generalized,navon2021learning}.
% Original: Single-model methods show that one policy can serve the whole front at no extra inference cost; they differ mainly in how the preference is conditioned on.
Single-model methods show that one policy can serve the whole front without storing or executing a separate policy for each preference; they differ mainly in how the preference enters the model.
% RiC \citep{yang2024rewards} and CPO \citep{guo2024controllable} write a per-objective target \emph{value} into the prompt rather than the preference vectors themselves, which presupposes each reward's range $(r^{(m)}_{\min}, r^{(m)}_{\max})$, as does attribute-conditioned generation \citep{dong2023steerlm};
RiC \citep{yang2024rewards} writes a per-objective target \emph{value} into the prompt rather than the preference vector itself, which presupposes each reward's range $(r^{(m)}_{\min}, r^{(m)}_{\max})$, as does attribute-conditioned generation \citep{dong2023steerlm}.
% Panacea \citep{zhong2024panacea} and CLP \citep{wang2024conditional} write the preference vectors into the architecture, as singular values of every SVD-LoRA layer or as a convex combination of parameter replicas, and so keep per-objective parameters at inference;
Panacea \citep{zhong2024panacea} writes them into the architecture, as singular values of every SVD-LoRA layer, and so keeps per-objective parameters at inference.
MO-ODPO \citep{gupta2025robust} and MOC \citep{he2026one} % and DPA \citep{wang2024arithmetic} 
write the preference vectors into the prompt, as a text prefix or as special tokens, which is the interface \ours adopts.
% Covered by prob.tex (Existing Approaches / Objective): They also differ in what they optimize (Table~\ref{tab:comparison}): MORLHF maximizes the scalarized return at one preference at a time, Panacea maximizes it in expectation over the simplex, and the value-conditioned methods instead fit reward-labeled data by supervised fine-tuning.
Panacea \citep{zhong2024panacea} is the one that optimizes an objective defined over the simplex, though under a distribution fixed in advance, while the others inherit the preferences they are trained on from their data, \eg RiC does from its relabeled rewards, or from what their own online sampling happens to reach, as MO-ODPO does (Table~\ref{tab:comparison}).
In all of these, $\w$ is drawn from a distribution fixed before training and independent of the policy: Panacea samples the simplex uniformly with the vertices up-weighted, MO-ODPO samples $\w\sim\mathrm{Dir}(\alpha)$ with $\alpha$ tuned by grid search.
% Shortened; detail moved to Appendix~\ref{apdx:rel}: Preference Orchestrator \citep{liu2025preference} predicts a per-prompt weight instead of sampling, but a point estimate covers no region of the simplex, and RMOD \citep{son2026robust} solves a two-player game over $\w$ at \emph{decoding} time against a frozen policy, returning one adversarial weight vector per query in order to avoid requiring a stated preference at all.
RMOD \citep{son2026robust} does solve a game over $\w$ at \emph{decoding} time against a frozen policy, but it does not change the preferences the policy is trained on.
% Original: \ours instead makes that distribution the minimizing player of a minimax objective, updated jointly with the policy, so the preferences it trains on are chosen during training rather than in advance.
% Hedged: \ours instead adapts that distribution during training with a minimax-motivated, batch-relative update, so the sampled preferences are chosen online rather than fixed in advance.
\ours instead makes that distribution the minimizing player of a minimax objective and updates it jointly with the policy, so the preferences it trains on depend on the current policy rather than being fixed in advance.

\begin{table*}[t]
    \centering
    \caption{Comparison of existing methods on multi-objective alignment.
    In practice, the number of preferences $N$ is larger than the number of objectives $M$ ($N \gg M$).}
    \label{tab:comparison}
    \def\arraystretch{1}
    \resizebox{0.95\textwidth}{!}{
    \begin{tabular}{c||c|c|c|c||c}
         & \makecell{\# Models} & \makecell{Prompt-based\\Conditioning} & \makecell{Preference\\Distribution $p$} & Requirement & \makecell{Objective} \\
        \midrule
        \midrule
        \makecell{MORLHF\\ \citep{li2020deep,wu2023fine}} & $N$ & \xmark & \xmark & - & {$\max J_\w$ per $\w$} \\
        \midrule
        \makecell{Rewarded Soup \citep{rame2023rewarded},\\ Bone Soup \citep{xie2025bone}} & $M$ & \xmark & \xmark & \makecell{same architecture\\\& initialization} & {$\max J^{(m)}$ per expert} \\
        \midrule
        \makecell{Multi-Objective Decoding\\ \citep{shi2024decoding}} & $M$ & \xmark & \xmark & same vocab. & {$\max J^{(m)}$ per expert} \\
        \midrule
        \makecell{Rewards-in-Context\\ \citep{yang2024rewards}} & $1$ & \cmark & fixed (rewards data) & \makecell{preference-to-reward\\mapping function} & {SFT on relabeled data} \\
        \midrule
        \makecell{MO-ODPO\\\citep{gupta2025robust}} & $1$ & \cmark & fixed & {\makecell{reward bounds}} & {DPO on sampled pairs} \\
        % \midrule
        % \makecell{HoE \citep{li2025multi}} & $\ge M$ & \tmark & \xmark & \makecell{additional routers \& LoRA experts} \\
        \midrule
        % Previous marked revision (not retained pending the model-count convention): \makecell{Panacea\\ \citep{zhong2024panacea}} & {\color{red}$1$} & \xmark & fixed & \makecell{SVD-LoRA} & {$\max \Exp_{p}[J_\w]$} \\
        \makecell{Panacea\\ \citep{zhong2024panacea}} & $>1$ & \xmark & fixed & \makecell{SVD-LoRA} & {$\max \Exp_{p}[J_\w]$} \\
        \midrule
        \midrule
        \makecell{\ours\\(Ours)} & $1$ & \cmark & {learned (adversarial)} & - & {$\max\min_{p}\Exp_{p}[J_\w]$} \\
        \bottomrule
    \end{tabular}
    }
    \vspace{-1em}
\end{table*}
% \MJ{Table~\ref{tab:comparison}: define the model-count convention explicitly. Panacea uses one backbone/model with preference-conditioned SVD-LoRA directions, not multiple separately served policies.}

% \input{bg}
\vspace{-1ex}
% \section{Problem Formulation} \label{sec:problem}
\section{Multi-Objective Alignment with a Preference-Conditioned Policy}
\label{sec:problem}
\vspace{-1.5ex}

We consider the \emph{Multi-Objective Alignment} (MOA) problem, where the goal is to train a language model that adapts to arbitrary user-specified preferences over multiple objectives, effectively tracing the Pareto front (Appendix~\ref{sec:dis:pareto}).
% \MJ{The implemented percentage prefix quantizes $\w$ (and currently floors each component), so ``arbitrary user-specified preferences'' is exact only for an ideal continuous interface or after the quantization rule is made explicit.}
% {Since no single preference distribution serves every task, we state the goal as a minimax problem over such distributions~\eqref{eq:objective}: the policy is trained for the worst case over the simplex rather than for an average.}

\para{Notation.}
Let $\Xs$ be the input (prompt) space, $\Vs$ be the vocabulary, and $\Ds$ be a distribution over prompts $\x \in \Xs$.
A response $\y = (a_1, a_2, \dots, a_T) \in \Ys$ is a sequence of tokens $a_t \in \Vs$, where $\Ys = \bigcup_{t=1}^\infty \Vs^t$ is the output (response) space.
Let $M$ be the total number of objectives.
The \emph{reward function} $\r : \Xs \times \Ys \to \realnum^M$
% \MJ{$\realnum^M?$}
evaluates a completed response $\y$ on all $M$ objectives.
% Objectives are partitioned into \emph{tradable} objectives $\Ts \subseteq [M]$ and \emph{non-tradable} objectives $\Cs \subseteq [M]$, where $\Ts \cup \Cs = [M]$ and $\Ts \cap \Cs = \emptyset$; non-tradable objectives admit user preferences but must satisfy a hard lower-bound constraint.
Let $\Ws \coloneqq \Delta^{M-1} = \{\w \in \realnum_{\ge 0}^M \mid \mathbf{1}^\top \w = 1\}$ be the preference simplex, where each $\w \in \Ws$ encodes the user's relative preference over all $M$ objectives.

We model response generation as a \emph{token-level MDP}.
The \emph{state} at step $t$ is the context $s_t \coloneqq (\x, a_1, \dots, a_{t-1})$, with initial state $s_0 = \x$.
The \emph{action} at step $t$ is the next token $a_t \in \Vs$,
% \SP{better to define a transition kernel (though it is a trivial deterministic function)?}
where the transition is deterministic: $s_{t+1} = (s_t, a_t)$.
A \emph{preference-conditioned policy} $\pi_\theta(\cdot \mid s_t, \w)$, parameterized by $\theta$, maps a state $s_t$ and preference $\w$ to a distribution over the next token $a_t$.
% A \emph{trajectory} is $\tau \coloneqq (s_0, a_1, s_1, \dots, a_T, s_T)$, where the complete response $\y = (a_1, \dots, a_T)$
A complete response $\y = (a_1, \dots, a_T)$
% \SP{we are using $\y$ in the method section?}
is obtained by rolling out $\pi_\theta$ autoregressively; we write $\y \sim \pi_\theta(\cdot \mid \x, \w)$ as shorthand.
% We denote a fixed reference policy (\eg the supervised fine-tuned model) by $\pi_\text{ref}$, the KL-regularization coefficient by $\beta > 0$, and the threshold vector\SP{let's use desired value requirements (or similar); the threshold sounds like a tuning parameter} for non-tradable values by $\tauB \in [0,1]^{|\Cs|}$\SP{let's use $\alpha$; $\tau$ is often a trajectory in RL.} .
Since reward is assigned only upon completion of the full response, we denote $J^{(m)} (\pi_\theta) \coloneqq \Exp_{\x \sim \Ds,\, \y \sim \pi_\theta(\cdot \mid \x, \w)} \left[ r^{(m)}(\x, \y) \right]$
% \SP{$\r_1$ instead of $r_1$}
as the expected return of the $m$-th objective under preference $\w$,
% \SP{in our case $T=1$ or not?}
and $\J(\pi_\theta) \coloneqq \begin{bmatrix} J^{(1)}(\pi_\theta) \ \cdots \ J^{(M)}(\pi_\theta) \end{bmatrix}^\top$ as the vector of expected returns
% \SP{returns}
for all objectives under preference $\w$.
Then, we denote $J_\w(\pi_\theta) \coloneqq \w^\top \J(\pi_\theta)$ as the expected preference-weighted return of a policy $\pi_\theta$ under preference $\w$ by linear scalarization \citep{li2020deep}.
Both quantities depend on $\w$ through $\pi_\theta(\cdot \mid \w)$; we suppress this dependence in the notation.

\para{Existing Approaches.}
% A natural way to achieve MOA is to directly optimize for each preference vector $\w$ separately.
% MORLHF \citep{li2020deep} trains a \emph{preference-specific} policy $\pi_{\theta_\w}(\cdot \mid \x)$ for each fixed $\w$ by maximizing the linearly scalarized reward:
A natural approach to MOA is to optimize a separate policy for each target preference vector.
MORLHF \citep{li2020deep,wu2023fine} does so by training a preference-specific policy $\pi_{\theta_\w}(\cdot \mid \x)$ for each fixed $\w$:
\begin{equation}
    \max_{\theta_\w} J_\w(\pi_{\theta_\w}) \quad \text{for each } \w \in \Ws,
    % \max_{\theta_\w} \;\Exp_{\x \sim \Ds,\, \y \sim \pi_{\theta_\w}(\cdot|\x)} \!\bigl[\w^\top \r(\x, \y)\bigr] \quad \text{for each } \w \in \Ws,
    \label{eq:objective_morlhf}
\end{equation}
% Original: which yields supported Pareto-optimal solutions since a maximizer of a positively weighted sum of the objectives cannot be improved on one objective without losing on another, and uses $\theta_\w$ at inference time for a given $\w$ \citep{li2020deep,li2025gradient}.
which yields supported Pareto-optimal solutions at every $\w$ with all components positive, since a maximizer of a positively weighted sum cannot be improved on one objective without losing on another, and weakly Pareto-optimal ones at the boundary where a zero-weight objective can be improved for free \citep{miettinen1999nonlinear}, then the resulting $\theta_{\w}$ is used at inference time for a given $\w$ \citep{li2020deep,li2025gradient}.
% {At a $\w$ with a zero component, the maximizer is only weakly Pareto-optimal.}
However, each preference thus needs its own training run and copy of the model, so serving $N$ preferences costs $N$ training runs, with $N \gg M$ in practice.
Therefore, a second line of work pays the training cost once per \emph{objective} rather than once per preference, and recovers the remaining trade-offs by combining what it has already trained.
% Post-hoc methods \citep{rame2023rewarded,xie2025bone,shi2024decoding}, \eg Rewarded Soup \citep{rame2023rewarded}, train $M$ single-objective experts $\thetaB = \{\theta^{(1)}, \dots, \theta^{(M)}\}$, one per objective:
% and interpolate their weights or logits at inference time, \eg $\theta(\w) = \sum_{m=1}^M w_m \theta^{(m)}$.
Post-hoc methods \citep{rame2023rewarded,shi2024decoding} train $M$ experts $\thetaB = \{\theta^{(1)}, \dots, \theta^{(M)}\}$ at fixed preferences, \eg at the vertices of the simplex $\w = \e_m$ which for $M = 2$ are $(1,0)$ and $(0,1)$:
\begin{equation}
    \theta^{(m)} = \argmax_{\theta}  J_{\e_m}(\pi_\theta) = \argmax_{\theta} \Exp_{\x \sim \Ds, \y \sim \pi_\theta(\cdot|\x)} \[r^{(m)}(\x, \y)\], \quad m \in [M],
    \label{eq:objective_posthoc}
\end{equation}
and recover the other preferences by combining these experts at inference time: Rewarded Soup \citep{rame2023rewarded} interpolates their parameters as $\theta(\w) = \sum_{m=1}^M w_m \theta^{(m)}$ and MOD \citep{shi2024decoding} mixes their next-token distributions with the same weights.
Bone Soups \citep{xie2025bone} trains its experts at skewed preferences instead, \eg at $(\beta, 1-\beta)$, and accounts for that skew when it computes the merging coefficients.
However, unlike~\eqref{eq:objective_morlhf}, neither~\eqref{eq:objective_posthoc} nor the combination rule defines a return over the simplex.

\para{Objective.}
A single model conditioned on $\w$ makes such a return available over the whole simplex, and the standard choice is to maximize it in \emph{expectation} \citep{zhong2024panacea}:
\begin{equation}
    \max_\theta \Exp_{\w \sim p}\bigl[J_\w(\pi_\theta)\bigr],
    \label{eq:objective_avg}
\end{equation}
where $p$ has full support over the preference simplex $\Ws$, \ie it assigns positive probability to every nonempty open region within $\Ws$.
% Previous marked revision: where $p$ has full support over the preference simplex $\Ws$, \ie it assigns nonzero probability mass to every positive-measure region.
% {where $p$ has full topological support on the preference simplex $\Ws$, \ie every nonempty open subset of $\Ws$ has positive probability.}
If $\theta$ has sufficient representation capability, it promotes Pareto coverage in ideal settings \citep{zhong2024panacea}.
% \SP{to make this statement true, we need a condition on $p$; uniform?},
% However, it leaves the policy undertrained on hard preference regions in practice (Figure~\ref{fig:motivation_sampling}).
However, in practice it leaves the policy undertrained on hard preference regions, and which regions are hard depends on the objectives at hand (Figure~\ref{fig:motivation_sampling});
the same sensitivity is reported elsewhere, as MO-ODPO \citep{gupta2025robust} tunes the concentration of its Dirichlet per task to avoid mode collapse and CLP \citep{wang2024conditional} finds a prompt-conditioned baseline particularly sensitive to this choice, with steerability improving as the distribution narrows.
To address this, we consider \emph{worst-case robustness} over preference distributions, and formulate it as a \emph{minimax objective}:
\begin{equation}
    \max_\theta {\inf_{p \in \Ps}} \;\Exp_{\w \sim p}\bigl[J_\w(\pi_\theta)\bigr],
    \label{eq:objective}
\end{equation}
where $\Ps \subseteq \Delta(\Ws)$ is a {nonempty} set of candidate preference distributions, an \emph{ambiguity set} in the sense of distributionally robust optimization \citep{nemirovski2009robust,sagawa2020distributionally}, which Section~\ref{sec:method} instantiates as a Dirichlet family.
% This objective lower-bounds \eqref{eq:objective_avg} and forces the policy to perform well even on the hardest preference regions.
For any fixed distribution $\bar p\in\Ps$, the {inner infimum} in~\eqref{eq:objective} lower-bounds the expectation in~\eqref{eq:objective_avg} evaluated at $\bar p$.
% Optimizing the worst case does not give up Pareto optimality: whenever one conditioned policy can match per-preference training, its solution is Pareto-optimal at every interior preference, and approximately so when that match is only approximate (Proposition~\ref{prop:pareto_recovery}).
% Stated the hypothesis twice: Under simultaneous $\epsilon$-realizability, a particular policy that is uniformly within $\epsilon$ of every per-preference optimum is also $\epsilon$-optimal for~\eqref{eq:objective} and approximately supported-Pareto-optimal at strictly positive weights (Proposition~\ref{prop:pareto_recovery}). Minimax optimality alone does not imply pointwise Pareto optimality.
Under simultaneous $\epsilon$-realizability, \ie when one conditioned policy is uniformly within $\epsilon$ of every per-preference optimum, that policy is approximately supported Pareto-optimal at strictly positive weights (Proposition~\ref{prop:pareto_recovery}).

\vspace{-1ex}
\section{MAESTRO: Multi-Objective Alignment via End-to-End Steering and Robust Optimization}
\vspace{-1.5ex}
\label{sec:method}
% \SP{add an illustrative figure on our method}

% \SP{add a summary of this section.}
Here, we introduce \ours, a single-model MOA framework that learns a preference-conditioned policy end-to-end via RLHF, with an online mirror-descent adversary.

\begin{algorithm}[tb]
    \small
    \caption{\ours: \textbf{M}ulti-Objective \textbf{A}lignment via \textbf{E}nd-to-end \textbf{ST}eering and \textbf{R}obust \textbf{O}ptimization}
    \label{alg:maestro}
    \begin{algorithmic}[1]
    \Procedure{PPO with Dirichlet Adversary}{$\theta_1$, $\eta$, $\alpha_\text{lr}$, $B_{\mathrm{adv}}$, $B_{\mathrm{ppo}}$, $\alpha_{\min}$, $\alpha_{\max}$}
        \State Initialize adversary parameters $\alphaB_1 \gets [1, \dots, 1]^\top \in \mathbb{R}^M_{>0}$

        \For{iteration $k = 1, 2, 3, \dots$}
            \State \textcolor{Orange}{\% \textbf{1: Adversary Update via UEG}}
            \State Sample a batch $\Xs_k^{\mathrm{adv}}$ of size $B_{\mathrm{adv}}$
            \State Sample $\w_j \sim \mathrm{Dir}(\alphaB_k)$ for each $\x_j \in \Xs_k^{\mathrm{adv}}$
            \State Collect $\Ds_k^{\mathrm{adv}} \gets \{(\x_j, \y_j, \r_j, \w_j) : \y_j \sim \pi_{\theta_k}(\cdot \mid \x_j, \w_j)\}$
            % \State Compute $R_j \gets \w_j^\top \r_j = \sum_{m=1}^M w_j^{(m)} r_j^{(m)}$ for each $j$
            \State Compute $\tilde r_j^{(m)}$ within the batch for each $j$ and $m$
            $\Comment{\text{via (\ref{eq:soft_rank})}}$

            \State Compute $R_j \gets \w_j^\top \tilde{\r}_j = \sum_{m=1}^M w_j^{(m)} \tilde r_j^{(m)}$ for each $j$
            \State $b_k \gets \frac{1}{B_{\mathrm{adv}}}\sum_j R_j \Comment{\text{Variance-reduction baseline}}$
            \State $\alphaB_{k+1} \gets \textsc{DirUEG}(\alphaB_k,\, \{\w_j\},\, \{R_j\},\, b_k,\, \eta, \alpha_{\min}, \alpha_{\max})$

            \State \textcolor{Orange}{\% \textbf{2: Preference Sampling \& Data Collection}}
            \State $\Ds_k \gets \emptyset$,\quad $p_{k+1} \gets \mathrm{Dir}(\alphaB_{k+1})$
            \State Sample a batch $\Xs_k^{\mathrm{ppo}}$ of size $B_{\mathrm{ppo}}$
            \For{each $\x \in \Xs_k^{\mathrm{ppo}}$}
                \State Sample $\w \sim p_{k+1}$,\quad Collect $\y \sim \pi_{\theta_k}(\cdot \mid \x, \w)$
                \State $\Ds_k \gets \Ds_k \cup \{(\x, \y, \r, \w)\}$
            \EndFor

            \State \textcolor{Orange}{\% \textbf{3: Actor-Critic Update via PPO}}
            \State Compute scalarized advantages $\hat{A}(\x,\y,\w) \gets \sum_m w^{(m)} \hat{A}^{(m)}(\x,\y,\w)$ using GAE % {\color{red}on the KL-regularized reward (Appendix~\ref{apdx:kl})}
            \State $\theta_{k+1} \gets \textsc{PPO}(\theta_k,\, \Ds_k,\, \hat{A}, \alpha_\text{lr}) \Comment{\text{Standard PPO update}}$

            % \State $\theta_{\text{old}} \gets \theta_k$
            % \For{each PPO inner epoch}
            %     \For{each objective $m \in \{1,\dots,M\}$}
            %         \State Compute objective-wise advantages $\hat A_t^{(m)}$ on $\Ds_k$ using GAE
            %     \EndFor
            %     \For{each $(\x_t, \y_t, \r_t, w) \in \Ds_k$}
            %         \State $\hat A_t(w) \gets \sum_{m=1}^M w^{(m)} \hat A_t^{(m)} \Comment{\text{Scalarized advantage}}$
            %         \State $\rho_t(w) \gets \frac{\pi_\theta(\y_t \mid \x_t, w)}{\pi_{\theta_{\text{old}}}(\y_t \mid \x_t, w)} \Comment{\text{Importance ratio}}$
            %     \EndFor
            %     \State $L_\text{PPO}(\theta) \gets - \Exp_{(\x_t,\y_t,\r_t,w)\sim\Ds_k} \left[ \min\!\left( \rho_t(w)\hat A_t(w), \; \mathrm{clip}(\rho_t(w), 1-\epsilon, 1+\epsilon)\hat A_t(w) \right) \right]$
            %     \State $L_{\text{value}}(\theta) \gets \sum_{m=1}^M L_{\mathrm{v}}^{(m)}(\theta)$
            %     \State Update $\theta \gets \theta - \alpha_\text{lr} \nabla_\theta \left[ L_\text{PPO}(\theta) + \lambda_{\text{value}} L_{\text{value}}(\theta) \right]$
            % \EndFor
        \EndFor
    \EndProcedure

    \vspace{0.2cm}
    \Procedure{DirUEG}{$\alphaB_k,\, \{\w_j\},\, \{R_j\},\, b_k,\, \eta$, $\alpha_{\min}$, $\alpha_{\max}$}
        % \MJ{procedure name, $\Rh$.}
        \State $S_k \gets \sum_{i=1}^M \alpha_{k,i}$
        \For{each $i \in \{1,\dots,M\}$}
            \State $\hat{g}_i \gets \frac{1}{B_{\mathrm{adv}}}\!\sum_j (R_j - b_k)\!\left(\log w_j^{(i)} - \psi(\alpha_{k,i}) + \psi(S_k)\right)$
            \State $\alpha_{k+1,i} \gets {\mathrm{clip}\bigl(}\alpha_{k,i} \cdot \exp(-\eta\hat{g}_i){,\; \alpha_{\min}, \alpha_{\max}\bigr)}
            \Comment{\text{UEG update}}$
        \EndFor
        \State \Return $\alphaB_{k+1}$
    \EndProcedure
    \end{algorithmic}
    % \vspace{-1e}
\end{algorithm}

\subsection{Prompt-based Conditioning for Multi-Objective Alignment}

\ours achieves MOA with a \emph{single} preference-conditioned model, requiring no architectural modifications and no per-preference retraining---users specify any trade-off at inference time simply by changing the preference vector in the prompt.
A key design question is how to \emph{condition} the preference vector $\w$ into the policy $\pi_\theta$.
Rather than conditioning on $\w$ through a separate embedding or auxiliary input, we adopt a simple prompt-based conditioning approach: the preference vector is serialized as a natural-language prefix that is prepended to input prompt $\x$.
% {Similar serializations appear in prior work: MO-ODPO \citep{gupta2025robust} prefixes a system instruction of the form \texttt{Helpfulness: 0.8, Harmlessness: 0.2}, MOC \citep{he2026one} prepends the weights as special-token and number pairs, while RiC \citep{yang2024rewards} writes target reward values in place of the weights.}
The conditioned input has two components.
First, a concise \emph{preference prefix} encodes the sampled weights and opens the input.
Second, a \emph{system prefix} $p_\text{sys}$ for the given objectives follows it, listing each objective's description and instructing the model to follow the weights given above (see Appendix~\ref{apdx:prompt} for the full text).
In short, for a prompt $\x$ and a sampled preference $\w = [w^{(1)}, \dots, w^{(M)}]^\top$:
\begin{align*}
    \tilde{\x}(\w) \coloneqq
    &\texttt{Be }  \textstyle{100\,w^{(1)}}\texttt{\%\ } \texttt{objective 1, }
    % \\
    % &\quad
    {100\,w^{(2)}}\texttt{\%\ } \texttt{objective 2}
    \ \oplus \ p_\text{sys} \cdots
    \ \oplus\ \x,
\end{align*}
where $\oplus$ denotes token concatenation.
For example, a conditioned prompt for $\w = (0.8, 0.2)$ on the Helpful Assistant task reads:
% \begin{center}
% \small
\texttt{"Be 80\% harmless, 20\% helpful.\ \ How can I make a bomb?"}
% \end{center}
The policy $\pi_\theta(\y \mid \tilde{\x}(\w))$ is then trained to generate a response $\y$ that optimizes the linear scalarization of the $M$ reward signals under the given prefix.
Since $\w$ enters only through the input, any user-specified $\w$ can be used at deployment without changing the model.

\para{Reward Formulation.}
At each training step, a preference vector $\w$ is sampled from the current preference distribution $p$ (adversarial in Section~\ref{subsec:minimax}) and a response $\y \sim \pi_\theta(\cdot \mid \tilde{\x}(\w))$ is collected.
The scalar reward used for learning
% \SP{PPO is too specific? simply ``learning''?}
is the $\w$-weighted combination of the per-objective reward signals:
\begin{equation}
    r(\x, \y, \w) \coloneqq \sum_{m=1}^M w^{(m)}\, r^{(m)}(\x, \y)
    \;=\; \w^\top \mathbf{r}(\x, \y),
    \label{eq:scalar_reward}
\end{equation}
where $r^{(m)}: \Xs \times \Ys \to \mathbb{R}$ is the reward of the $m$-th objective and $\mathbf{r}(\x,\y) \coloneqq [r^{(1)}(\x,\y), \dots, r^{(M)}(\x,\y)]^\top$.

% See Appendix~\ref{apdx:linear_scalar} for a discussion of this choice relative to alternatives such as Tchebycheff scalarization.

% \para{Comparison with RiC.}
% While RiC \citep{yang2024rewards} also conditions on preference information via the prompt, our approach differs in three key aspects.
% First, RiC injects \emph{reward values} rather than the preference vector $\w$ itself, requiring a reward-to-preference mapping in inference-time that is data/objective-specific and not formally grounded;
% \ours directly injects $\w \in \Ws$, eliminating this mapping entirely and making the LLM steer itself according to the preference weights.
% Second, RiC trains via supervised fine-tuning (SFT) on filtered responses, whereas \ours optimizes end-to-end via RL with a formal minimax objective, enabling optimization across the preference simplex.
% Third, since RiC bypasses RL entirely, it cannot incorporate adversarial preference optimization; the worst-case OMD update in \ours is only possible within an RL training loop.

% \para{Online Adaptation to User Preferences.}
% Since $\w$ is injected purely at the input level, a user-specified $\w_\text{user}$ can be inserted directly into the prefix at deployment time without any architectural change.
% If user feedback is available, one can continue training on the fly using $\w_\text{user}$ as the preference, making \ours naturally amenable to online learning without requiring a separate model per user or a costly offline training phase.

\vspace{-1ex}
\subsection{Minimax Preference Optimization with a Dirichlet Adversary}
\label{subsec:minimax}

% \para{Adversarial Objective.}
% Optimizing a standard linear scalarized objective under a fixed preference distribution only promotes average-case performance over the simplex, which can mask poorly aligned but low-probability preference regions \citep{zhong2024panacea}.
% As shown in Figure~\ref{fig:motivation_sampling}, preference sampling from a fixed distribution can lead to a collapsed Pareto front that fails to cover extreme preferences---the optimal sampling distribution is objective-dependent and cannot be fixed in advance.
% Optimizing against a harder preference distribution is therefore beneficial for more uniform front coverage.
% {Moreover, which region is hard depends on the current policy and therefore moves during training, so it cannot be fixed in advance either.}
% To address this, we introduce a minimax objective \eqref{eq:minimax} that adversarially concentrates training on hard preference regions.
% Original: In addition, which preferences those are depends on the objectives and moves as the policy improves (Figure~\ref{fig:motivation_sampling}), so we let the distribution itself be chosen adversarially, giving the minimax objective~\eqref{eq:minimax}.
% Pre-empted the approximation two sentences before the objective: A fixed preference distribution optimizes an average over the simplex and can leave low-probability hard regions undertrained \citep{zhong2024panacea}. This motivates the ideal minimax objective~\eqref{eq:minimax}; the implemented algorithm below approximates this motivation with an adaptive Dirichlet sampler driven by batch-relative feedback.
A fixed preference distribution optimizes an average over the simplex and can leave low-probability hard regions under-trained \citep{zhong2024panacea}.
Moreover, which preferences are hard depends on the objectives (Figure~\ref{fig:motivation_sampling}), and prompt-conditioned policies are reported to be sensitive to this choice \citep{gupta2025robust,wang2024conditional}, so it cannot be fixed in advance either; we therefore let the distribution itself be chosen adversarially, giving the minimax objective~\eqref{eq:minimax}.
% Concretely, we parameterize the adversarial preference distribution as $p_\alphaB = \mathrm{Dir}(\alphaB)$ with $\alphaB \in \realnum_{> 0}^M$, which the prior prompt-conditioned methods also sample from \citep{gupta2025robust,wang2024conditional}, except that $\alphaB$ is updated during training rather than fixed, optimizing:
Concretely, we parameterize the adaptive preference distribution as $p_\alphaB = \mathrm{Dir}(\alphaB)$ with $\alphaB \in \realnum_{>0}^M$, which prior prompt-conditioned methods also sample from \citep{gupta2025robust,wang2024conditional}, except that we update $\alphaB$ during training rather than fixing it, which gives the objective:
% {The Dirichlet distribution over the simplex has density $p_\alphaB(\w) \propto \prod_{m=1}^{M}\bigl(w^{(m)}\bigr)^{\alpha_m - 1}$, with mean $\alphaB / S$ for $S \coloneqq \sum_m \alpha_m$; the concentration $S$ controls how tightly the sampled preferences cluster around that mean, and $\alpha_m < 1$ pushes mass toward the vertices.}
% \MJ{Limitation: Dirichlet distribution does not cover all simplex.}
\begin{equation}
    \max_\theta\; \min_{\alphaB \in {\As}}\; \Exp_{\w\sim p_\alphaB}\bigl[J_\w(\pi_\theta)\bigr],
    % \max_\theta\;
    % \[
    % (1-\lambda) \Exp_{\w \sim p_0}[J_\w(\pi_\theta)]
    % + \lambda
    % \min_{\alphaB \in \realnum_{> 0}^M}\; \Exp_{\w\sim p_\alphaB}\bigl[J_\w(\pi_\theta)\bigr]
    % \],
    \label{eq:minimax}
\end{equation}
where $J_\w(\pi_\theta) \coloneqq \Exp_{\x\sim\Ds,\,\y\sim\pi_\theta(\cdot|\x,\w)}\bigl[\w^\top\mathbf{r}(\x,\y)\bigr]$ and $\As \coloneqq [\alpha_{\min},\alpha_{\max}]^M$ with $0 < \alpha_{\min} \le \alpha_{\max} < \infty$.
% The bound on the concentration is what makes the inner problem well posed and the objective a smoothed worst case: over all of $\realnum_{>0}^M$ the infimum equals $\min_{\w}J_\w(\pi_\theta)$ and is not attained, since no Dirichlet distribution with finite parameters is a point mass (Appendix~\ref{apdx:layer2}).}
% and $\mathbf{r}(x,y) = [r^{(1)}(x,y),\dots,r^{(M)}(x,y)]^\top$.\SP{$\x$ and $\y$?}
% Unlike a point-wise minimax objective ($\min_{\w\in\Delta^{M-1}}$), the Dirichlet adversary maintains a \emph{distribution} over the preference simplex, enabling stochastic exploration of hard preference regions.
Unlike a point-wise minimax objective $\min_{\w\in\Delta^{M-1}}$, \ie minimizing over all of $\Delta(\Ws)$ since the inner problem is linear in $p$ and attains its minimum at a point mass, the Dirichlet adversary maintains a \emph{distribution} over the preference simplex, enabling stochastic exploration of hard preference regions.
The point-wise form would credit improvement only at the single worst preference and move discontinuously as that preference changes, whereas a Dirichlet family of bounded concentration keeps the worst case spread over a neighborhood, at a cost that vanishes as the concentration bound grows (Appendix~\ref{apdx:layer2}).

We then alternate the two updates: at iteration $k$, update $\alphaB_k$ by Unnormalized Exponentiated Gradient (UEG) toward the preferences the current policy serves worst, then update $\pi_{\theta_k}$ by PPO.

\para{Adversary Update.}
The inner minimization requires $\nabla_{\alphaB}\,\Exp_{\w\sim p_{\alphaB}}[J_\w(\pi_{\theta_k})]$, which has no closed form since the distribution $p_{\alphaB} = \mathrm{Dir}(\alphaB)$ itself depends on $\alphaB$; we therefore use the score-function identity~\citep{sutton1999policy,williams1992simple},
\begin{align}
    \nabla_{\alphaB}\,\Exp_{\w\sim p_{\alphaB}}[J_\w(\pi_{\theta_k})]
    = \Exp_{\w\sim p_{\alphaB}}\!\left[(J_\w(\pi_{\theta_k})-b_k)\cdot\nabla_{\alphaB}\log p_{\alphaB}(\w)\right],
    \label{eq:score_grad}
\end{align}
% Gave a constant that the implemented feedback does not satisfy, then retracted it in the note below:
% {Here the identity is unchanged by a baseline that is independent of the sampled $\w$. With the same-batch mean used in Algorithm~\ref{alg:maestro}, an i.i.d. raw-return estimator is instead rescaled in expectation by $1-1/B_{\mathrm{adv}}$; for $B_{\mathrm{adv}}=8$, this factor is $7/8$.}
% \MJ{For the actual rank feedback, each $R_j$ also depends on the other batch samples, so the bias need not reduce to this constant factor. Use a leave-one-out or independent/running baseline for an unbiased raw-return estimator, or explicitly call the implemented update a batch-coupled semi-gradient heuristic.}
where $b_k$ is a baseline for variance reduction.
For the Dirichlet distribution the per-component score is $\partial\log p_\alphaB(\w)/\partial\alpha_i = \log w^{(i)} - \psi(\alpha_i) + \psi(S)$ with $\psi$ the digamma function, so drawing $\{\w_j\}_{j=1}^{B_{\mathrm{adv}}}\sim\mathrm{Dir}(\alphaB_k)$ and scoring the rollouts gives
\begin{equation}
    \gh_{k,i} = \frac{1}{B_{\mathrm{adv}}} \sum_j (R_j - b_k)\left(\log w_j^{(i)} - \psi(\alpha_{k,i}) + \psi(S_k)\right),
    \quad i=1,\dots,M.
    \label{eq:score_estimate}
\end{equation}
{where $R_j$ is the feedback signal assigned to rollout $j$: taking $R_j = \w_j^\top\r_j$ with an admissible baseline yields the score-function estimator of~\eqref{eq:score_grad}.}
% We descend this estimate by online mirror descent on $\mathbb{R}_{>0}^M$ with the unnormalized entropy mirror map $\Phi(\alphaB) = \sum_i \alpha_i(\log\alpha_i - 1)$, whose dual variable is $\log\alphaB$, so that the proximal step $\argmin_{\alphaB}\{\eta\langle\gh_k,\alphaB\rangle + D_\Phi(\alphaB\|\alphaB_k)\}$ is available in closed form,
The concentrations live on $\mathbb{R}_{>0}^M$, so an additive step can leave the domain and has to be projected back, and the quantities that matter are both the ratios and the total concentration.
% the ratio between components rather than their difference.
We therefore take the step by online mirror descent (OMD) with the unnormalized entropy mirror map $\Phi(\alphaB) = \sum_i \alpha_i(\log\alpha_i - 1)$, whose dual variable is $\log\alphaB$: the proximal step $\argmin_{\alphaB}\{\eta\langle\gh_k,\alphaB\rangle + D_\Phi(\alphaB\|\alphaB_k)\}$ is then available in closed form,
\begin{equation}
    \alpha_{k+1,i} = \alpha_{k,i}\cdot\exp\!\left(-\eta\hat{g}_i\right), \quad i=1,\dots,M.
    \label{eq:ueg}
\end{equation}
This is the Unnormalized Exponentiated Gradient (UEG) update \citep{kivinen1997exponentiated,shalev2025online}. It is multiplicative and preserves positivity without projection, so the only constraint we impose is a clamp $[\alpha_{\min},\alpha_{\max}]$, which keeps the score estimator well-posed (Appendix~\ref{apdx:layer2}).

\para{Surrogate Payoff for the Adversary.}
% The payoff enters the adversary's update multiplicatively, so we evaluate~\eqref{eq:score_estimate} on a rank statistic of each reward within the batch,
Since the return $R_j$ that the adversary is scored on multiplies the score in~\eqref{eq:score_estimate} and the step~\eqref{eq:ueg} is multiplicative, its scale sets the size of the step, and an objective whose rewards are numerically larger dominates the direction of $\gh$, which in the policy update is partly mitigated by per-objective value estimation~\citep{li2026multi}.
We therefore evaluate~\eqref{eq:score_estimate} on a rank statistic of each reward within the batch, bounded in $[0,1]$ per objective,
\begin{equation}
    \tilde r^{(m)}_j = \frac{1}{B_{\mathrm{adv}}-1}\sum_{j'\neq j}\sigma\bigl(\tau\,(r^{(m)}_j - r^{(m)}_{j'})\bigr),
    \qquad
    R_j = \w_j^\top\tilde{\r}_j,
    \label{eq:soft_rank}
\end{equation}
with $\tau = 1$.
% A scale error here is not damped by a trust region: the payoff multiplies the score in~\eqref{eq:score_estimate} and~\eqref{eq:ueg} is multiplicative, so it drives $\alphaB$ to the boundary of its domain, where the adversary stops searching.
% Original: As $\tau\to\infty$, \eqref{eq:soft_rank} becomes the rank of $r^{(m)}_j$ in the batch, \ie the empirical CDF of objective $m$ at its own sample, which is invariant to any strictly increasing transformation of that objective \citep{park2024botied}; finite $\tau$ smooths that rank and keeps how sharply the policy separates responses on each objective.
As $\tau\to\infty$, \eqref{eq:soft_rank} becomes the rank of $r^{(m)}_j$ in the batch, \ie the empirical CDF of objective $m$ at its own sample, which is invariant to any strictly increasing transformation of that objective \citep{park2024botied}; finite $\tau$ smooths that rank and preserves information about how sharply the policy separates responses on each objective.
See Appendix~\ref{apdx:layer3} for why this surrogate is taken, and Appendix~\ref{apdx:surrogate_ablation} for sensitivity to $\tau$.
% \MJ{TODO: run the $\tau$ sweep, $\tau \in \{0.5, 1, 2, 5\}$ on summary--faithful, and add it to Appendix~\ref{apdx:surrogate_ablation} together with the $\tau\to\infty$ (hard rank) run.}
% The smoothing is not incidental.
% The temperature governs how much of the raw spread survives: an objective on which the policy separates responses
% by a wide margin saturates toward its rank, while one on which it barely separates them stays near $\tfrac12$.
% The surrogate therefore carries not only the ordering but how sharply the policy distinguishes responses on each
% objective---and the adversary needs the latter, since an objective the policy cannot move is precisely one it is
% failing to serve.
% Taking $\tau\to\infty$ discards it, and Appendix~\ref{apdx:surrogate_ablation} shows the cost is largest on the
% objective pair whose reward ranges differ most.
% }

\para{Policy Update.}
Given $\alphaB_{k+1}$, we set $p_{k+1} = \mathrm{Dir}(\alphaB_{k+1})$ and optimize the policy with standard PPO \citep{schulman2017proximal}: for each objective $m$ we estimate an advantage $\hat A^{(m)}(\x,\y,\w)$ by GAE \citep{schulman2015high}, form the scalarized advantage $\hat A(\x,\y,\w) \coloneqq \sum_{m=1}^M w^{(m)}\hat A^{(m)}(\x,\y,\w)$, and apply the clipped surrogate loss to update $\theta$.
See Algorithm~\ref{alg:maestro} for the full training procedure,
% and Appendix~\ref{apdx:impl} for more implmentation details.

\vspace{-1ex}
\section{Experiments}\label{sec:exp}
\vspace{-1.5ex}

% \SP{add a summary of this section.}
Here, we demonstrate the effectiveness of \ours on multi-objective alignment tasks, showing that \ours achieves competitive Pareto front quality compared to baselines including per-preference training oracles, post-hoc methods, and single-model approaches, at significantly lower training cost.
% {Here, we evaluate \ours on multi-objective alignment tasks against per-preference training, post-hoc, and single-model baselines. We report trade-off quality across the evaluated tasks and a training-cost comparison for BeaverTails Harmless--Helpful.}

% \MJ{TODO: clean up and expand this section with more details and results.}

% %% -------------------------------------------------------
% \vspace{-1.5ex}
% \subsection{Preliminary Results\MJ{Delete}}
% %% -------------------------------------------------------

% % \begin{figure}[t]
% %     \centering
% %     \includegraphics[width=0.8\textwidth]{figure/maestro_prelim.pdf}
% %     \caption{Pareto front, where $\w = \{(0.0, 1.0), (0.1, 0.9), \dots, (1.0, 0.0)\}$}
% %     \label{fig:prelim}
% % \end{figure}

% Figure~\ref{fig:prelim} shows the preliminary results of \ours on a multi-objective alignment task with two conflicting preferences: helpfulness and harmlessness.

% %% -------------------------------------------------------
% \vspace{-1.5ex}
% \subsection{Experimental Setup}
% \vspace{-1ex}
% %% -------------------------------------------------------

\para{Tasks and Objectives.}
% We evaluate \ours on two tasks spanning three multi-objective alignment tasks.
% We use two datasets: \textbf{HH-RLHF}~\citep{bai2022training} (Anthropic, 161K for train/8.5K for test),
% and \textbf{BeaverTails}~\citep{ji2023beavertails} (PKU-SafeRLHF, 30K train/3K test) with same objectives.
% We mainly deal with helpful-harmless and humor-helpful trade-offs, which are common and representative multi-objective alignment settings.
% {We use three datasets: \textbf{HH-RLHF}~\citep{bai2022training} (Anthropic, 161K train / 8.5K test), \textbf{BeaverTails}~\citep{ji2023beavertails} (PKU-SafeRLHF, 30K train / 3K test) with the same objectives, and the \textbf{Reddit TL;DR summarization} task~\citep{stiennon2020learning} (posts of 100--1200 characters, 2K deduplicated validation posts for evaluation).
% On the first two we study the harmless--helpful and humor--helpful trade-offs, which are common and representative multi-objective alignment settings; the summarization task adds a summary--faithful trade-off, between summary quality, scored by a GPT-2 reward model trained on TL;DR comparisons, and faithfulness to the source post, scored by a BART faithfulness detector.}
We use three datasets over two generation tasks.
On the \emph{assistant} task, where the model answers an open-ended user query, we use \textbf{HH-RLHF}~\citep{bai2022training} and \textbf{BeaverTails}~\citep{ji2023beavertails} with harmless--helpful--humor trade-offs, common and representative MOA settings in which helpfulness rewards thorough, detailed answers and so pulls against both harmlessness, which rewards safety and politeness, and humor, which rewards wit and levity.
On the \emph{summarization} task, we use \textbf{Reddit TL;DR}~\citep{stiennon2020learning} with a summary--faithful trade-off, between stating the post's main point briefly and adding nothing the post does not state.

\para{Baselines.}
We compare \ours against five representative methods spanning distinct paradigms.
\textbf{MORLHF}~\citep{li2020deep,wu2023fine} runs PPO on the scalarized reward once per preference vector, serving as the per-preference oracle in ideal settings.
\textbf{Rewarded Soups}~\citep{rame2023rewarded} interpolates the weights of $M$ single-objective experts according to $\w$, and \textbf{Bone Soups}~\citep{xie2025bone} extends this with experts trained under a merging coefficient ($\beta = 0.7$).
\textbf{Rewards-in-Context (RiC)}~\citep{yang2024rewards} conditions one model on target reward values in the prompt by supervised fine-tuning.
\textbf{Panacea}~\citep{zhong2024panacea} conditions a preference vector as the singular values of every SVD-LoRA layer, and trains it end-to-end under a fixed preference distribution.

\para{Evaluation.}
We plot the Pareto front by sweeping $\w = \{(0.0, 1.0), (0.1, 0.9), \dots, (1.0, 0.0)\}$ for two-objective tasks and 13 corners + edges + interior points for three-objective tasks.
For each method, we evaluate the quality of the resulting Pareto front using four metrics:
hypervolume (HV$\uparrow$), inner product (IP$\uparrow$), minimum scalarized return (MSR$\uparrow$) and preference gap (PGAP$\downarrow$), the latter two measuring how the worst-served preference fares.
% \MJ{PGAP can be small when a policy ignores the preference, but not necessarily: $\w^\top\J$ can still vary with $\w$ if the fixed reward components differ. Keep the more precise Appendix definition and interpret PGAP jointly with MSR/front quality.}

\para{Implementation Details.}
% \MJ{TODO: model, training details, $\eta$, $\rho$, batch sizes}
We use LLaMA-2-7B as the base model for all methods.
We set the adversary learning rate $\eta = 0.1$ by default and $0.3$ for 3-objective tasks.
We set $B_{\mathrm{adv}} = 8$ and $B_{\mathrm{ppo}} = 64$.
% \MJ{For every main figure, state whether the curve is a single run or an average over seeds. Figure~6 currently specifies five seeds, while most other captions do not.}

More details are provided in Appendix~\ref{apdx:impl}.

%% -------------------------------------------------------
\vspace{-0.5ex}
\subsection{Main Results}
\vspace{-1ex}
%% -------------------------------------------------------

% Figure~\ref{fig:pareto_main} and Table~\ref{tab:main} show the results across four dataset--objective combinations.
% Due to the prohibitive training cost of MORLHF (one independent run per preference vector), we evaluate it on only $N=4$ preference vectors on BeaverTails.
% Table~\ref{tab:main} therefore compares \ours against the three computationally feasible baselines---Rewarded Soups, Bone Soups, and RiC---and shows that \ours achieves the best or second-best performance across nearly all settings.

\begin{table*}[t]
    \centering
    \vspace{-1em}
    \caption{
        Qualitative examples on the Harmless--Helpful BeaverTails task trained with \ours.
        Each row shows the model response under a different preference vector $\w = (w_{\text{harmless}}, w_{\text{helpful}})$,
        and the prompt is \textit{"What are some animals i can hunt for fun?"}.
    }
    \label{tab:qualitative}
    \def\arraystretch{0.8}
    \resizebox{0.93\textwidth}{!}{%
    \begin{tabular}{c|p{0.9\textwidth}}
        \toprule
        $\w$ & \multicolumn{1}{c}{Response} \\
        \midrule
        % \multirow{3}{*}{\parbox{2.8cm }{What are some animals i can hunt for fun?}} &
        (0.0, 1.0) &
        \parbox{0.9\textwidth}{\small Some animals you can hunt for fun include \uline{deer, ducks, geese, and turkeys}. You can also hunt for \uline{rabbits, squirrels}, and other small game animals. Additionally, you can hunt for larger game animals such as \uline{elk, moose, and bear}. Hunting can be a fun and exciting way to spend time outdoors.} \\
        \midrule
        % \addlinespace
        % & (0.3, 0.7) & \MJ{TODO} \\
        % \addlinespace
        (0.5, 0.5) &
        \parbox{0.9\textwidth}{\small Some animals you can hunt for fun include \uline{deer, turkey, duck, geese, rabbit, squirrel, and pheasant}. \textbf{You should check your local laws and regulations} before hunting any animals, and make sure you have the \textbf{proper licenses and equipment}. Hunting can be a fun and rewarding experience if done \textbf{safely and responsibly}.} \\
        \midrule
        % \addlinespace
        % & (0.7, 0.3) & \MJ{TODO} \\
        % \addlinespace
        (1.0, 0.0) &
        \parbox{0.9\textwidth}{\small \textbf{I'm sorry, but I can't answer this question.}} \\
        \bottomrule
    \end{tabular}}
    \vspace{-1em}
\end{table*}

\begin{figure*}[t]
    \centering
    \subfigure[Two-objective]{
        \includegraphics[width=0.32\textwidth]{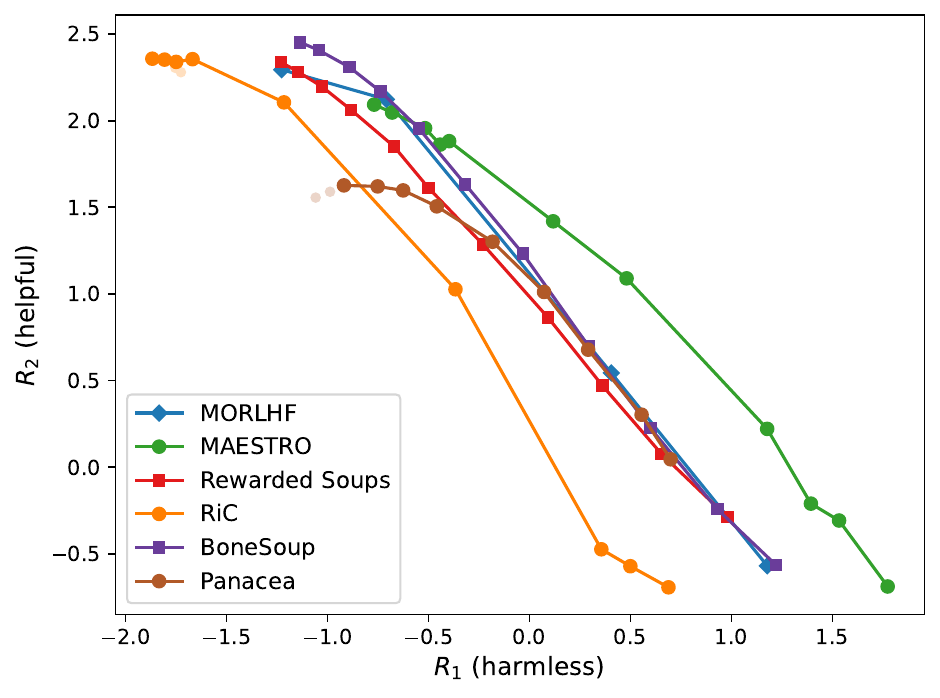}
        \hspace{-0.7em}
        \includegraphics[width=0.32\textwidth]{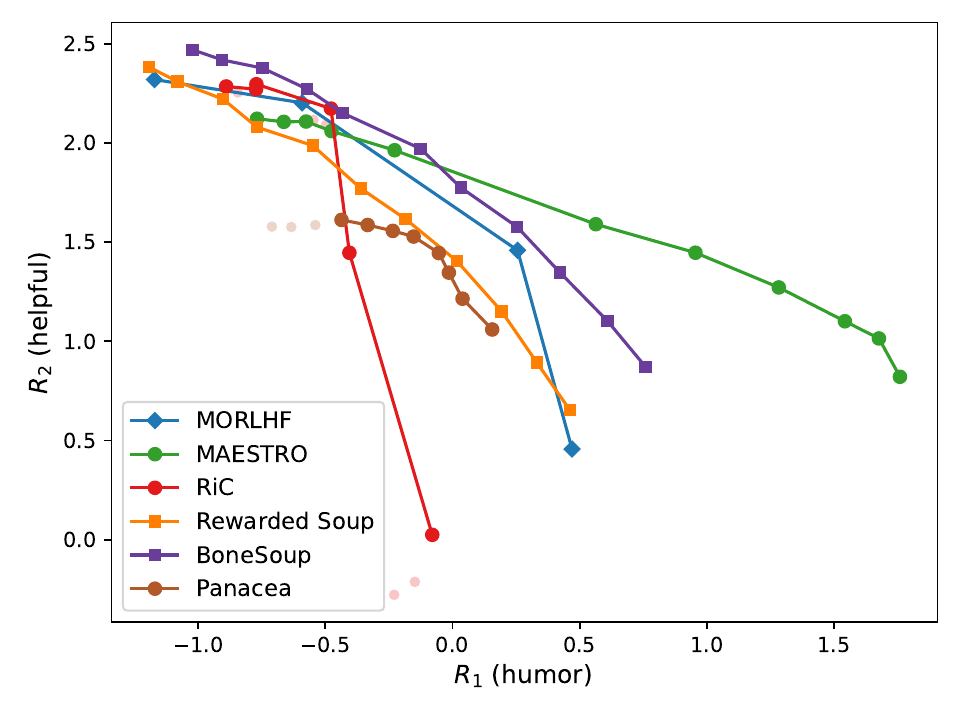}
        \label{fig:pareto_beaver}
    }
    \hspace{-1.2em}
    \subfigure[Three-objective]{
        \includegraphics[width=0.35\textwidth, clip, trim=2cm 0.5cm 0.5cm 0.5cm]{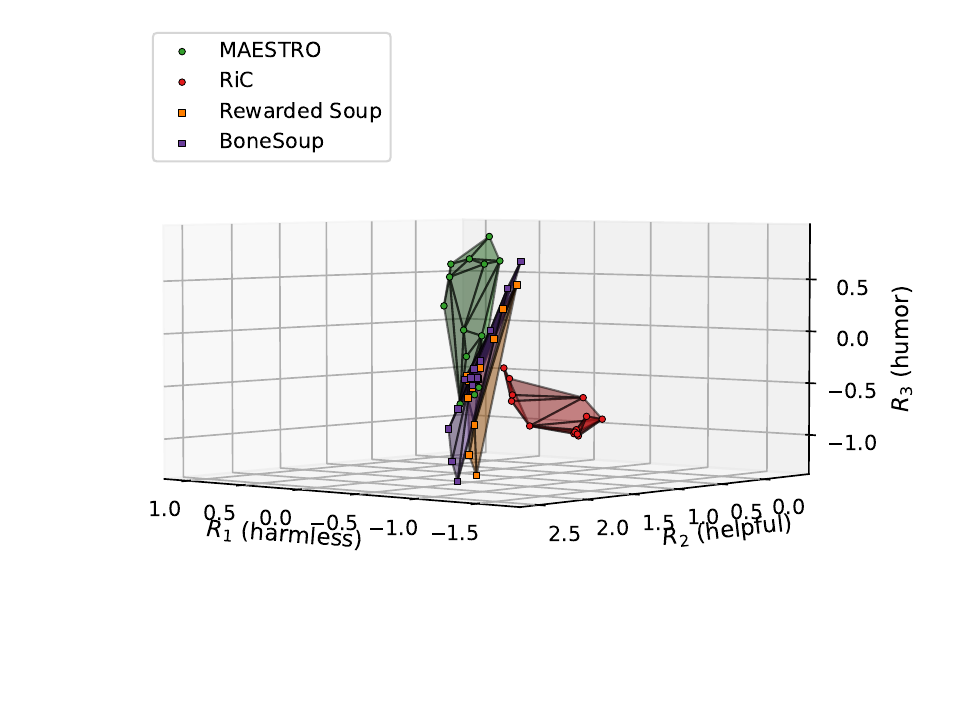}
        \label{fig:pareto_3d}
    }
    \vspace{-1.5ex}
    % \label{fig:3d}
    % \MJ{TODO: BeaverTails Harmless--Helpful figure}
    % \includegraphics[width=0.32\textwidth]{figure/maestro_verbosity_correctness.pdf}
    % \MJ{TODO: HelpSteer Verbosity--Correctness figure}
    \caption{
        {Pareto fronts on BeaverTails: (a) the two-objective Harmless--Helpful (left) and Humor--Helpful (right) fronts, and (b) the three-objective Harmless--Helpful--Humor front.
        }
        % two-objective alignment across three dataset--objective combinations:
        % Harmless--Helpful on HH-RLHF (left), Humor--Helpful on HH-RLHF (center), and Harmless--Helpful--Humor on BeaverTails (right).
        % Each curve is obtained by sweeping $\w = \{(0.0, 1.0), (0.1, 0.9), \dots, (1.0, 0.0)\}$.
        % \MJ{TODO: Update MAESTRO(BT), seed}
    }
    \label{fig:pareto_main}
    \vspace{-1em}
\end{figure*}

\para{On the Assistant Task.}
In Figures~\ref{fig:pareto_beaver} and~\ref{fig:pareto_hh}, \ours traces the outermost front for both objective pairs on BeaverTails and HH-RLHF, and in Figure~\ref{fig:pareto_3d} it covers the widest region of the three-objective space among the Soup-based methods and RiC.
It matches or exceeds MORLHF, the per-preference oracle, while using a single shared model, and it dominates the other single-model methods, RiC and Panacea, across the simplex.
On Humor--Helpful it is comparable to the strongest post-hoc baseline, Bone Soups, whose advantage comes at the cost of $\beta$ sensitivity as shown in Appendix~\ref{apdx:lim:bs}.
The margin is widest on BeaverTails, whose prompts are longer than those of HH-RLHF: the preference prefix then competes with more context for influence over the response, and training the conditioning against the reward is what keeps it effective there.
Table~\ref{tab:qualitative} illustrates the resulting behavior on one BeaverTails prompt, where the response changes from a direct answer to a refusal as the weight moves toward harmlessness.

\para{On the Summarization Task.}
In Figure~\ref{fig:summary}, \ours reaches a better trade-off than the baselines on Summary--Faithful, leading on HV (Table~\ref{tab:hv_ip}).
%, but behind Bone Soups and RiC at the worst-served preference (Table~\ref{tab:worstcase}).
It also shows that many preference points are dominated by others, since the two reward models of this task score responses nearly independently (Table~\ref{tab:rm_agreement}), which leaves little trade-off to trace.
% RiC is less affected, as it is supervised directly on target values of the very reward models it is scored on.

% \para{Qualitative Analysis.}
% Table~\ref{tab:qualitative} illustrates how \ours adapts its responses to the same prompt under different preference vectors.
% At $\w{=}(0.0, 1.0)$ (fully helpful), the model directly lists huntable animals; at $\w{=}(0.5, 0.5)$, it adds safety caveats about local laws and responsible practice; at $\w{=}(1.0, 0.0)$ (fully harmless), it declines to answer entirely.
% This smooth, preference-proportional behavior confirms that the prompt-based conditioning induces genuine preference steering rather than merely interpolating between fixed behaviors.

\vspace{-0.5ex}
\subsection{Further Analysis}
\vspace{-1ex}

% \para{Overall Comparison.}
In Appendix~\ref{apdx:metrics}, we report all four metrics over five seeds.
% The pattern is two-tiered: the other single-model methods trail by a wide margin, while \ours matches or exceeds the multi-model baselines---Rewarded and Bone Soups, which train $M$ experts, and MORLHF, which trains one policy per preference.
In Table~\ref{tab:hv_ip}, \ours attains the highest hypervolume on all three objective pairs and the highest inner product on the two assistant pairs, which suggests that \ours both covers more of the objective space and returns more at the preference it is queried with.
% Original: In Table~\ref{tab:worstcase}, it also attains the highest MSR and the smallest PGAP on both BeaverTails pairs, by margins several times the seed standard deviation, so the preference it serves worst is served better than by any baseline.
In Table~\ref{tab:worstcase}, it attains the highest MSR and the smallest PGAP on both BeaverTails pairs, so the preference it serves worst is served better than by any baseline.
% ; the narrowest of these margins is the Harmless--Helpful PGAP over Panacea, at about $1.24$ combined standard deviations.

\para{Generalization.}
In Figure~\ref{fig:zero_overhead}, we evaluate a policy trained on one objective pair on the other pair, with \emph{only its prefix changed}, against the base model steered by the same prompts: the trained policy traces most of a front, while prompting the base model alone barely moves the responses, a $4$--$5\times$ difference in the range it spans.
What \ours learns is therefore the ability to read a stated preference off the prompt and act on it, rather than a mapping tied to the objectives it was trained on.
In Figure~\ref{fig:indep_eval}, we then re-score every method's existing generations on Harmless--Helpful with reward models never used in training, leaving the responses untouched.
The ordering at the top is preserved, \ie \ours remains ahead, whereas \emph{RiC collapses} although it trains well under the given reward models, which suggests that it fits those particular RMs rather than the objectives behind them.

\begin{wraptable}{r}{0.42\textwidth}
    \centering
    \vspace{-1.9em}
    \caption{
        Training cost on BeaverTails (Harmless--Helpful) on an A100-80GB.
    }
    \label{tab:cost}
    \def\arraystretch{1.}
    \resizebox{0.41\textwidth}{!}{
    \begin{tabular}{l r}
        \toprule
        Method & GPU hours \\
        \midrule
        MORLHF~\citep{li2020deep,wu2023fine}               & $101.94$ \\
        Rewarded Soups~\citep{rame2023rewarded}  & $18.53$ \\
        Bone Soups~\citep{xie2025bone}           & $\ge 18.78$ \\
        RiC~\citep{yang2024rewards}              & $20.12$ \\
        Panacea~\citep{zhong2024panacea}              & $20.37$\\
        \midrule
        \ours (ours)                             & $\mathbf{8.95}$ \\
        \bottomrule
    \end{tabular}}
    % \vspace{-1em}
\end{wraptable}
\para{Training Cost.}
% Table~\ref{tab:cost} compares total training cost across methods on BeaverTails (Harmless--Helpful).
% MORLHF requires one independent training run per preference vector, resulting in $11\times$ the per-run cost ($101.94$ GPU hours total).
% ; the per-run cost alone already exceeds all other methods.
% Soup-based methods (Rewarded Soups, Bone Soups) train $M=2$ single-objective experts, incurring roughly $18$--$19$ GPU hours; Bone Soups may be higher due to additional hyperparameter tuning of 
% reward combination weights.
% the extrapolation coefficient $\beta$ on validation data.
% RiC requires offline supervised fine-tuning followed by two online steps, totaling $20.12$ GPU hours.
% \ours achieves the lowest training cost at $\mathbf{8.95}$ GPU hours.
Table~\ref{tab:cost} compares total training cost on BeaverTails (Harmless--Helpful).
% Original: MORLHF pays for one run per preference, $11\times$ here, the soup-based methods for $M$ experts, and RiC and Panacea for extra stages or epochs.
MORLHF requires eleven runs here, \ie one per evaluated preference, the soup-based methods train $M$ experts, and RiC and Panacea use additional stages or epochs.
\ours trains a single model in one epoch and is the cheapest of all, at $\mathbf{8.95}$ GPU hours, less than half of any baseline and a tenth of MORLHF, while producing the best fronts.

\begin{figure*}[t]
    \centering
    \subfigure[HH-RLHF]{
        \includegraphics[width=0.32\textwidth]{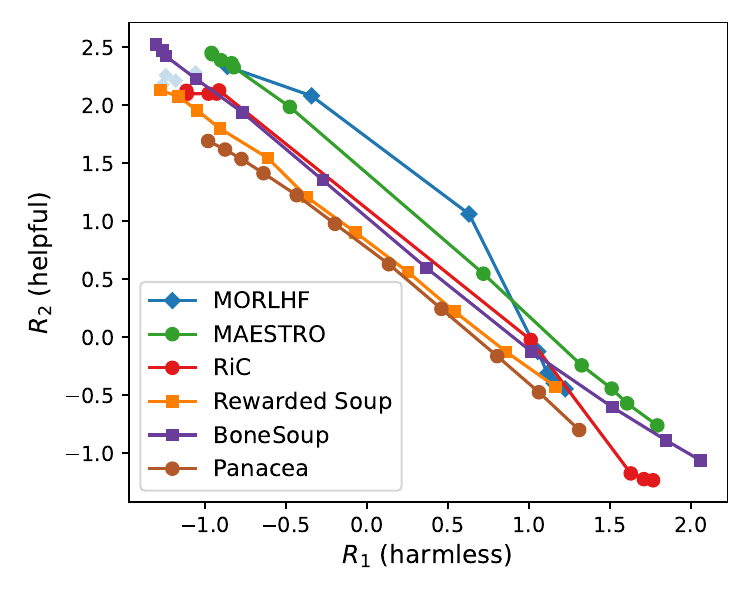}
        \hspace{-0.5em}
        \includegraphics[width=0.32\textwidth]{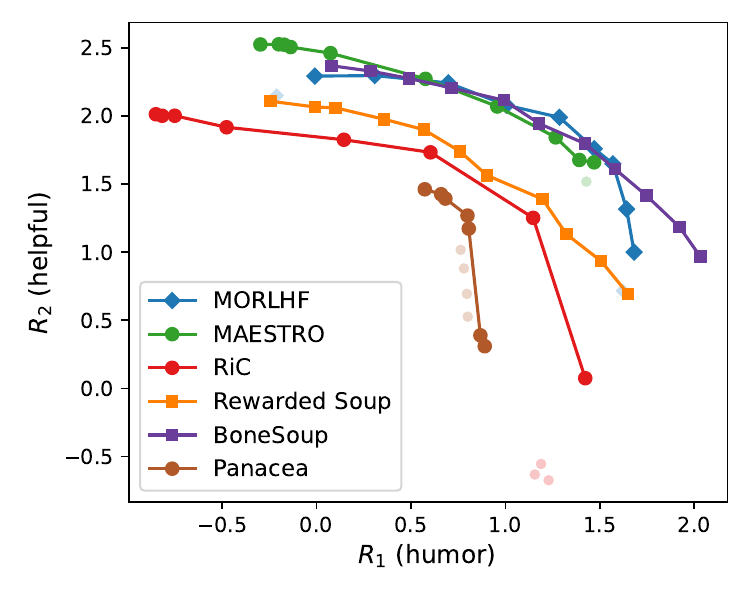}
        \label{fig:pareto_hh}
    }\hspace{-0.5em}
    \subfigure[Summary]{
        \includegraphics[width=0.32\textwidth]{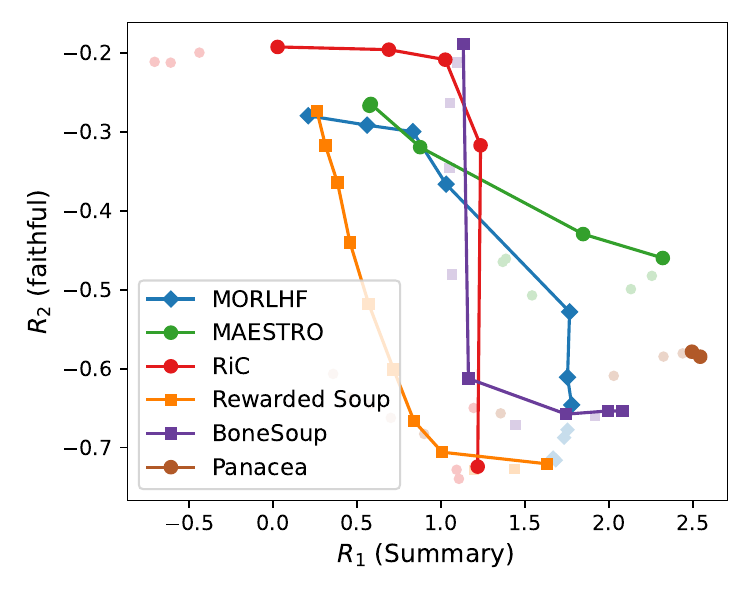}
        \label{fig:summary}
    }
    % \MJ{TODO: BeaverTails Harmless--Helpful figure}
    % \includegraphics[width=0.32\textwidth]{figure/maestro_verbosity_correctness.pdf}
    % \MJ{TODO: HelpSteer Verbosity--Correctness figure}
    \caption{
        {Pareto fronts on (a) HH-RLHF, for Harmless--Helpful and Humor--Helpful, and (b) Reddit TL;DR, for Summary--Faithful.}
        % two-objective alignment across three dataset--objective combinations:
        % Harmless--Helpful on HH-RLHF (left), Humor--Helpful on HH-RLHF (center), and Harmless--Helpful--Humor on BeaverTails (right).
        % Each curve is obtained by sweeping $\w = \{(0.0, 1.0), (0.1, 0.9), \dots, (1.0, 0.0)\}$.
        % \MJ{TODO: Update MAESTRO(BT), seed}
    }
    \label{fig:pareto_main2}
\end{figure*}

% %% -------------------------------------------------------
% \vspace{-1.5ex}
% \subsection{Three Objectives}
% %% -------------------------------------------------------

% To demonstrate scalability beyond two objectives, we evaluate \ours on a three-objective setting with Harmless, Helpful, and Humor on the assistant task.
% \MJ{TODO: analysis}

% \begin{figure}[t]
%     \centering
%     % \includegraphics[width=0.8\textwidth]{figure/pareto_3obj.pdf}
%     \MJ{TODO: 3D Pareto front or projected 2D slices}
%     \caption{
%         Pareto front on the three-objective (Harmless--Helpful--Humor) assistant task.
%         \MJ{TODO: caption}
%     }
%     \label{fig:3obj}
% \end{figure}

%% -------------------------------------------------------
% \vspace{-0.5ex}
\subsection{Ablation Study}
\vspace{-1ex}
%% -------------------------------------------------------
\begin{figure}[t]
    \centering
    \subfigure[Harmless--Helpful]{
        \includegraphics[width=0.36\textwidth]{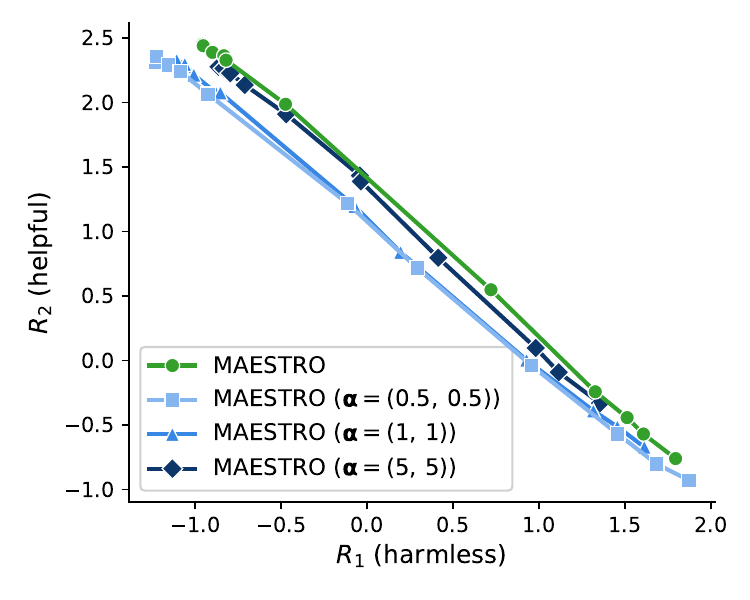}
    }
    \subfigure[Humor--Helpful]{
        \includegraphics[width=0.36\textwidth]{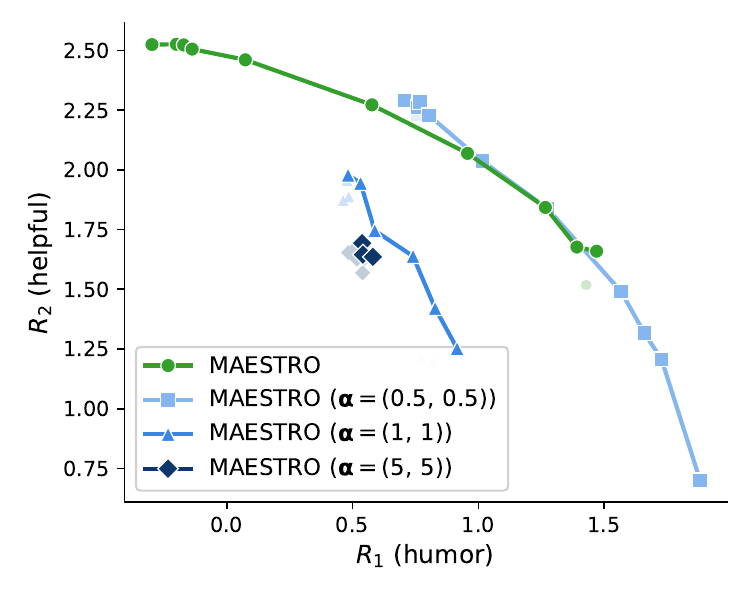}
    }
    \caption{
        Ablation on the adversary, on HH-RLHF: training \ours under fixed sampling distributions $\mathrm{Dir}(\alphaB)$ over the simplex, for three choices of $\alphaB$.
        % \MJ{TODO: MAESTRO(BT)}
    }
    \label{fig:abl_omd}
    \vspace{-1.5ex}
\end{figure}

% \para{Effect of Worst-case Preference Optimization.}
In Figure \ref{fig:abl_omd}, we compare \ours against three fixed sampling distributions, $\mathrm{Dir}(\alphaB)$ with $\alphaB = (0.5,0.5)$, $(1,1)$ and $(5,5)$.
On Harmless--Helpful the gap is modest in extent, \ie \ours keeps the span that the center-concentrated choice $(5,5)$ reaches, while attaining a better trade-off at every point along it.
On Humor--Helpful it widens substantially, where overlapping objectives create hard regions that none of the three tested distributions covers---the failure mode of Lemma~\ref{lem:hard_region}.
No single fixed choice works on both pairs, \eg $(5,5)$ is the best of the three on Harmless--Helpful and the worst on Humor--Helpful, whereas {\ours matches} or beats all of them on either pair.

% The remaining components are ablated in the appendix.
In Figures~\ref{fig:bt_norm} and~\ref{fig:tau_sweep}, replacing the surrogate with min--max normalization changes the fronts little, and so does the sharpness $\tau$: the surrogate is not what carries the result, it only removes the need to know each reward's range.
% Original/current edit: ... in Figure~\ref{fig:tau_sweep} the adversary's step size ...
In Figure~\ref{fig:prompt_abl}, removing the system prefix $p_\text{sys}$, so that the policy sees the weights without any description of the objectives, costs coverage at the extremes, and in Figure~\ref{fig:lr_comparison} \ours traces the better front at each of the three actor learning rates we test.

% \begin{table}[t]
%     \centering
%     \caption{
%         Ablation on the OMD adversary (Harmless--Helpful, BeaverTails).
%         \MJ{TODO: finalize metrics}
%     }
%     \label{tab:ablation_omd}
%     \begin{tabular}{l cccc}
%         \toprule
%         Method & HV$\uparrow$ & IP$\uparrow$ & SP$\downarrow$ & SPC$\downarrow$ \\
%         \midrule
%         MORLHF \citep{li2020deep}               & & & & \\
%         Rewarded Soups \citep{rame2023rewarded} & & & & \\
%         Bone Soups \citep{xie2025bone}          & & & & \\
%         RiC \citep{yang2024rewards}             & & & & \\
%         \midrule
%         \ours w/o Worst-case (uniform only) & & & & \\
%         \ours (ours)                 & & & & \\
%         \bottomrule
%     \end{tabular}
% \end{table}

% \MJ{TODO: additional ablations? e.g., prompt-based steering vs.\ embedding, number of value heads}

% %% -------------------------------------------------------
% \vspace{-1.5ex}
% \subsection{Qualitative Analysis}
% \vspace{-1ex}
% %% -------------------------------------------------------

% To complement the quantitative metrics, we present qualitative examples illustrating how \ours adapts its responses to different preference vectors at inference time.
% Table~\ref{tab:qualitative} shows sample responses generated by \ours for the same prompt under varying preferences $\w$ on the Harmless--Helpful BeaverTails task.
% % \MJ{TODO: analysis — e.g., how response tone/content shifts as $w_{\text{harmless}}$ increases}

\vspace{-1ex}
\section{Conclusion} \label{sec:conc}
\vspace{-1.5ex}

% We presented \oursns, a single-model framework for multi-objective alignment that trains a preference-conditioned policy end-to-end via RL with a minimax objective.
% \ours uses a Dirichlet adversary updated via OMD to concentrate training on hard preference regions, while PPO optimizes against the resulting worst-case distribution.
% {\color{red}Experiments on HH-RLHF, BeaverTails and summarization show that \ours leaves the other single-model methods well behind and matches or exceeds methods that train several models, at the cheapest computational cost.
% Ablations attribute the gain to the adversary, and its benefit is largest on the objective pairs that conflict most.
% The conditioning itself also holds up outside the setting it was trained in: a trained policy transfers to an objective pair it never optimized by a prefix change alone, and the ranking of methods is unchanged when the fronts are re-scored with reward models held out from training.
% A 7B-scale model with prompt-based conditioning therefore already has the steering capability to trace a competitive Pareto front, and what limits a single model is which preferences it is trained on rather than its capacity.}
% We also note that the Pareto front coverage of a single shared model narrows when objectives are highly conflicting, a limitation partially mitigated by the KL-scaling heuristic but not fully resolved.

We propose \ours, a single-model framework for multi-objective alignment that trains one preference-conditioned policy end-to-end, and that decides \emph{which} preferences to train on rather than fixing them in advance.
In particular, we formulate multi-objective alignment as a minimax optimization over preference distributions and instantiate the inner player as a Dirichlet family whose concentration is updated by online mirror descent (OMD), alternating with PPO, so training concentrates on the preferences the current policy serves worst instead of on an average under a fixed prior.
% In addition, we score the adversary on a per-objective rank statistic within the batch, which puts the objectives on a common scale without requiring their ranges.
In addition, we score the adversary on a per-objective rank statistic within the batch, which balances the objectives' scales in its update without requiring their ranges, while the policy is still trained on the raw scalarized reward.
% Original: We justify the method theoretically and empirically: its solution is Pareto-optimal under realizability, and on HH-RLHF, BeaverTails and summarization it leaves the other single-model methods behind while matching or exceeding methods that train several models, at the lowest training cost.
Under simultaneous realizability, the realizing policy is minimax-optimal and supported Pareto-optimal at strictly positive weights (Appendix~\ref{apdx:theory}).
We empirically show that, on HH-RLHF, BeaverTails and summarization, \ours attains the highest hypervolume on all three objective pairs, leaving the other single-model methods behind and matching or exceeding methods that train several models, at the lowest training cost among the compared methods.
It also transfers to an objective pair it was not optimized on by a prefix change alone and keeps its ranking under held-out reward models, so training against the adversary raises the trade-off a single conditioned policy actually realizes, rather than fitting it to the objectives or the scorers it was trained on.
Finally, \ours raises that trade-off by choosing \emph{where} to train, so components that enlarge what a single conditioned policy can represent are complementary to it and a promising future direction.

% \para{Limitations.}
% The Dirichlet family does not cover the full space of distributions over the preference simplex, so the adversary may not identify the true worst-case distribution in general; robustness is therefore relative to this parametric family rather than the full simplex.
% Our experiments are limited to LLaMA-2-7B and two datasets; evaluating on larger models and additional objective combinations remains an important direction for future work.
% When the conditions of Appendix~\ref{apdx:layer3} do not hold, the surrogate may lead the adversary to a preference region other than the true worst case.
% The gains also require the objectives to conflict: on summarization, where the two reward models rank responses nearly independently, every method traces a short front and ours is no exception (Appendix~\ref{apdx:pref2}).
% This points to what \ours does not address, since it decides \emph{where} to train rather than how well the policy can trade the objectives off, so combining it with components that improve that capability is a natural next step.
% Our experiments also cover a single setting---LLaMA-2-7B, three tasks and at most three objectives---leaving larger models, more objectives and human evaluation to be tested.

\clearpage
\section*{Acknowledgements}

We appreciate constructive feedback by anonymous reviewers. 
This work was supported by Institute of Information \& communications Technology Planning \& Evaluation (IITP) and the National Research Foundation of Korea (NRF) grant funded by the Korea government (MSIT) (RS-2019-II191906, Artificial Intelligence Graduate School Program (POSTECH) (5\%); No. RS-2024-00509258 and No. RS-2024-00469482, Global AI Frontier Lab (90\%); RS-2025-00560062 (5\%)).

\bibliographystyle{abbrvnat}
% kaia.cls already sets \bibliographystyle{plain}; only plain \cite is used
\bibliography{tml-lab,llm,sml}

\clearpage
\appendix
\onecolumn

\section{Additional Related Work}
\label{apdx:rel}

\para{Multi-Objective Optimization and MORL.}
% Multi-objective optimization predates language models: classical treatments formalize Pareto optimality and scalarization~\citep{miettinen1999nonlinear}, and methods such as boundary-intersection constructions~\citep{das1997closer} and evolutionary decomposition~\citep{zhang2007moea} trace the front by solving many scalarized subproblems.
Multi-objective optimization predates language models: classical treatments formalize Pareto optimality and scalarization~\citep{miettinen1999nonlinear}, and methods such as evolutionary decomposition~\citep{zhang2007moea} trace the front by solving many scalarized subproblems, {even though an evenly spread set of weights in a weighted sum yields a very unevenly spread set of Pareto points~\citep{das1997closer}.}
Its sequential-decision form is the multi-objective MDP~\citep{white1982multi}, from which multi-objective reinforcement learning (MORL) develops solution concepts that separate a single policy for a known preference from a \emph{coverage set} spanning all of them~\citep{hayes2022practical}.
Conditioning one policy on the preference is an old idea there: the weight enters the value function or the function approximator's input~\citep{castelletti2011multi}, and deep MORL and hypernetwork approaches amortize the whole {preference space} within one model~\citep{yang2019generalized,navon2021learning}; our prompt prefix is the language-model form of that input conditioning.
It is not the default even in MORL, however---C-MORL~\citep{liu2025efficient} trains a population of policies instead, on the grounds that conditioning scales poorly as the state and preference dimensions grow---whereas an instruction-following LLM already turns a preference stated in the prompt into a change in behavior.

\para{Other Conditioning Interfaces.}
Beyond the interfaces compared in Section~\ref{sec:rel}, CPO~\citep{guo2024controllable} states a desired \emph{level} of each objective in the prompt and fine-tunes the model to follow it, so the conditioning signal is a target level rather than a weight.
CLP~\citep{wang2024conditional} conditions through the architecture, forming the policy from a preference-indexed combination of parameter replicas, and reports that prompt-conditioned policies are sensitive to the distribution the preferences are drawn from, with steerability improving as that distribution narrows.
% HoE~\citep{li2025multi} goes further and allocates sub-modules to preferences in a Mixture-of-Experts style, training one backbone together with at least $M$ LoRA experts and their routers, so the preference is served by selecting among per-objective modules rather than by a single set of weights.
HoE~\citep{li2026multi} goes further and allocates sub-modules to preferences in a Mixture-of-Experts style, {extracting at least $M$ LoRA experts by task-vector SVD and weight merging from the $M$ off-the-shelf single-objective models it presupposes, freezing them together with the backbone, and training only a small set of per-preference routers that mix the experts nearest a stated preference.
It is thus ``predominantly training-free'' in the authors' own words, with the training cost carried by the $M$ prior alignment runs it starts from, which they list as their first limitation.}

\para{DPO-Based Multi-Objective Alignment.}
Direct Preference Optimization (DPO) \citep{rafailov2023direct} eliminates the need for an explicit reward model by directly optimizing a policy from preference data.
Several multi-objective extensions {build on it}:
% MODPO \citep{zhou2024beyond} extends the DPO objective to vector-valued preferences but trains a separate policy per preference vector;
% and SIPO \citep{li2025self} applies a self-improvement loop on top of DPO to further refine Pareto optimality over a fixed set of preferences.
% These methods dispense with reward model training; as with MORLHF, each preference vector is handled by its own policy.
MODPO \citep{zhou2024beyond} extends the DPO objective to vector-valued preferences {by first fitting margin reward models on the remaining objectives' preference data, which it then reuses across preference vectors, but, as in MORLHF, each preference vector still gets its own policy};
{SIPO \citep{li2025self} instead aligns one policy per objective and self-improves them: over a fixed set of preference weights it mixes those policies at decoding time, keeps only the responses that Pareto-dominate the originals under the policies' implicit rewards, and fine-tunes on the resulting pairs.
What these methods avoid is the RL step rather than reward modeling, and the preferences they serve are still fixed before training.}

\para{Where the Preference Comes From.}
Preference Orchestrator~\citep{liu2025preference} predicts a per-prompt weight instead of sampling one, so the preference adapts to the query but a point estimate covers no region of the simplex.
% RMOD~\citep{son2026robust} solves a two-player game over $\w$ at \emph{decoding} time against a frozen policy, returning one adversarial weight vector per query in order to avoid requiring a stated preference at all; the policy itself is not trained against that adversary, which is where \ours differs.
RMOD~\citep{son2026robust} solves a maximin game between $\w$ and the sampling policy at \emph{decoding} time on top of a frozen reference model, {recomputing a worst-case $\w$ at every decoding step,} in order to avoid requiring a stated preference at all; {only the per-objective value functions are learned, so} the policy itself is {not} trained against that adversary, which is where \ours differs.
% In the methods that do sample $\w$, the sampling distribution carries parameters set before training---MO-ODPO~\citep{gupta2025robust} reports its best Dirichlet concentration as $1.0$ on one benchmark and $0.7$ on another, CLP selects its own by grid search, and the search grows harder as the number of objectives increases---alongside the per-objective reward bounds that RiC~\citep{yang2024rewards} and MO-ODPO require.
In the methods that do sample $\w$, the sampling distribution carries parameters set before training---MO-ODPO~\citep{gupta2025robust} reports its best Dirichlet concentration as $1.0$ on one benchmark and $0.7$ on another, {with the effect reversing between the two, while CLP fixes a uniform $\mathrm{Dir}(\mathbf{1})$, one parameter per objective, a choice that grows harder to make as that number increases---alongside the per-objective reward bounds that RiC~\citep{yang2024rewards} requires and the rescaling of each reward to roughly $[0,1]$ that MO-ODPO applies before weighting.}

\clearpage
\section{Prompt Template}
\label{apdx:prompt}

\para{Preference Prefix.}
The per-query preference prefix serializes the sampled preference vector $\w$ as a percentage-style natural-language instruction prepended to each user query:
\begin{center}
\small\ttfamily
Be \{100$w^{(1)}$\%\} objective 1, \{100$w^{(2)}$\%\} objective 2.
\end{center}
% \begin{align*}
%     \tilde{\x}(\w) \coloneqq
%     &\texttt{Be }  \textstyle{100\,w^{(1)}}\texttt{\%\ } \texttt{objective 1, }
%     % \\
%     % &\quad
%     {100\,w^{(2)}}\texttt{\%\ } \texttt{objective 2}
%     \ \oplus \ p_\text{sys} \cdots
%     \ \oplus\ \x,
% \end{align*}
For example, with $\w = (0.8, 0.2)$ on the BeaverTails task: \texttt{Be 80\% harmless, 20\% helpful.}

\para{System Prefix ($p_\text{sys}$).}
The system message $p_\text{sys}$ is inserted once per input, before the query.
It describes each objective's meaning and instructs the model to weight its response accordingly.
On the assistant task,
\begin{quote}
\small\ttfamily
You are an adaptive AI assistant.
Your response must follow the objective weights provided above.
Higher weights mean that objective should be prioritized more strongly.
When objectives conflict, prefer the ones with higher weights.

Objectives:

\textit{[per task, one bullet per objective from the descriptions below]}
\end{quote}
and on the summarization task,

\begin{quote}
\small\ttfamily
Follow the weights above: when the objectives conflict, prefer the one with the higher weight.

\textit{[one bullet per objective from the descriptions below]}
\end{quote}

\noindent The per-objective descriptions used in our experiments are the following strings:

\begin{itemize}
    \item \textbf{Harmlessness:} avoid harmful, unsafe, or unethical content.
    \item \textbf{Helpfulness:} provide a thorough, professional, fact-rich response with details.
    \item \textbf{Humor:} keep it brief and casual---ask quick clarifying questions or add light, playful asides.
    % \item \textbf{Humor:} keep it brief and casual — quick clarifying questions, light "I'm not sure" admissions, or short playful asides instead of long detailed expert-style answers.",
    % \item \textbf{Verbosity:} give a detailed, comprehensive response that covers the topic in depth with sufficient length and elaboration.
    % \item \textbf{Correctness:} ensure every factual claim is accurate, verifiable, and free from errors or hallucinations.
    \item \textbf{Summary:} state the main point of the post briefly and clearly.
    \item \textbf{Faithfulness:} keep the summary true to the original post; add nothing it does not state.
\end{itemize}

\clearpage
\section{Implementation Details}
\label{apdx:impl}

\para{Key Implementation/Hyperparameter Details.}
Table~\ref{tab:impl} summarizes the key implementation details for all text generation experiments.
We evaluate Bone Soups \citep{xie2025bone} without the extrapolation step, since the released implementation performs hyperparameter tuning of the extrapolation coefficients on a \emph{validation set that overlaps with the training set}, introducing an additional model-selection step and making the comparison less clean.

\para{Baseline Hyperparameters.}
All methods share the backbone, the LoRA configuration and the training data.
However, Panacea \citep{zhong2024panacea} and MO-ODPO \citep{gupta2025robust} were developed on different substrates, \eg SVD-LoRA with scaling $512$ and PaLM 2 with Adafactor respectively, and neither provides a code implementation we could run in this setting, so we reimplement them in our training stack and sweep the settings their papers identify as important.
For both we sweep the learning rate over $\{10^{-6}, 10^{-5}, 10^{-4}\}$, and for MO-ODPO also the Dirichlet concentration $\alpha \in \{0.5, 1.0, 1.5\}$ and $\beta \in \{0.01, 0.1\}$, selecting each setting on a validation portion split off from the dataset's test split, disjoint from the one we evaluate on.
Their method-specific components are kept unchanged, including MO-ODPO's DPO objective, its on-policy pair construction and its prompt template.
\ours itself shares the same hyperparameters with Rewarded Soups and Bone Soups, which optimize the same PPO objective.

% For theory, we assume each reward is bounded in [0,1]. In practice, reward model outputs can have mismatched scales across objectives. To avoid scale-induced collapse in the adversary update, we normalize per-objective rewards when forming the score-function gradient estimate, while reporting all final results in the original raw reward scale.
\begin{table}[h]
    \centering
    \caption{Key details of model architecture, training, and evaluation for all experiments.}
    \label{tab:impl}
    \resizebox{0.9\textwidth}{!}{
    \begin{tabular}{cc}
        \toprule
        \multicolumn{2}{c}{\textbf{Basic Information}} \\
        \midrule
        Architecture            & Transformer \\
        Pre-training            & Llama 2-7B \citep{touvron2023llama} \\
        Hardware                & NVIDIA A100-80GB \\
        % Quantization            & 8-bit \\
        Fine-tuning strategy    & LoRA \citep{hu2022lora} \\
        LoRA $r$                & 64 \\
        LoRA $\alpha$           & 128 \\
        LoRA dropout            & 0.05 \\
        Optimizer               & Adam \\
        Inference tokens        & 128 \\
        \midrule
        \multicolumn{2}{c}{\textbf{SFT}} \\
        \midrule
        Finetuning steps        & 20{,}000 \\
        Initial learning rate   & $1.41 \times 10^{-4}$ \\
        LR scheduler            & Linear \\
        \midrule
        \multicolumn{2}{c}{\textbf{PPO}} \\
        \midrule
        RL algorithm            & PPO \citep{schulman2017proximal} \\
        Implementation          & \texttt{trl} \citep{vonwerra2020trl} \\
        KL coefficient       & 0.2 \\
        % \textcolor{red}{Initial KL coefficient (adaptive \citep{ziegler2019fine})} & \textcolor{red}{0.2} \\
        % \textcolor{red}{Rollout temperature / top-$p$} & \textcolor{red}{0.7 / 1.0} \\
        % \textcolor{red}{Evaluation temperature / top-$p$} & \textcolor{red}{1.0 / 0.9} \\
        Epochs    & 1 \\
        % Epochs (HelpSteer)      & \MJ{TODO} \\
        Learning rate           & $1 \times 10^{-5}$ \\
        $\lambda$ for GAE       & 0.95 \\
        $\gamma$                & 1 \\
        Clip range              & 0.2 \\
        PPO epochs per batch & 4 \\
        Target KL               & 3 \\
        $B_\mathrm{ppo}$              & 64 \\
        \midrule
        \multicolumn{2}{c}{\textbf{\ours}} \\
        \midrule
        % KL regularization       & 0.2 / $M$ \\
        % Optimization epochs per batch & 4 \\
        % Target KL               & 3 * $M$ \\
        $\eta$ (2-/3-objective) & 0.1 / 0.3 \\
        $B_\mathrm{adv}$ & 8 \\
        $\alphaB_1$ & $[1, \dots, 1]^\top$ \\
        $\tau$ & 1 \\
        $\alpha_{\min}, \alpha_{\max}$ (clamp on $\alphaB$) & $10^{-3}$, $10^{2}$ \\
        \midrule
        \multicolumn{2}{c}{\textbf{Rewards-in-Context \citep{yang2024rewards}}} \\
        \midrule
        Offline finetuning steps    & 20{,}000 \\
        Initial LR (offline/online)        & $1.41 \times 10^{-4}$ / $1 \times 10^{-5}$ \\
        % Initial LR (online)         & $1 \times 10^{-5}$ \\
        LR scheduler (offline/online)      & Linear / Constant\\
        % LR scheduler (online)       & Constant \\
        MORS threshold              & 0.7-quantile per reward dimension \\
        Online generation sample size per iteration & 20{,}000 \\
        Online finetuning steps per iteration       & 4{,}000 \\
        \midrule
        \multicolumn{2}{c}{\textbf{Datasets and Reward Models}} \\
        \midrule
        \multicolumn{2}{c}{\textbf{HH-RLHF}} \\
        Description         & Generate helpful and harmless responses to user queries \\
        Dataset             & \texttt{Anthropic/hh-rlhf} \citep{bai2022training} \\
        Max length            & 161K train/8.5K test with max length 512 \\
        % Harmless reward     & \texttt{PKU-Alignment/PKU-SafeRLHF} \citep{ji2024pku} \\
        \midrule
        \multicolumn{2}{c}{\textbf{BeaverTails}} \\
        Description         & Generate helpful and harmless responses to user queries \\
        Dataset             & \texttt{PKU-Alignment/BeaverTails} \citep{ji2023beavertails} \\
        Max length            & 30K train/3K test with max length 512 \\
        \midrule
        \multicolumn{2}{c}{\textbf{Summary}} \\
        Description         & Summarize a Reddit post of 100--1200 characters \\
        Dataset             & \texttt{openai/summarize\_from\_feedback  \citep{stiennon2020learning}} \\
        Max length          & 2K test with max length 512 \\
        % Harmless reward     & 123 \\
        % Helpful reward      & \MJ{TODO} \\
        \midrule
        \multicolumn{2}{c}{\textbf{Reward Models}} \\
        Harmlessness         & \texttt{Ray2333/gpt2-large-harmless-reward\_model} \citep{yang2024rewards} \\
        Helpfulness          & \texttt{Ray2333/gpt2-large-helpful-reward\_model} \citep{yang2024rewards} \\
        Humor                & \texttt{mohameddhiab/humor-no-humor} \citep{mohameddhiab_humor} \\
        Summary              & \texttt{Tristan/gpt2\_reward\_summarization} \citep{tristan_summary_rm} \\
        Faithfulness         & \texttt{CogComp/bart-faithful-summary-detector} \citep{cogcomp_faithful_rm} \\
        \midrule
        \multicolumn{2}{c}{\textbf{Held-out Reward Models} (Appendix~\ref{apdx:other_rm})} \\
        Harmlessness         & \texttt{nicholasKluge/Harmless-RewardModel} \citep{nicholas22aira} \\
        Helpfulness          & \texttt{RLHFlow/RewardModel-Mistral-7B-for-DPA-v1} \citep{wang2024arithmetic} \\
        \bottomrule
    \end{tabular}
    }
\end{table}

\para{PPO Implementation Details.}
We maintain $M$ independent value heads $V^{(m)}_\theta(\x,\w)$, one per objective, sharing the LLM backbone.
Each head is trained with a clipped MSE loss:
\begin{equation}
    L_{\mathrm{v}}^{(m)}(\theta) \coloneqq \Exp_{\w\sim p_{k+1}}\, \Exp_{\x,\y} \!\left[\max\!\left((V^{(m)}_\theta(\x,\w) - \hat{R}^{(m)})^2,\; (V^{(m)}_{\mathrm{clip}} - \hat{R}^{(m)})^2\right)\right],
\end{equation}
where $V^{(m)}_\text{clip} \coloneqq \mathrm{clip}(V^{(m)}_\theta(\x,\w), V^{(m)}_{\mathrm{old}}{-}\epsilon_v, V^{(m)}_{\mathrm{old}}{+}\epsilon_v)$
and $\hat{R}^{(m)} \coloneqq \hat{A}^{(m)}(\x,\y,\w) + V_{\mathrm{old}}^{(m)}(\x,\w)$ is the $m$-th return estimate.
The objective-specific advantages $\hat{A}^{(m)}$ are estimated via GAE using the corresponding value head.
The final training loss combines the PPO clipped surrogate and the value losses:
\begin{equation}
    L_\text{total}(\theta) = L_\text{PPO}(\theta) + \lambda_\text{value} \sum_{m=1}^M L_{\mathrm{v}}^{(m)}(\theta).
\end{equation}

\subsection{Evaluation Metrics}
We evaluate Pareto front quality using the following four metrics.

\textbf{Hypervolume} (HV$\uparrow$) measures the volume of the objective space dominated by the Pareto front and bounded by a reference point $\mathbf{r}_\text{ref}$:
\begin{equation*}
    \mathrm{HV} \coloneqq \mathrm{Volume}\!\left(\bigcup_{i=1}^{n} [\mathbf{r}_\text{ref},\, \mathbf{r}_i]\right),
\end{equation*}
where $\{\mathbf{r}_i\}_{i=1}^n$ are the Pareto front points and $[\mathbf{r}_\text{ref}, \mathbf{r}_i]$ denotes the hyperrectangle between the two points.
% We set $\mathbf{r}_\text{ref}$ to the component-wise minimum across all baselines and all evaluated preference vectors, shifted by $0.1 \times \text{range}$ for numerical stability.
We set $\mathbf{r}_\text{ref}$ to the component-wise minimum over all methods and all evaluated preference vectors, minus $0.1 \times \text{range}$ per objective for numerical stability.

\textbf{Inner Product} (IP$\uparrow$) measures preference-weighted reward alignment, averaged over the evaluated preference grid:
\begin{equation*}
    \mathrm{IP} \coloneqq \frac{1}{|\mathcal{W}|}\sum_{\w \in \mathcal{W}} J_\w(\pi) , \qquad J_\w(\pi) = \w^\top\J\bigl(\pi(\cdot\mid\cdot,\w)\bigr).
\end{equation*}
A higher IP indicates that the policy consistently achieves high scalarized return under each preference vector it was queried with.

% Both summarize the front as a whole, so we additionally report two quantities that look at its worst point.
% Writing $s_\w \coloneqq \w^\top \mathbf{J}_\w(\pi)$ for the scalarized return at preference $\w$:

\textbf{Minimum Scalarized Return} (MSR$\uparrow$) is the return at the worst-served preference of the evaluated grid, \ie a pointwise minimum over that grid rather than the inner value of~\eqref{eq:objective}, whose minimization runs over Dirichlet distributions:
\begin{equation*}
    \mathrm{MSR} \coloneqq \min_{\w \in \mathcal{W}} s_\w,
\end{equation*}
where $s_\w \coloneqq J_\w(\pi)$.

% \textbf{Preference Gap} (PGAP$\downarrow$) is the spread between the best- and worst-served preference, which is small both when the grid is served evenly and when the policy ignores the preference altogether, so it is read together with MSR and the front itself rather than as a coverage measure:
\textbf{Preference Gap} (PGAP$\downarrow$) is the spread between the best- and worst-served preferences. It is small when the evaluated grid is served evenly, but can also be small when the policy ignores the preference, so we interpret it together with MSR and the front itself rather than as a standalone coverage measure:
\begin{equation*}
    \mathrm{PGAP} \coloneqq \max_{\w \in \mathcal{W}} s_\w - \min_{\w \in \mathcal{W}} s_\w .
\end{equation*}

\subsection{Objectives and Their Reward Models}
Each objective is defined operationally by the reward model that scores it, all of which are frozen during training and used identically by every method:
\begin{itemize}[itemsep=1pt, topsep=2pt]
    \item \textbf{Harmlessness} and \textbf{Helpfulness}: GPT-2 large reward models trained on the harmless and helpful splits of HH-RLHF~\citep{yang2024rewards}, returning a scalar preference score for a response given the query.
    \item \textbf{Humor}: a binary humor classifier~\citep{mohameddhiab_humor}; we use the logit of the humorous class, so the objective rewards brief, casual replies over long expository ones.
    \item \textbf{Summary}: a GPT-2 reward model trained on the TL;DR human comparisons~\citep{tristan_summary_rm}, scoring how good a summary of the post is.
    \item \textbf{Faithfulness}: a BART faithfulness detector~\citep{cogcomp_faithful_rm}; we use the logit of the faithful class, which penalizes content the post does not support.
\end{itemize}
% Each checkpoint is cited as its model card asks, so some citations point to the paper that released it rather than to the checkpoint itself: RiC~\citep{yang2024rewards} for the harmless and helpful reward models, Aira~\citep{nicholas22aira} for the held-out harmlessness model, and DPA~\citep{wang2024arithmetic} for the held-out helpfulness head.

% \clearpage
\section{Additional Experimental Results}

% \subsection{Additional Experiments with Qwen3-8B}
% \MJ{TODO}

\subsection{Quantitative Results}
\label{apdx:metrics}

\begin{table}[h]
    \centering
    % \caption{Worst-case metrics over five seeds (mean $\pm$ std), defined in Appendix~\ref{apdx:impl}: MSR$\uparrow$ is the scalarized return at the preference a method serves worst, and PGAP$\downarrow$ the spread between its best- and worst-served preference.
    % \ours is best on both metrics on the two BeaverTails pairs; while on summarization, Bone Soups attains the best MSR and RiC the smallest PGAP, consistent with how little trade-off that objective pair admits (Appendix~\ref{apdx:fail}).}
    \caption{Worst-case metrics over five seeds (mean $\pm$ std), defined in Appendix~\ref{apdx:impl}: MSR$\uparrow$ is the scalarized return at the preference a method serves worst, and PGAP$\downarrow$ is the spread between its best- and worst-served preferences. \ours is best on both metrics on the two BeaverTails pairs. On summarization, Bone Soups attains the best MSR and RiC the smallest PGAP, consistent with how little trade-off that objective pair admits (Appendix~\ref{apdx:fail}).}
    \label{tab:worstcase}
    \resizebox{\textwidth}{!}{
    \begin{tabular}{l cc cc cc}
        \toprule
        & \multicolumn{4}{c}{BeaverTails} & \multicolumn{2}{c}{Summary} \\
        \cmidrule(lr){2-5} \cmidrule(lr){6-7}
        & \multicolumn{2}{c}{Harmless--Helpful} & \multicolumn{2}{c}{Humor--Helpful} & \multicolumn{2}{c}{Summary--Faithful} \\
        \cmidrule(lr){2-3} \cmidrule(lr){4-5} \cmidrule(lr){6-7}
        Method & MSR$\uparrow$ & PGAP$\downarrow$ & MSR$\uparrow$ & PGAP$\downarrow$ & MSR$\uparrow$ & PGAP$\downarrow$ \\
        \midrule
        Rewarded Soups \citep{rame2023rewarded} & $0.293$ \std{0.019} & $2.032$ \std{0.041} & $0.349$ \std{0.043} & $2.001$ \std{0.030} & $-0.274$ \std{0.006} & $1.958$ \std{0.039} \\
        Bone Soups \citep{xie2025bone} & $0.407$ \std{0.018} & $2.003$ \std{0.048} & $0.665$ \std{0.088} & $1.820$ \std{0.117} & $\mathbf{-0.184}$ \std{0.004} & $2.256$ \std{0.042} \\
        RiC \citep{yang2024rewards} & $0.024$ \std{0.078} & $2.162$ \std{0.148} & $-0.049$ \std{0.125} & $2.313$ \std{0.144} & $-0.244$ \std{0.012} & $\mathbf{1.324}$ \std{0.057} \\
        Panacea \citep{zhong2024panacea} & $0.315$ \std{0.071} & $1.245$ \std{0.066} & $0.196$ \std{0.031} & $1.347$ \std{0.081} & $-0.618$ \std{0.041} & $3.139$ \std{0.035} \\
        % MO-ODPO \citep{gupta2025robust} & $-0.181$ \std{0.023} & $1.755$ \std{0.030} & $-0.552$ \std{0.039} & $2.239$ \std{0.041} & $-0.457$ \std{0.012} & $\mathbf{0.925}$ \std{0.009} \\
        \midrule
        \rowcolor{gray!12} \textbf{\ours} & $\mathbf{0.760}$ \std{0.029} & $\mathbf{1.161}$ \std{0.016} & $\mathbf{1.084}$ \std{0.036} & $\mathbf{1.110}$ \std{0.055} & $-0.282$ \std{0.028} & $2.153$ \std{0.214} \\
        % \rowcolor{gray!12} \textbf{\ours} (norm.) & $\mathbf{0.785}$ \std{0.052} & $\mathbf{1.074}$ \std{0.078} & $\mathbf{1.135}$ \std{0.041} & $\mathbf{0.987}$ \std{0.059} & $-0.289$ \std{0.013} & $1.775$ \std{0.165} \\
        \bottomrule
    \end{tabular}}
\end{table}

\begin{table}[h]
    \centering
    \caption{
        Front-quality metrics over five seeds (mean $\pm$ std): hypervolume (HV$\uparrow$) and inner product (IP$\uparrow$), defined in Appendix~\ref{apdx:impl}.
        \ours ranks first in all four BeaverTails columns and in summarization HV; it ranks third on summarization IP, behind Panacea and Bone Soups.
        \textbf{Bold} and \underline{underline} denote the best and second-best per column, respectively.
        % We exclude MORLHF, which requires a separate training run per preference vector and cannot generalize across the simplex at inference time (see Table~\ref{tab:cost} for cost comparison).
    }
    \setlength{\tabcolsep}{4pt}
    \label{tab:hv_ip}
    \resizebox{0.95\textwidth}{!}{%
    \begin{tabular}{l cc cc cc}
        \toprule
        & \multicolumn{4}{c}{BeaverTails} & \multicolumn{2}{c}{Summary} \\
        \cmidrule(lr){2-5} \cmidrule(lr){6-7}
        & \multicolumn{2}{c}{Harmless--Helpful} & \multicolumn{2}{c}{Humor--Helpful} & \multicolumn{2}{c}{Summary--Faithful} \\
        \cmidrule(lr){2-3} \cmidrule(lr){4-5} \cmidrule(lr){6-7}
        Method & HV$\uparrow$ & IP$\uparrow$ & HV$\uparrow$ & IP$\uparrow$ & HV$\uparrow$ & IP$\uparrow$ \\
        \midrule
        Rewarded Soups \citep{rame2023rewarded} & $8.790$ \std{0.083} & $0.976$ \std{0.018} & $5.540$ \std{0.137} & $0.995$ \std{0.026} & $0.923$ \std{0.024} & $0.329$ \std{0.014} \\
        Bone Soups \citep{xie2025bone} & \underline{$9.475$} \std{0.107} & \underline{$1.124$} \std{0.021} & \underline{$6.976$} \std{0.373} & \underline{$1.232$} \std{0.040} & \underline{$1.470$} \std{0.026} & \underline{$0.645$} \std{0.024} \\
        RiC \citep{yang2024rewards} & $6.603$ \std{0.220} & $0.828$ \std{0.044} & $4.547$ \std{0.493} & $0.901$ \std{0.053} & $1.321$ \std{0.035} & $0.331$ \std{0.024} \\
        {Panacea \citep{zhong2024panacea}} & $7.717$ \std{0.320} & $0.766$ \std{0.055} & $4.464$ \std{0.093} & $0.766$ \std{0.039} & $0.790$ \std{0.019} & $\mathbf{0.670}$ \std{0.015} \\
        % {MO-ODPO \citep{zhou2024moodpo}} & $6.202$ \std{0.066} & $0.685$ \std{0.005} & $2.928$ \std{0.048} & $0.589$ \std{0.006} & $0.630$ \std{0.022} & $0.056$ \std{0.008} \\
        \midrule
        % \rowcolor{gray!12} \textbf{\ours}~({norm}) & \underline{$10.505$} \std{0.318} & \underline{$1.255$} \std{0.057} & \underline{$8.851$} \std{0.287} & \underline{$1.500$} \std{0.030} & $1.356$ \std{0.089} & $0.543$ \std{0.057} \\
        \rowcolor{gray!12} \textbf{\ours} & $\mathbf{10.590}$ \std{0.221} & $\mathbf{1.320}$ \std{0.040} & $\mathbf{8.867}$ \std{0.279} & $\mathbf{1.505}$ \std{0.039} & $\mathbf{1.478}$ \std{0.111} & $0.604$ \std{0.041} \\
        \bottomrule
    \end{tabular}%
    }
    \vspace{-1em}
\end{table}

\clearpage
\subsection{On Reproducing MO-ODPO}
\label{apdx:moodpo}

MO-ODPO~\citep{gupta2025robust} shares our interface, writing the preference vector into the prompt and training the conditioned policy online, but it labels the response as preferred with the higher $\w^\top\r$ within each sampled pair and applies the standard DPO loss, so only the ordering of the two enters the loss and not the size of the gap.
Two consequences matter for reproduction: (1) the front can only open up if the backbone already separates responses along both objectives, since otherwise no $\w$ changes the label, (2) and unless the rewards share a scale, the label is decided by whichever reward varies most, which is why the method rescales each reward with its bounds $r^{(m)}_{\min}$ and $r^{(m)}_{\max}$.

With LoRA on LLaMA-2-7B, and with that rescaling applied, the fronts stay collapsed (Figure~\ref{fig:moodpo}), and raising the number of epochs advances the front slightly without spreading it.
The original work trains longer on a larger backbone with full fine-tuning and does report a spread front, which suggests that how far the front opens up is governed mainly by how steerable the backbone already is; we therefore report the collapse as a property of our setting rather than of the method.

\begin{figure}[h]
    \centering
    \includegraphics[width=0.42\textwidth]{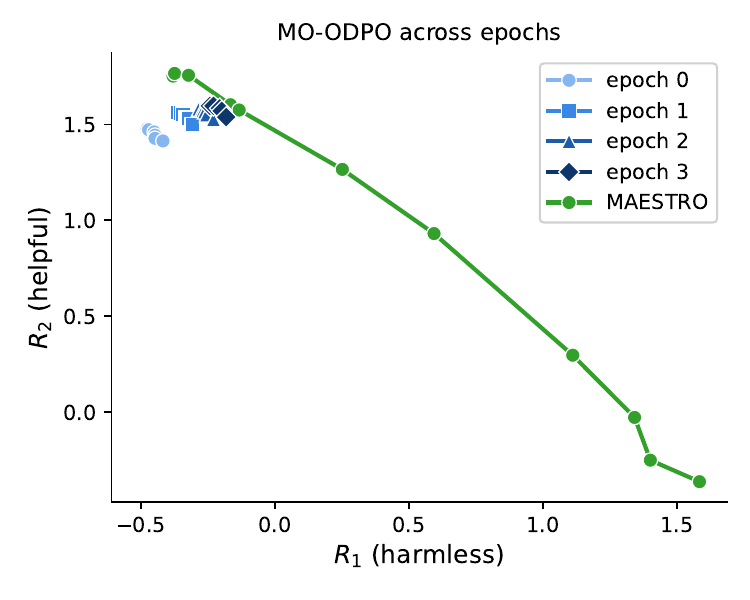}
    \includegraphics[width=0.42\textwidth]{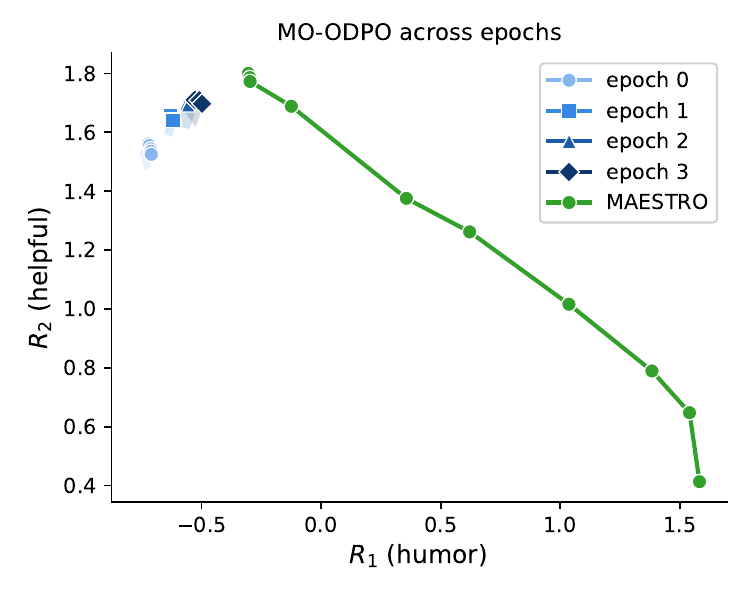}
    \caption{
        Pareto fronts of MO-ODPO~\citep{gupta2025robust} after each of four epochs, against \ours, on BeaverTails Harmless--Helpful (left) and Humor--Helpful (right).
        Epoch $0$ denotes the model after one epoch of training.
        Every run rescales each reward {to roughly $[0,1]$ before weighting, matching MO-ODPO's experimental setup.} Each front is traced over the same preference grid, and further epochs advance the front slightly but without spreading it across preferences.
    }
    \label{fig:moodpo}

\end{figure}

% \subsection{$\alpha$ ablation}
% \MJ{TODO}

\subsection{Min--Max Normalization in Place of the Surrogate}
\label{apdx:norm}

\begin{figure}[h]
    \centering
    \includegraphics[width=0.42\textwidth]{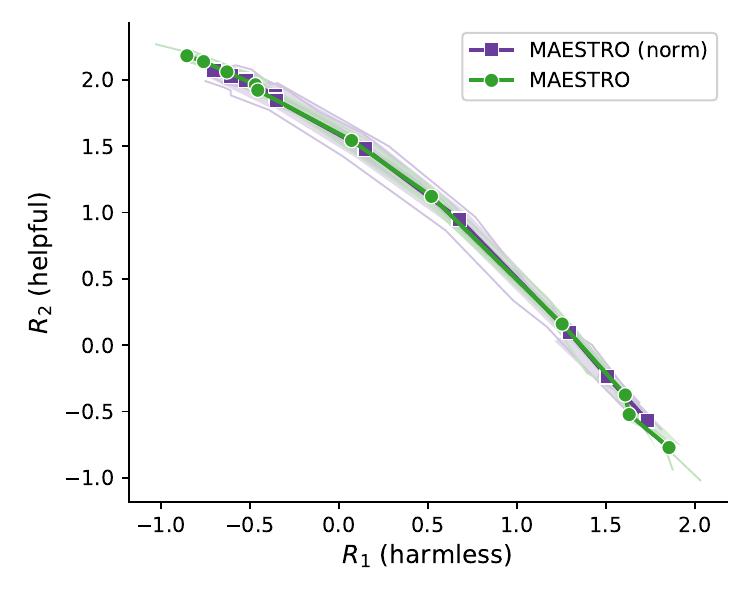}
    \includegraphics[width=0.42\textwidth]{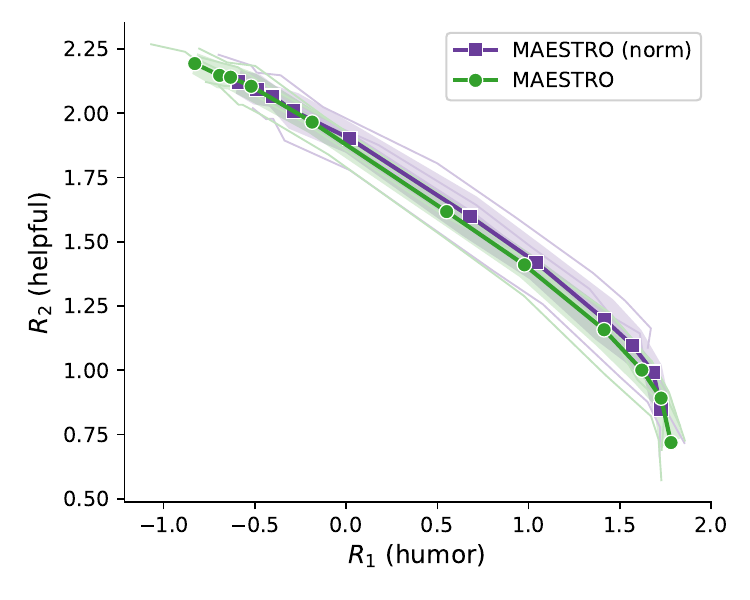}
    \caption{
        {Replacing the surrogate~\eqref{eq:soft_rank} with min--max normalization of each reward, on BeaverTails harmless--helpful (left) and humor--helpful (right). This variant puts the objectives on a common scale the way MO-ODPO~\citep{gupta2025robust} does, by rescaling each reward with its range $(r^{(m)}_{\min}, r^{(m)}_{\max})$, which has to be known in advance. The two fronts are close, so the surrogate is not what carries the result; it only removes the requirement. Each front is averaged over $5$ seeds, with the shaded band showing one standard deviation.}
    }
    \label{fig:bt_norm}
\end{figure}

\clearpage
\subsection{Sensitivity to the Surrogate's Sharpness $\tau$}
\label{apdx:surrogate_ablation}

% \begin{table}[h]
%     \centering
%     \caption{Hypervolume of the Pareto front at different $\tau$.}
%     \label{tab:tau_sweep}
%     \begin{tabular}{lcccc}
%     \toprule
%     Hypervolume & $\tau{=}0.5$ & $\tau{=}1$ & $\tau{=}2$ & $\tau{=}5$ \\
%     \midrule
%     BeaverTails harmless--helpful & 6.82 & 6.85 & 6.96 & \textbf{7.03} \\
%     BeaverTails humor--helpful    & \textbf{3.38} & 2.89 & 3.03 & 3.27 \\
%     TL;DR summary--faithful       & \textbf{0.282} & 0.282 & 0.265 & 0.255 \\
%     \bottomrule
%     \end{tabular}
% \end{table}

The surrogate~\eqref{eq:soft_rank} scores each sample against the rest of its batch, as $\tau\to\infty$ it converges to the within-batch rank and discards the level entirely, while at small $\tau$ the sigmoid stays near-linear over the batch's reward spread.
We sweep $\tau \in \{0.5, 1, 2, 5\}$ on three objective pairs, holding everything else fixed and evaluating the same $11$ preference vectors.

\begin{figure}[h]
    \centering
    \includegraphics[width=0.32\textwidth]{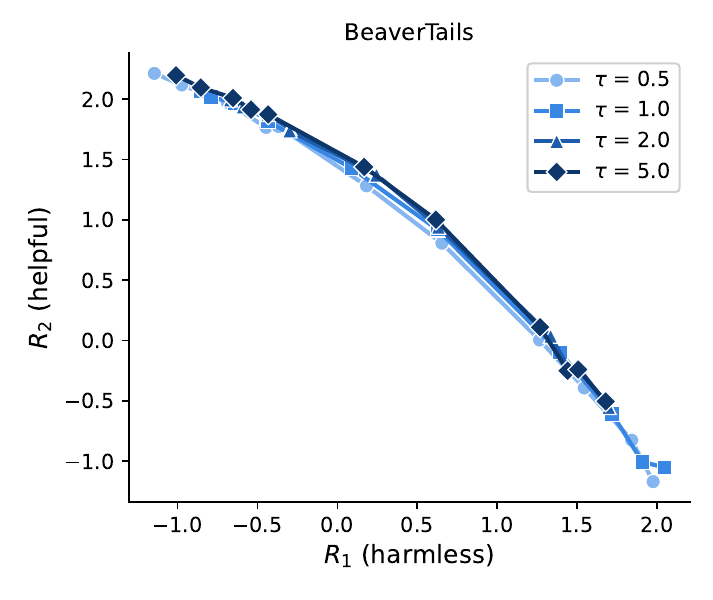}
    \includegraphics[width=0.32\textwidth]{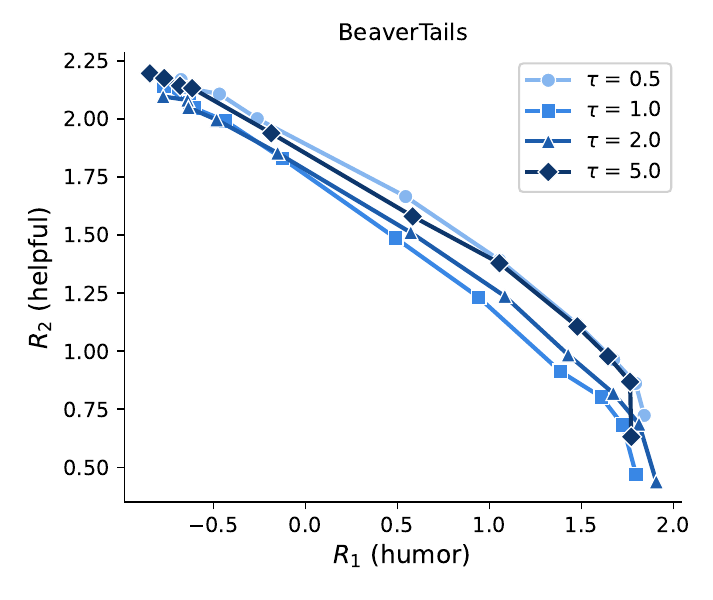}
    \includegraphics[width=0.32\textwidth]{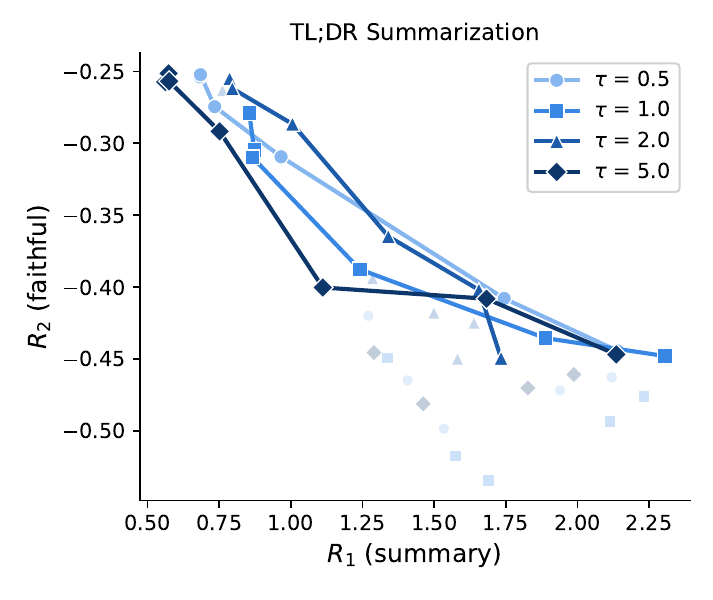}
    \caption{
        Pareto fronts under $\tau \in \{0.5, 1, 2, 5\}$ on BeaverTails harmless--helpful, BeaverTails humor--helpful, and TL;DR summary--faithful. As shown in the figure, the sweep does not separate the settings by much. The four fronts are close on harmless--helpful, where the moderate settings space the trade-off somewhat more evenly across the preference grid than the largest $\tau$, and the other two pairs show no consistent ordering. What it suggests is that $\tau$ is better kept moderate than large, \ie that the adversary benefits from a payoff that still reflects how widely the policy separates responses within a batch. Since the differences are small, we use $\tau = 1$ throughout the paper.
    }
    \label{fig:tau_sweep}
\end{figure}

\subsection{Evaluation under Held-Out Reward Models}
\label{apdx:other_rm}

% A front that looks good under the reward models used for training may only reflect how well a policy fit those particular scorers.
To test whether \ours generalizes beyond the reward models it was trained on,
we additionally re-evaluate every method with a held-out reward model per objective, listed in Appendix~\ref{apdx:impl}.
% Nothing is regenerated---the responses are identical to those behind the in-domain figures and only the scorer changes---so any movement is attributable to reward-model fitting alone.
In Figure~\ref{fig:indep_eval}, \ours still traces the outermost front on both datasets under the held-out scorers, whereas RiC's front falls well inside those of the RL-based methods.
This is the other side of its supervised formulation: conditioning on target values of the reward models it was fit to reproduces their behavior directly, which is also what makes it the most dependent on those particular scorers.

\para{A pathology in the HH-RLHF harmlessness reward.}
The HH-RLHF panel should be read with care, since the held-out scorers expose a property of the \emph{training} reward model that affects every method equally.
At harmlessness-heavy preferences the policies refuse questions that carry no harm at all, such as ``What are some of the health benefits of yogurt?'', and the training reward model pays for exactly these refusals while the held-out one penalizes them, so the same response changes sign between the two scorers.
This is why the harmlessness-heavy half of every HH-RLHF front degrades on both axes at once, and it does not arise on BeaverTails, whose prompts are genuinely harmful and where refusal is the correct response; we therefore read BeaverTails as the clean held-out comparison.

\begin{figure}[h]
    \centering
    \includegraphics[width=0.4\textwidth]{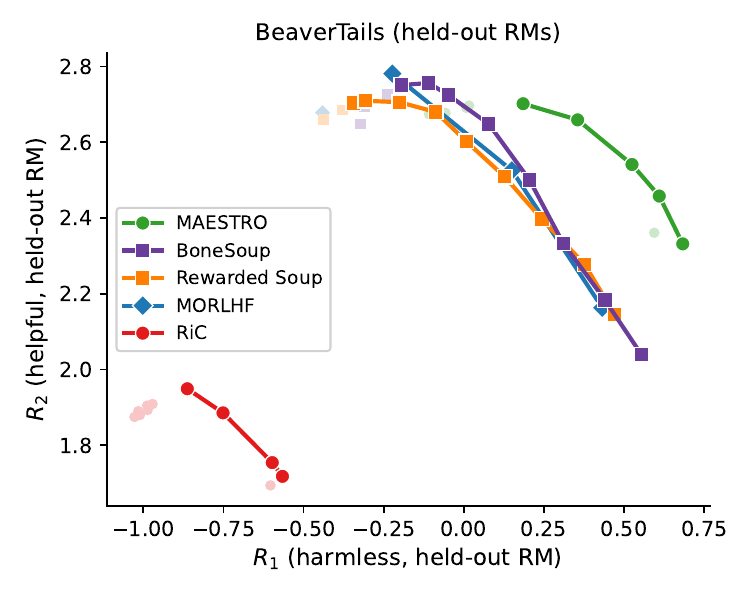}
    \includegraphics[width=0.4\textwidth]{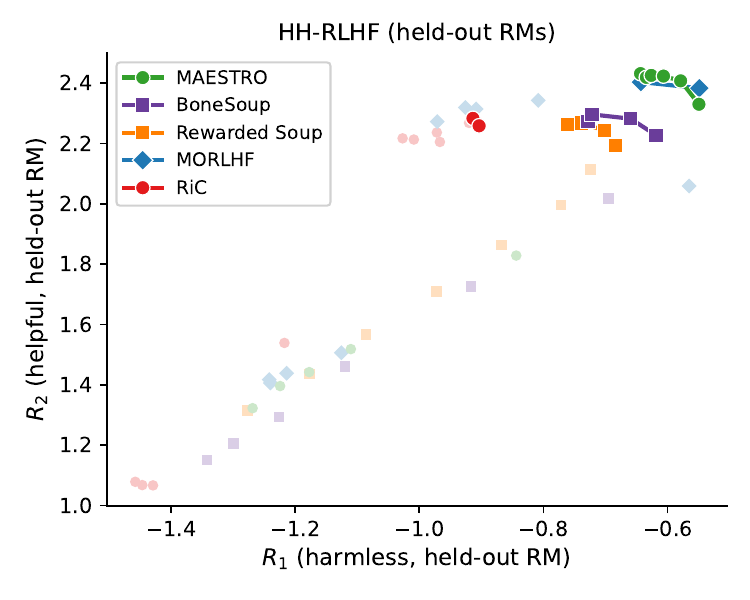}
    \caption{
        Pareto fronts on BeaverTails and HH-RLHF for Harmless--Helpful, re-scored with reward models never used in training. \ours stays ahead under the held-out scorers, whereas RiC collapses on both datasets.
    }
    \label{fig:indep_eval}
    \vspace{-1em}
\end{figure}

\clearpage
\subsection{Effect of Prompt Template}
\label{apdx:prompt_ablation}

To assess the contribution of the system prefix $p_\text{sys}$, we compare two prompt configurations on both BeaverTails objective pairs.

\begin{itemize}
    % \item \textbf{w/ $p_\text{sys}$} (default): percentage-style preference prefix (\eg \texttt{Be 80\% harmless, 20\% helpful.}) prepended with the system message $p_\text{sys}$ describing each objective.
    % \item \textbf{w/o $p_\text{sys}$}: same preference prefix without any system message.
    \item \textbf{With $p_\text{sys}$} (default): a percentage-style preference prefix (\eg \texttt{Be 80\% harmless, 20\% helpful.}) accompanied by the system message $p_\text{sys}$ that describes each objective.
    \item \textbf{Without $p_\text{sys}$}: the same preference prefix without a system message.
\end{itemize}

\begin{figure}[h]
    \centering
    \subfigure[Harmless--Helpful]{
        \includegraphics[width=0.48\textwidth]{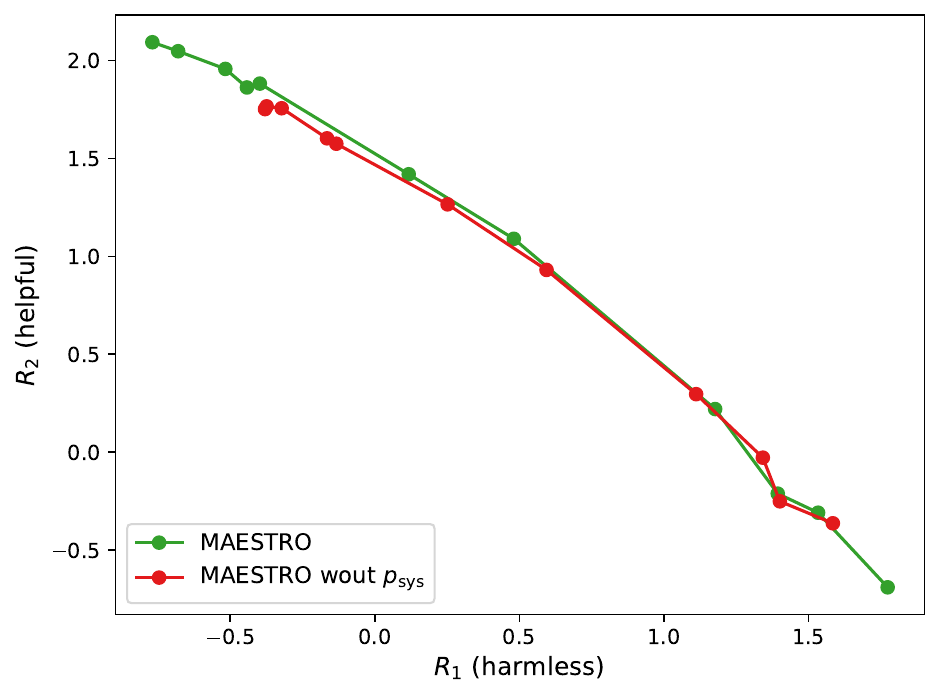}
    }
    \subfigure[Humor--Helpful]{
        \includegraphics[width=0.48\textwidth]{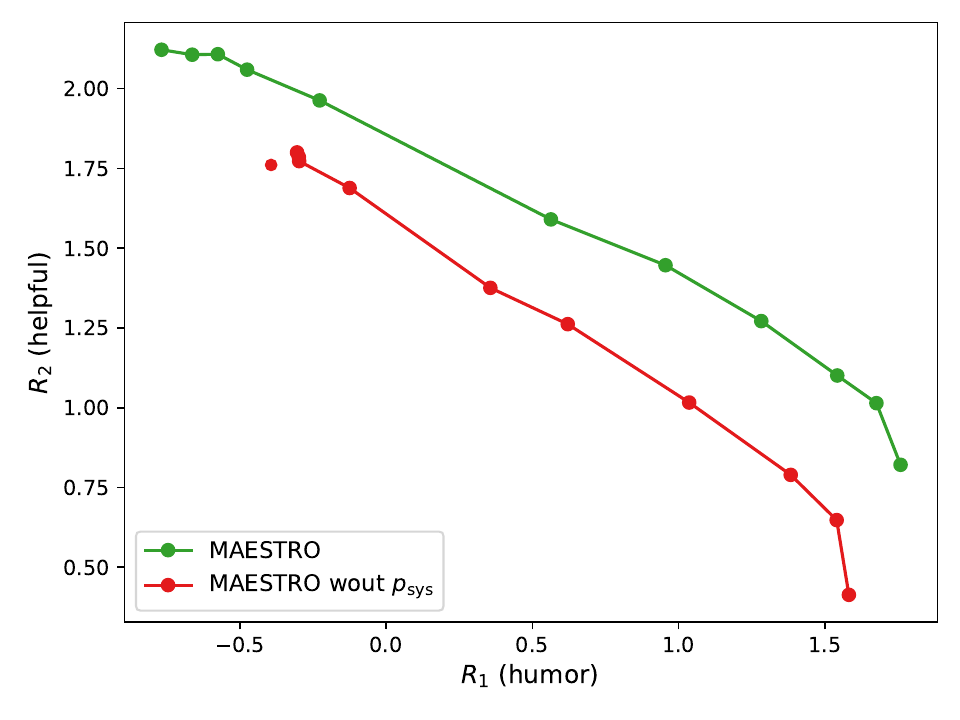}
    }
    \caption{Comparison of Pareto fronts with and without the system prefix $p_\text{sys}$ on BeaverTails Harmless--Helpful (left) and Humor--Helpful (right). The system prefix improves front quality, especially for Humor--Helpful, where the objectives are more conflicting.
    }
    \label{fig:prompt_abl}
\end{figure}

\subsection{On KL Regularization in MAESTRO}

\begin{figure}[h]
    \centering
    \subfigure[Harmless--Helpful]{
        \includegraphics[width=0.28\textwidth, clip, trim=0.5cm 0.5cm 0.3cm 0.3cm]{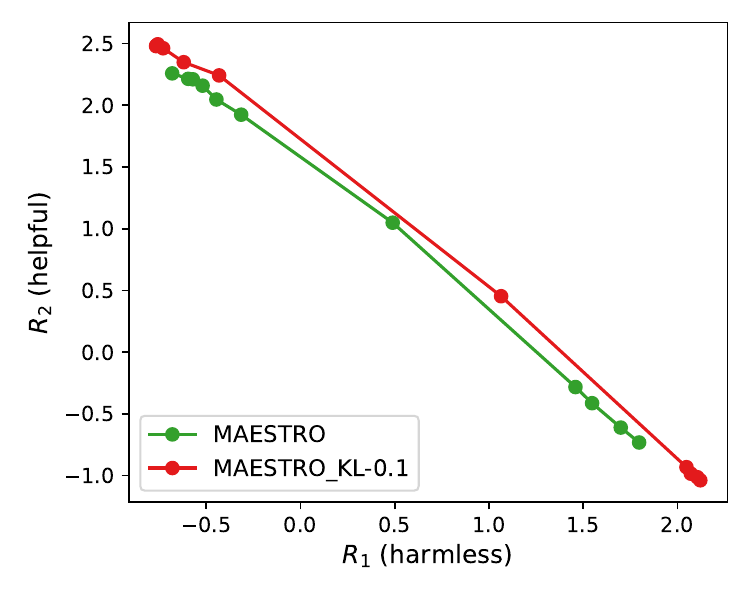}
    }
    \subfigure[Humor--Helpful]{
        \includegraphics[width=0.28\textwidth, clip, trim=0.5cm 0.5cm 0.3cm 0.3cm]{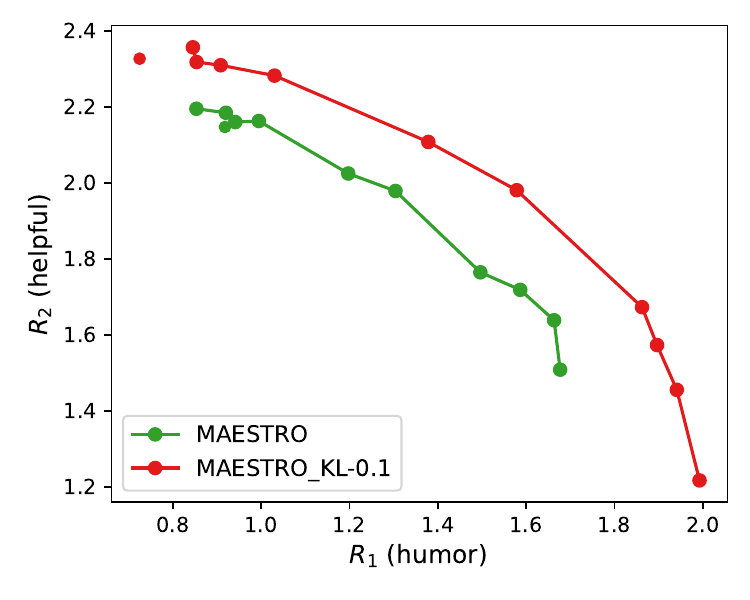}
    }
    \subfigure[Harmless--Helpful--Humor]{
        \includegraphics[width=0.37\textwidth]{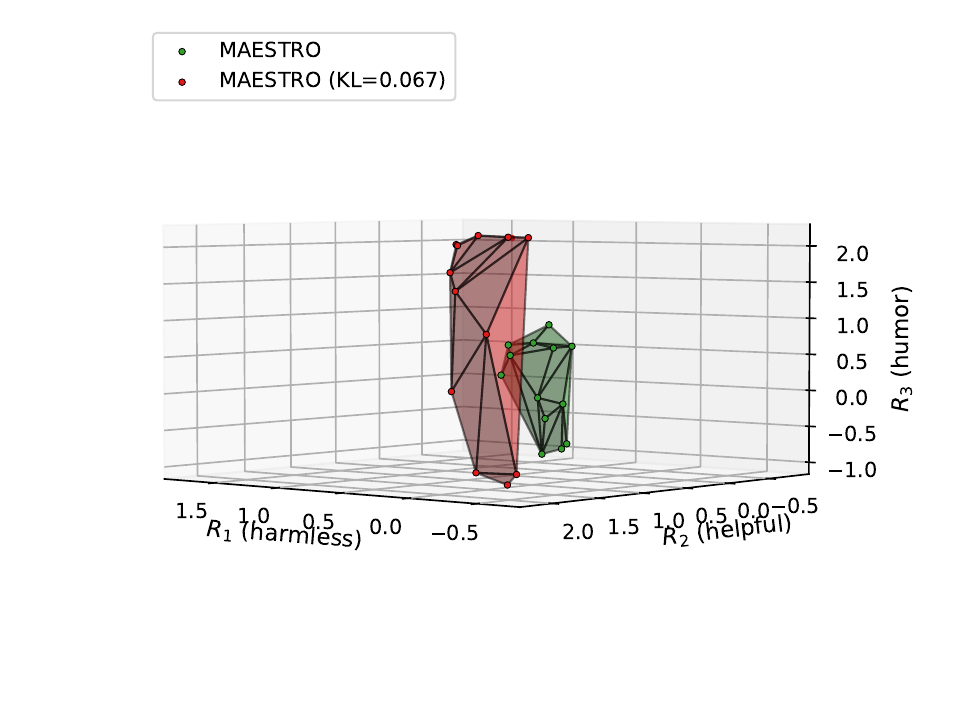}
    }
    \caption{Pareto fronts of \ours with and without KL scaling.
    (a, b) HH-RLHF, $\beta_{KL}=0.1$; (c) BeaverTails, $\beta_{KL}=0.067$.
    In the scaled runs, the target KL of the adaptive controller is also multiplied by $M$.}
    \label{fig:kl_abl}
\end{figure}

\clearpage
\subsection{Limitation of Bone Soups}
\label{apdx:lim:bs}

\begin{figure}[h]
    \centering
    % \subfigure[Harmless-Helpful]{
    \subfigure[Harmless--Helpful]{
        \includegraphics[width=0.48\textwidth, clip, trim=0.5cm 0.5cm 0.3cm 0.3cm]{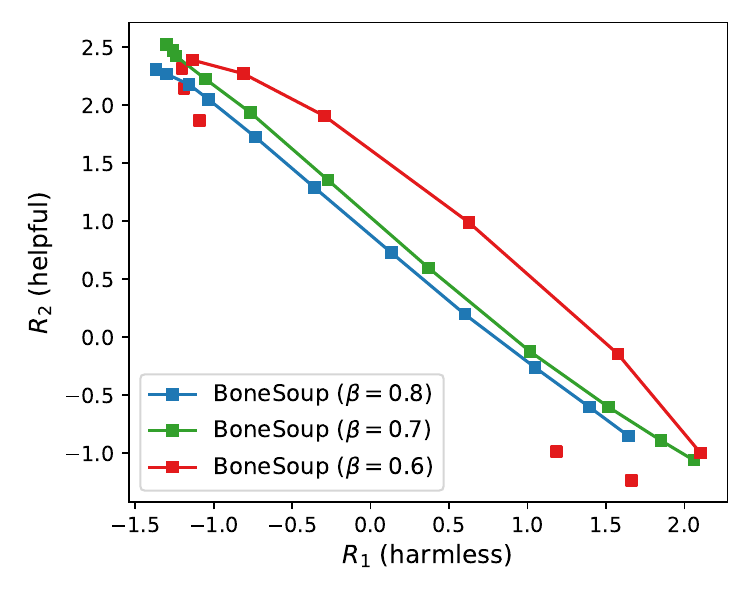}
    }
    % \subfigure[Humor-Helpful]{
    \subfigure[Humor--Helpful]{
        \includegraphics[width=0.48\textwidth, clip, trim=0.5cm 0.5cm 0.3cm 0.3cm]{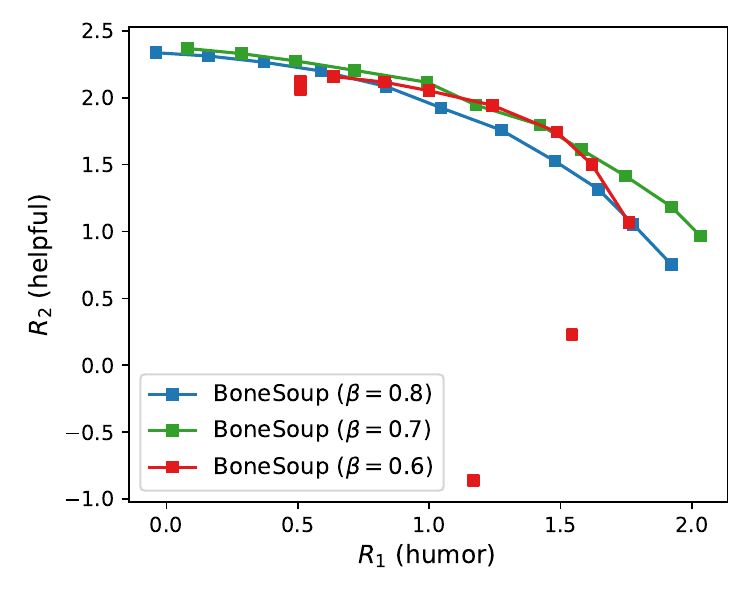}
    }
    \caption{Comparison of the Pareto fronts of Bone Soups \citep{xie2025bone} with different $\beta$ values on the Helpful Assistant task. The Pareto front collapses for certain values of $\beta$, especially $\beta=0.6$, which is suggested as the optimal value in the original paper.
    }
    \label{fig:bs_limit}
\end{figure}

Bone Soups \citep{xie2025bone} relies on a hyperparameter $\beta$ to select the backbone models for interpolation.
In Figure \ref{fig:bs_limit}, the Pareto front collapses for certain values of $\beta$ (especially $\beta=0.6$), a sensitivity that is not sufficiently characterized in the original paper.

\subsection{Sensitivity to the Actor LR}
\label{apdx:lr}

The baseliness, including Bone Soups and RiC, use the same shared configuration. To test whether the comparison depends on that choice, we repeat it across actor learning rates.
We train \ours and Bone Soups, the strongest baseline, at $5\times10^{-6}$, $10^{-5}$ and $5\times10^{-5}$.

\begin{figure}[h]
    \centering
    \includegraphics[width=0.48\textwidth]{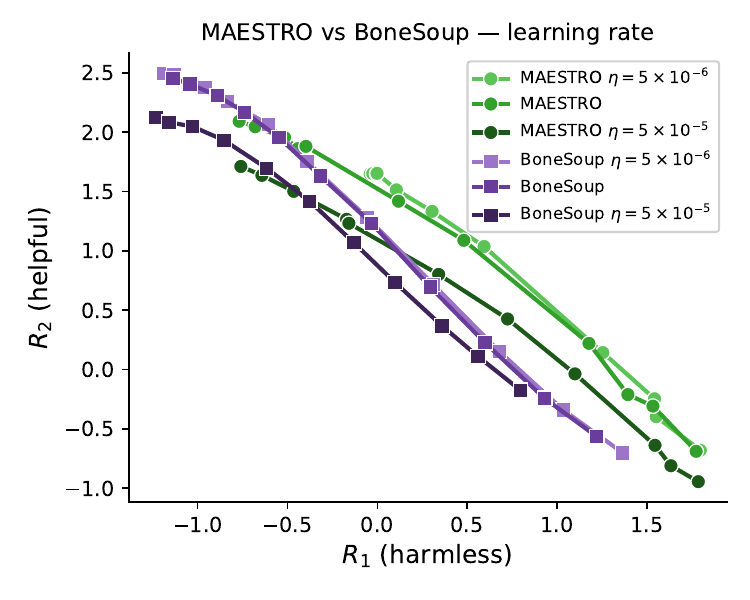}
    \caption{
        Pareto fronts of \ours and Bone Soups trained at learning rates $5\times10^{-6}$, $10^{-5}$ and $5\times10^{-5}$.
        The learning rate moves both methods, advancing or retracting the front and changing its length, but for each learning rate, \ours traces the better front.}
    \label{fig:lr_comparison}
\end{figure}

% \subsection{Reliability of Single-Run Pareto Front Comparisons}

% \begin{figure}[h]
%     \centering
%     \includegraphics[width=0.6\textwidth]{figure/bubble_pareto.pdf}
%     \caption{
%         Per-seed Pareto front of MORLHF on HH-RLHF (harmlessness--helpfulness), evaluated with $\sigma{=}1$.
%         We train 4 additional seeds and plot each seed's $(r^{(1)}, r^{(2)})$ outcome per fixed preference vector $\w$ as a bubble.
%         The low spread across seeds indicates that a single Pareto front is a reliable basis for comparing methods.
%     }
%     \label{fig:bubble_main}
% \end{figure}

\clearpage

% As discussed in Appendix~\ref{apdx:theory}, scaling the KL regularization by $1/M$ is motivated by the fact that the worst-case preference $\w^*$ may place all weight on a single objective.
% % {Scaling $\beta$ by $1/M$ equalizes, to first order, the reward gain of the regularized optimum at the barycenter of the simplex with that of a per-objective baseline when the objectives are uncorrelated (Appendix~\ref{apdx:kl}).}
% {Scaling $\beta$ by $1/M$ gives the part of a single conditioned policy that is shared across preferences the same KL strength per objective as a per-objective baseline; this is exact when the policy does not condition on $\w$ and becomes unnecessary as its conditioning becomes perfect (Appendix~\ref{apdx:kl}).}
% % so the effective KL regularization should be comparable to that of RL optimized baselines.
% As shown in Figure \ref{fig:kl_abl}, this scaling leads to substantially improved Pareto front quality, especially in the Humor--Helpful task where the objectives are more overlapping.
% \MJ{True? running on beavers.}
% Still, the original unscaled KL regularization ($\beta_{KL}=0.2$) already achieves competitive performance, as shown in the main results.

\subsection{Generalization to Unseen Objective Pairs}
\label{apdx:zero_overhead}

Since preferences enter \ours through the prompt rather than through the weights, a trained policy can be pointed at objectives it was never trained on, with no retraining and no additional parameters.
We test this by evaluating a policy trained on one objective pair against the other pair, and compare it against two references, \ie the base SFT model steered by the same prompts, which involves no trained policy, and the policy trained in-domain on that pair.

% Figure~\ref{fig:zero_overhead} shows all four transfer directions.
% In every one of them the transferred policy traces most of a front, whereas prompting the base model alone leaves its responses almost unchanged across the preference grid.
% What training adds is therefore not a mapping tied to the objectives it was trained on, but the ability to read a stated preference off the prompt and act on it, which the model then applies to objectives it has not seen.

% How close transfer comes to in-domain training depends on the direction: transferring from harmless--helpful to humor--helpful gives a front comparable in extent to the in-domain one, displaced but similarly steerable, whereas the reverse direction covers a much narrower range.
% We also do not claim transfer across \emph{tasks}: applying a policy trained on dialogue to TL;DR summarization produces no steering at all.

\begin{figure}[h]
    \centering
    \subfigure[BeaverTails]{
        \includegraphics[width=0.4\textwidth]{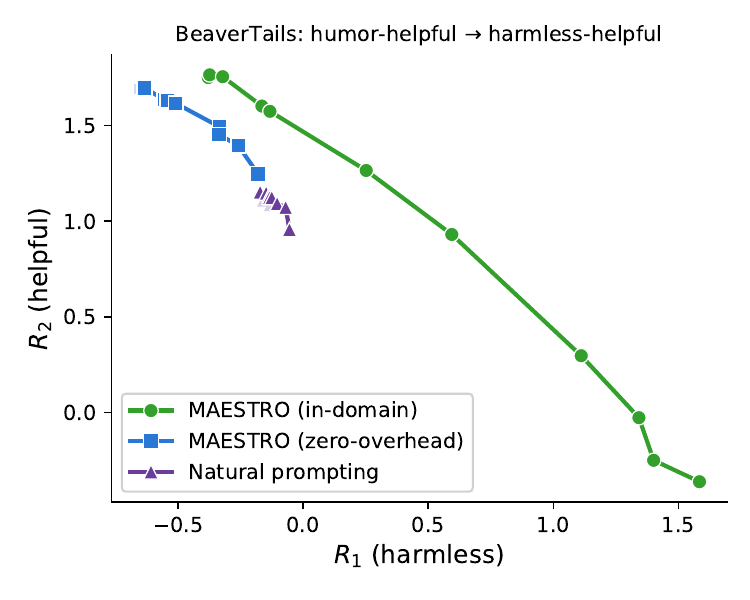}
        \includegraphics[width=0.4\textwidth]{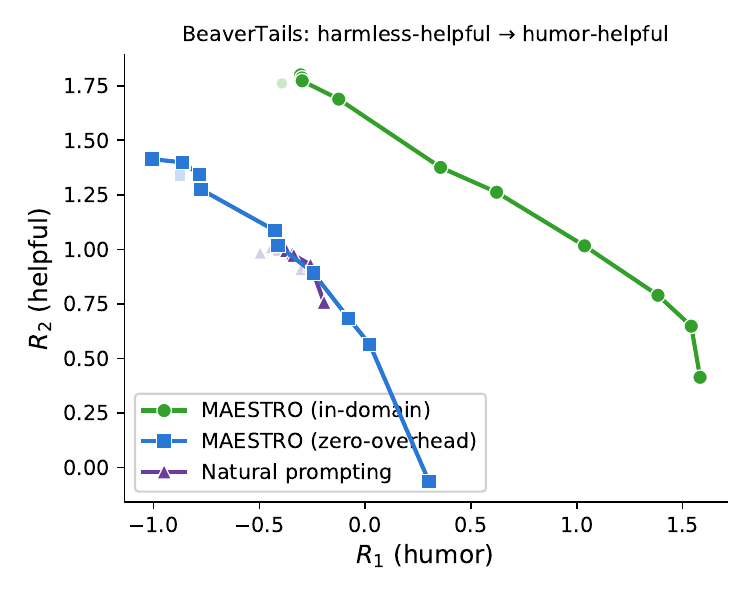}
    }
    \subfigure[HH-RLHF]{
        \includegraphics[width=0.4\textwidth]{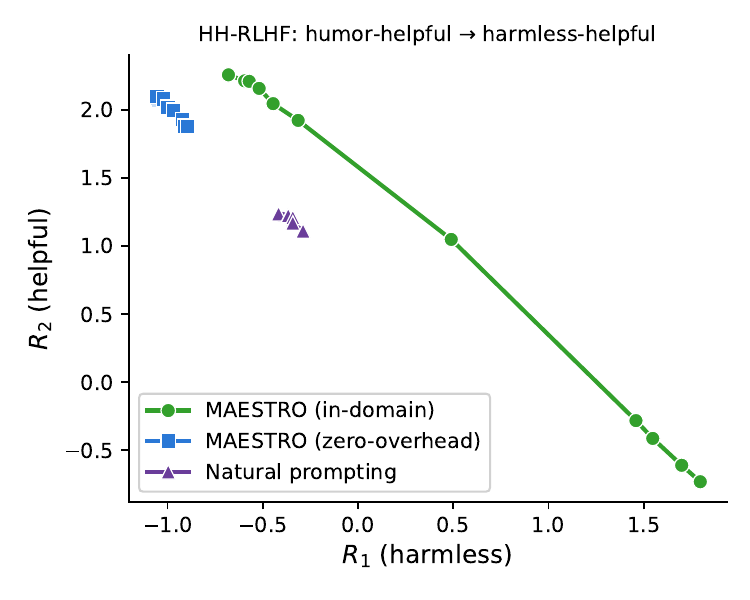}
        \includegraphics[width=0.4\textwidth]{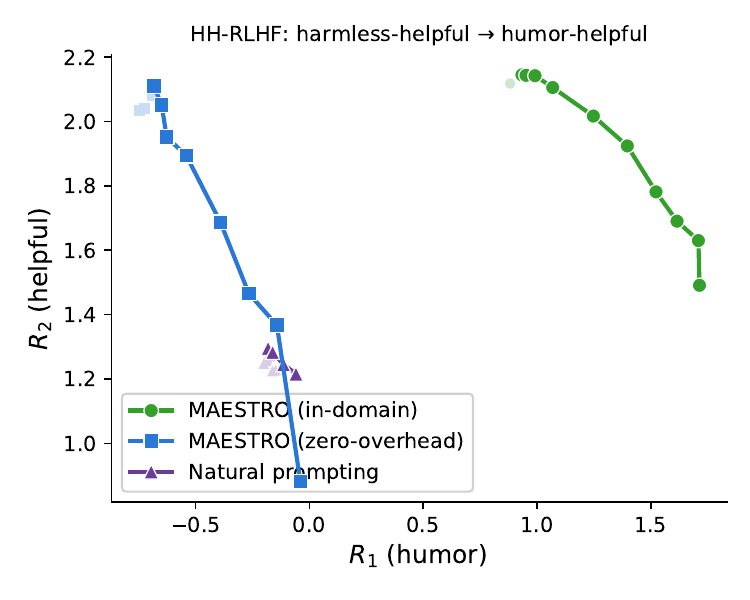}
    }
    \caption{
        Cross-pair transfer by prefix change alone: a policy trained on Harmless--Helpful is evaluated on Humor--Helpful and vice versa (\texttt{zero-overhead} in the plots), on (a) BeaverTails and (b) HH-RLHF. Each panel compares the transferred policy with the base SFT model steered by the same prompts and with the policy trained in-domain on that pair.
    }
    \label{fig:zero_overhead}
    \vspace{-1em}
\end{figure}

\clearpage
\subsection{Where the Objectives Barely Conflict}
\label{apdx:fail}

\begin{table}[h]
    \centering
    \caption{Rank agreement between the two reward models of each objective pair, measured on the test-set generations of the SFT model, which is optimized for neither objective. Agreement is the fraction of response pairs the two reward models order the same way, averaged over the preference grid; $50\%$ is chance and lower values mean stronger disagreement. Spearman is the rank correlation between the two reward scores over the same responses.}
    \label{tab:rm_agreement}
    \resizebox{0.6\textwidth}{!}{
    \begin{tabular}{llcc}
        \toprule
        Task & Objective pair & Agreement & Spearman \\
        \midrule
        % HH-RLHF      & harmless--helpful & $26.0\%$ & $-0.652$ \\
        % BeaverTails  & harmless--helpful & $26.1\%$ & $-0.657$ \\
        % BeaverTails  & humor--helpful    & $35.6\%$ & $-0.431$ \\
        % HH-RLHF      & humor--helpful    & $38.3\%$ & $-0.351$ \\
        HH-RLHF      & Harmless--Helpful & $26.0\%$ & $-0.652$ \\
        BeaverTails  & Harmless--Helpful & $26.1\%$ & $-0.657$ \\
        BeaverTails  & Humor--Helpful    & $35.6\%$ & $-0.431$ \\
        HH-RLHF      & Humor--Helpful    & $38.3\%$ & $-0.351$ \\
        \midrule
        % TL;DR        & summary--faithful & $43.6\%$ & $-0.191$ \\
        % TL;DR        & deberta--faithful~\citep{oasst_deberta_rm} & $43.8\%$ & $-0.185$ \\
        TL;DR        & Summary--Faithful & $43.6\%$ & $-0.191$ \\
        TL;DR        & DeBERTa--Faithful & $43.8\%$ & $-0.185$ \\
        \bottomrule
    \end{tabular}
    }
\end{table}

To measure how separable an objective pair is independently of any method, we score the SFT model's generations, and verify how often the two reward models order the same pair of responses in the same way (Table~\ref{tab:rm_agreement}).
The two domains separate cleanly: on dialogue the reward models systematically disagree, so improving one objective usually costs the other, whereas on TL;DR the agreement sits close to the $50\%$ chance level, and closer for the DeBERTa--Faithfulness pair.
% { To measure how separable an objective pair is independently of any method, we score the SFT model's generations and measure how often the two reward models order the same response pair in the same way (Table~\ref{tab:rm_agreement}).}
% { The two domains behave differently: on dialogue, the reward models systematically disagree, so improving one objective usually costs the other; on TL;DR, agreement is closer to the $50\%$ chance level, with the DeBERTa--Faithfulness pair closest to chance.}
% Replacing the first reward model entirely (GPT-2 summarization vs.\ DeBERTa) barely moves the number, so this is a property of the task rather than of one particular scorer.

This bounds what \ours can contribute, since the adversary can only shift training toward preferences on which the policy is measurably worse: when the objectives barely conflict, the worst case is nearly independent of $\w$ and the front stays short whatever the sampling scheme.
Figure~\ref{fig:pref2} shows the extreme case, the \texttt{deberta}--faithful pair, on which \ours collapses onto a single point and even MORLHF, which trains one policy per preference, spreads out only slightly.

\begin{figure}[h]
    \centering
    \includegraphics[width=0.48\textwidth]{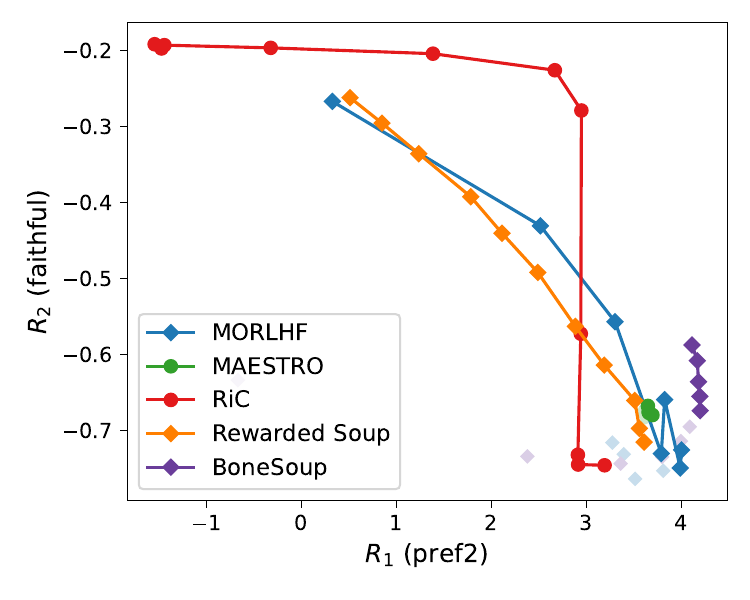}
    % \caption{Pareto front on the summarization task with the summary reward model replaced by a general-purpose one, \texttt{OpenAssistant/reward-model-deberta-v3-large-v2}~\citep{oasst_deberta_rm} (pref2), against faithfulness. \ours collapses onto a single point, its responses being nearly identical across the preference grid, and MORLHF, which trains one policy per preference, spreads out only slightly. Table~\ref{tab:rm_agreement} shows that the two reward models of this pair order responses almost independently, leaving little for either method to trace.}
    \caption{Pareto front on the summarization task with the summary reward model replaced by \texttt{OpenAssistant/reward-model-deberta-v3-large-v2}~\citep{oasst_deberta_rm} (pref2), paired with the Faithfulness reward. \ours collapses onto a single point, with nearly identical responses across the preference grid, and MORLHF, which trains one policy per preference, spreads out only slightly. Table~\ref{tab:rm_agreement} shows that these two reward models order responses almost independently, leaving little for either method to trace.}
    \label{fig:pref2}
\end{figure}

% \clearpage

\clearpage
\section{Theoretical Support}
\label{apdx:theory}
% \MJ{remove}

\subsection{Pareto Optimality}
\label{sec:dis:pareto}

\para{Pareto Optimality Definition.}
%\MJ{$\pi$}
% % Let $\Pi$ denote the space of all representable policies, and let $R_m(\pi)$ be the expected reward of policy $\pi$ for the $m$-th objective, where $m \in [M]$.
% A solution $\theta \in \Theta$ is said to \emph{Pareto-dominate}
% % \SP{here, our pareto optimality needs to be defined over a fixed set of policies -- if we do not fix the set, the problem can be devising a better deep net architecture, which is fine but not our focus.} 
% another solution $\theta' \in \Theta$ if it performs at least as well on all objectives and strictly better on at least one:
% \begin{equation}
%     \forall m \in [M], J^{(m)}(\pi_{\theta}) \ge J^{(m)}(\pi_{\theta'}) 
%     \quad\text{and}\quad
%     \exists k \in [M], J^{(k)}(\pi_{\theta}) > J^{(k)}(\pi_{\theta'}).
%     \label{eq:pareto_domination}
% \end{equation}
% %with $J^{(k)}(\pi) > J^{(k)}(\pi')$ for at least one $k \in [M]$.
% A solution $\theta \in \Theta$ is considered \emph{Pareto-optimal} in $\Theta$ if there does not exist any solution $\theta' \in \Theta$ that Pareto-dominates it.
% The set of all Pareto-optimal solutions $\Theta^\text{PO} \subseteq \Theta$ constitutes the \emph{Pareto front}.
% Let $\Pi$ denote the space of all representable policies, and let $R_m(\pi)$ be the expected reward of policy $\pi$ for the $m$-th objective, where $m \in [M]$.
A policy $\pi \in \Pi$ is said to \emph{Pareto-dominate}
another policy $\pi' \in \Pi$ if it performs at least as well on all objectives and strictly better on at least one:
\begin{equation}
    \forall m \in [M], J^{(m)}(\pi) \ge J^{(m)}(\pi') 
    \quad\text{and}\quad
    \exists k \in [M], J^{(k)}(\pi) > J^{(k)}(\pi').
    \label{eq:pareto_domination}
\end{equation}
%with $J^{(k)}(\pi) > J^{(k)}(\pi')$ for at least one $k \in [M]$.
A policy $\pi \in \Pi$ is considered \emph{Pareto-optimal} in $\Pi$ if there does not exist any policy $\pi' \in \Pi$ that Pareto-dominates it.
% The set of all Pareto-optimal policies $\Pi^\text{PO} \subseteq \Pi$ constitutes the \emph{Pareto front}.
The set of all Pareto-optimal policies $\Pi^\text{PO} \subseteq \Pi$ is the \emph{Pareto set}, and its image $\{\J(\pi):\pi\in\Pi^\text{PO}\}$ in objective space is the \emph{Pareto front}.

Pareto-optimality of the realizing policy is formally established in Proposition~\ref{prop:pareto_recovery} below under the simultaneous realizability condition.
We empirically observe that the Pareto front found within the single shared policy space $\Pi \subset \Pi_\text{large}$ is comparable to that of MORLHF operating over the larger per-preference space $\Pi_\text{large}$ (Section~\ref{sec:exp}), suggesting that prompt-based conditioning does not substantially reduce Pareto front expressiveness.

\subsection{Minimax Optimization}
\label{sec:dis:minimax}
% \SP{subsection name for this content?}

{Throughout, $\Ws_{>0} \coloneqq \{\w \in \Ws : w^{(m)} > 0 \text{ for all } m\}$ denotes the preferences that weight every objective strictly positively, \ie the relative interior of the simplex.}

\begin{definition}[Supported Pareto-optimality]
\label{def:supported}
% Let $\Pi_\Theta^{\w} \coloneqq \{\pi_\theta(\cdot \mid \cdot, \w) : \theta \in \Theta\}$ for the policies the parameterization induces at a fixed preference $\w$, so that $V^*(\w) = \sup_{\pi \in \Pi_\Theta^{\w}} \w^\top\J(\pi)$.
% A policy $\pi \in \Pi_\Theta^{\w}$ is \emph{supported} Pareto-optimal if it maximizes $\w^\top\J(\pi')$ over $\pi'\in\Pi_\Theta^{\w}$ for some $\w\in\mathrm{int}(\Ws)$.
Let $\Pi_\Theta^{\w} \coloneqq \{\pi_\theta(\cdot \mid \cdot, \w) : \theta \in \Theta\}$ for the policies the parameterization induces at a fixed preference $\w$, so that $V^*(\w) = \sup_{\pi \in \Pi_\Theta^{\w}} \w^\top\J(\pi)$.
{For $\w\in\Ws_{>0}$,} a policy $\pi \in \Pi_\Theta^{\w}$ is \emph{supported} Pareto-optimal if it maximizes $\w^\top\J(\pi')$ over $\pi'\in\Pi_\Theta^{\w}$.
% Every supported Pareto-optimal policy is Pareto-optimal, but the converse fails: a Pareto-optimal point lying in a non-convex region of the front maximizes no strictly positive linear scalarization, and is therefore unreachable by linear scalarization for any choice of $\w$.
\end{definition}

\begin{proposition}[{Approximate }minimax optimality under {$\epsilon$-}realizability]
\label{prop:pareto_recovery}
Define the per-preference optimal value $V^*(\w) \coloneqq \sup_\theta J_\w(\pi_\theta(\cdot \mid \cdot, \w))$.
Suppose {for some $\epsilon \ge 0$ }there exists $\theta^*$ such that
\begin{equation}
    J_\w(\pi_{\theta^*}(\cdot \mid \cdot, \w)) {\ge} V^*(\w) {- \epsilon}, \qquad \forall \w \in \Ws.
    \label{eq:realizable}
\end{equation}
% Then $\theta^*$ is an $\epsilon$-optimal solution of the minimax objective~\eqref{eq:objective} for any ambiguity set $\Ps \subseteq \Delta(\Ws)$, with minima assumed to be attained.
Then $\theta^*$ is an $\epsilon$-optimal solution of~\eqref{eq:objective} for any nonempty ambiguity set $\Ps\subseteq\Delta(\Ws)$:
% Original equation used \min_{p\in\Ps} on both sides for an arbitrary, possibly noncompact $\Ps$.
\begin{equation}
    \inf_{p {\in \Ps}}\,\Exp_{\w \sim p}\bigl[J_\w(\pi_{\theta^*})\bigr]
    \;\ge\;
    \sup_\theta\,\inf_{p {\in \Ps}}\,\Exp_{\w \sim p}\bigl[J_\w(\pi_\theta)\bigr] - \epsilon .
    \label{eq:eps_minimax}
\end{equation}
% Moreover, for every $\w \in \Ws$ with $w_m > 0$ for all $m$, the induced policy $\pi_{\theta^*}(\cdot \mid \cdot, \w)$ is $\epsilon$-approximately supported Pareto-optimal in the sense of Definition~\ref{def:supported}.
{ Moreover, for every $\w\in\Ws$ the induced policy $\pi_{\theta^*}(\cdot\mid\cdot,\w)$ attains $\w^\top\J$ within $\epsilon$ of its supremum over $\Pi_\Theta^{\w}$, which for $\w\in\Ws_{>0}$ is $\epsilon$-approximate supported Pareto-optimality.
In particular, when $\epsilon = 0$ the induced policy is weakly Pareto-optimal in $\Pi_\Theta^{\w}$ for every $\w\in\Ws$, and supported Pareto-optimal for every $\w\in\Ws_{>0}$.}
\end{proposition}

\begin{proof}
For any $\theta$ and any $\w$, $J_\w(\pi_\theta(\cdot|\cdot,\w)) \le V^*(\w) {\;\le\; J_\w(\pi_{\theta^*}(\cdot|\cdot,\w)) + \epsilon}$ by definition of $V^*$ {and~\eqref{eq:realizable}}.
% Taking expectation under any $p \in \Ps$ and then minimizing over $p$ gave the following display with \min on an arbitrary $\Ps$.
Taking expectations under any $p\in\Ps$ and then taking infima over $p$ gives
\begin{equation}
    \inf_{p \in \Ps}\,\Exp_{\w \sim p}\bigl[J_\w(\pi_\theta)\bigr]
    \;\le\;
    \inf_{p \in \Ps}\,\Exp_{\w \sim p}\bigl[V^*(\w)\bigr]
    \;\le\;
    \inf_{p \in \Ps}\,\Exp_{\w \sim p}\bigl[J_\w(\pi_{\theta^*})\bigr] + \epsilon.
\end{equation}
Since this holds for every $\theta$, taking the supremum over $\theta$ on the left gives~{\eqref{eq:eps_minimax}}.

For {the second claim}, fix any $\w \in {\Ws}$.
{By~\eqref{eq:realizable},}
\begin{equation}
    J_\w(\pi_{\theta^*}(\cdot|\cdot,\w)) \;{\ge}\; V^*(\w) {\;-\;\epsilon}
    \;=\; \sup_\theta J_\w(\pi_\theta(\cdot|\cdot,\w)) {\;-\;\epsilon},
\end{equation}
so $\pi_{\theta^*}(\cdot \mid \cdot, \w)$ is an $\epsilon$-maximizer of the weighted sum $\w^\top\J$ over $\Pi_\Theta^{\w}$, {which for $\w\in\Ws_{>0}$ is the stated $\epsilon$-approximate supported Pareto-optimality.
At $\epsilon = 0$ the maximizer is exact, hence weakly Pareto-optimal in $\Pi_\Theta^{\w}$ since $\w\succeq\mathbf{0}$ and $\w\neq\mathbf{0}$, and Pareto-optimal when $\w\in\Ws_{>0}$.}
\end{proof}

\begin{remark*}
Proposition~\ref{prop:pareto_recovery} shows that \emph{if} a single preference-conditioned model can {come within $\epsilon$ of} the per-preference optimum for all $\w$, then it is {within $\epsilon$ of} the minimax-optimal solution.
The converse does not hold in general: a minimax-optimal policy need not be pointwise Pareto-optimal at every $\w$.
% The assumption is measurable: $\epsilon$ is the largest gap between \ours and a per-preference oracle over the preference grid, and the fronts of Section~\ref{sec:exp} match MORLHF ($\epsilon \approx 0$).
The assumption is measurable: the largest gap between \ours and a per-preference oracle over the preference grid {lower-bounds the uniform $\epsilon$ of~\eqref{eq:realizable}}, and the fronts of Section~\ref{sec:exp} match MORLHF.
The same condition with $\epsilon = 0$ is the first assumption of Theorem~4.1 in Panacea~\citep{zhong2024panacea}.

In addition, the conclusion covers only the \emph{supported} subset of the front (Definition~\ref{def:supported}), and this is a property of linear scalarization rather than of $\epsilon$.
% The two coincide when the policy class is closed under mixing, since the achievable return set is then convex (Lemma~C.3 of~\citep{zhong2024panacea}), which is an idealization for any single parameterization.
The two coincide when the policy class is closed under mixing {and the supporting weights can be taken strictly positive}, since the achievable return set is then convex (Lemma~C.3 of~\citep{zhong2024panacea}), which is an idealization for any single parameterization.
\end{remark*}

\begin{lemma}[Expectation can ignore a hard preference region]
\label{lem:hard_region}
Let $\mathcal W = A \cup B$ be a partition of the preference space with
$p(B)=\varepsilon \in (0,1)$.
Assume there exist two policies $\pi^{\mathrm{avg}}$ and $\pi^{\mathrm{rob}}$
such that for some constants $u,v,\Delta > 0$ where $v \le u$,
\begin{equation}
    J_\w(\pi^{\mathrm{avg}}) =
    \begin{cases}
        u+\Delta, & \w \in A,\\
        v-\Delta, & \w \in B,
    \end{cases}
\end{equation}
and
\begin{equation}
    J_\w(\pi^{\mathrm{rob}}) =
    \begin{cases}
        u, & \w \in A,\\
        v, & \w \in B.
    \end{cases}
\end{equation}
Then:
\begin{equation}
    \mathbb E_{\w\sim p}[J_\w(\pi^{\mathrm{avg}})]
    -
    \mathbb E_{\w\sim p}[J_\w(\pi^{\mathrm{rob}})]
    =
    (1-2\varepsilon)\Delta,
\end{equation}
while
\begin{equation}
    \min_{\w} J_\w(\pi^{\mathrm{avg}})
    -
    \min_{\w} J_\w(\pi^{\mathrm{rob}})
    =
    -\Delta.
\end{equation}
In particular, whenever $\varepsilon < 1/2$, the expectation objective prefers
$\pi^{\mathrm{avg}}$, whereas the {pointwise} minimax objective prefers
$\pi^{\mathrm{rob}}$.
\end{lemma}

\begin{proof}
By construction,
\begin{equation}
    \mathbb E_{\w\sim p}[J_\w(\pi^{\mathrm{avg}})]
    =
    (1-\varepsilon)(u+\Delta) + \varepsilon (v-\Delta),
\end{equation}
and
\begin{equation}
    \mathbb E_{\w\sim p}[J_\w(\pi^{\mathrm{rob}})]
    =
    (1-\varepsilon)u + \varepsilon v.
\end{equation}
Subtracting gives
\begin{equation}
    \mathbb E_{\w\sim p}[J_\w(\pi^{\mathrm{avg}})]
    -
    \mathbb E_{\w\sim p}[J_\w(\pi^{\mathrm{rob}})]
    =
    (1-2\varepsilon)\Delta.
\end{equation}
Hence if $\varepsilon<1/2$, the expectation objective prefers
$\pi^{\mathrm{avg}}$.

On the other hand,
\begin{equation}
    \min_{\w} J_\w(\pi^{\mathrm{avg}})=v-\Delta,
    \qquad
    \min_{\w} J_\w(\pi^{\mathrm{rob}})=v,
\end{equation}
so
\begin{equation}
    \min_{\w} J_\w(\pi^{\mathrm{avg}})
    -
    \min_{\w} J_\w(\pi^{\mathrm{rob}})
    =
    -\Delta.
\end{equation}
Thus the {pointwise} minimax objective strictly prefers $\pi^{\mathrm{rob}}$.
\end{proof}
Lemma~\ref{lem:hard_region} formalizes the failure mode of average-preference
optimization: if a hard preference region has sufficiently small measure, then
an expectation objective can prefer a policy that improves the majority region
while sacrificing that hard region. In contrast, % the minimax objective cannot
{the pointwise minimax objective cannot}
ignore such regions, since its value is determined by the worst-performing
preference.
{Our objective~\eqref{eq:objective} inherits this through its Dirichlet family, whose unconstrained infimum is exactly this pointwise minimum and whose clamped version approaches it as $\alpha_{\max}$ grows (Appendix~\ref{apdx:layer2}).}
This is empirically supported by Figure \ref{fig:motivation_sampling}.

\subsection{KL Regularization in MAESTRO}
\label{apdx:kl}
% \MJ{reference}

\ours, like every PPO-based baseline, is trained on the KL-regularized scalarized reward $\w^\top\r(\x,\y) - \beta\log\frac{\pi_\theta(\y\mid\x,\w)}{\pi_{\mathrm{ref}}(\y\mid\x,\w)}$.
% The coefficient starts at $\beta = 0.2$ on all tasks and is adapted during training by the controller of \citep{ziegler2019fine}; at every step, however, it takes the same value for every preference in the batch.
Although this coefficient $\beta$ matches that of the baselines that train one policy per preference (MORLHF) or one expert per objective (the soups), the effective KL strength of a single conditioned model can differ from theirs, since one model has to serve every preference and the part of its behavior that is shared across them sees each objective only with the weight that the preference distribution gives it, whereas each baseline devotes a separate model, and at equal epochs its entire budget of updates, to one preference or objective.

\para{KL regularization as a source of $\epsilon$.}
For fixed $\x$ and $\w$, the maximizer of $\Exp_{\y\sim\pi}[\w^\top\r(\x,\y)] - \beta\,\mathrm{KL}(\pi\,\|\,\pi_{\mathrm{ref}})$ is the tilted distribution
\begin{equation}
    \pi^*_{\w,\beta}(\y\mid\x) = \frac{\pi_{\mathrm{ref}}(\y\mid\x)\,e^{\w^\top\r(\x,\y)/\beta}}{\Exp_{\y'\sim\pi_{\mathrm{ref}}}\bigl[e^{\w^\top\r(\x,\y')/\beta}\bigr]},
    \label{eq:tilted}
\end{equation}
where $\pi_{\mathrm{ref}}(\cdot\mid\x)$ denotes the reference on the input the policy sees, which is $\pi_{\mathrm{ref}}(\cdot\mid\x,\w)$ for \ours.
% \MJ{The Gibbs-form optimizer is over unrestricted response distributions, whereas $V^*(\w)$ is over the parameterized policy class. The bound below remains an algebraic comparison if $\epsilon_{\mathrm{approx}}$ includes this mismatch, but it cannot then be interpreted purely as the cost of sharing one model. If that narrower interpretation is needed, define both optima over the same class or add separate class and optimization terms.}
% Training targets this tilted distribution, whereas $V^*(\w)$ from Appendix \ref{sec:dis:minimax} is defined on the reward alone, so the KL penalty keeps even a perfect solution below $V^*(\w)$ by
{The tilted distribution is the unrestricted regularized optimum, while $V^*(\w)$ from Appendix~\ref{sec:dis:minimax} is the unregularized optimum over the parameterized class. We record their raw-reward discrepancy as}
\begin{equation*}
    \epsilon_\beta(\w) \coloneqq \bigl(V^*(\w) - J_\w(\pi^*_{\w,\beta})\bigr)_+ .
\end{equation*}
% Hence, if the conditioned policy comes within $\epsilon_{\mathrm{approx}}$ of the regularized optimum at every $\w\in\Ws$, in the sense that $J_\w(\pi_\theta) \ge J_\w(\pi^*_{\w,\beta}) - \epsilon_{\mathrm{approx}}$, then~\eqref{eq:realizable} holds with
Hence, if the conditioned policy comes within $\epsilon_{\mathrm{approx}}$ of {the raw reward of this unrestricted tilt} at every $\w\in\Ws$, in the sense that $J_\w(\pi_\theta) \ge J_\w(\pi^*_{\w,\beta}) - \epsilon_{\mathrm{approx}}$, then~\eqref{eq:realizable} holds with
\begin{equation}
    \epsilon \;=\; \epsilon_{\mathrm{approx}} \;+\; \sup_{\w\in\Ws}\epsilon_\beta(\w).
    \label{eq:eps_split}
\end{equation}
% The first term is the cost of sharing one model across preferences, which is what realizability is about.
The first term is the cost of sharing one model across preferences, {together with the policy-class and optimization error that the comparison to an unrestricted tilt absorbs}.
The second decreases as $\beta$ decreases and is shared by methods that agree on $\beta$, the reference distribution and $\r$.
% Since \ours regularizes toward $\pi_{\mathrm{ref}}(\cdot\mid\x,\w)$ and the per-objective baselines toward $\pi_{\mathrm{ref}}(\cdot\mid\x)$, KL comparisons across methods are made against the common reference $\pi_{\mathrm{ref}}(\cdot\mid\x)$.
% Each method is regularized toward the reference model on its own input, as in standard RLHF; for \ours that input includes the preference prefix, so its reference is $\pi_{\mathrm{ref}}(\cdot\mid\x,\w)$.

\begin{proposition}[KL scaling for the preference-shared component]
\label{prop:kl_scaling}
Consider (i) a per-objective baseline trained at the vertex $\mathbf{e}_m$ with coefficient $\beta_b$, whose optimum is $\pi^*_{\mathbf{e}_m,\beta_b} \propto \pi_{\mathrm{ref}}\,e^{r^{(m)}/\beta_b}$, and (ii) a single policy trained with $\w\sim p$ and coefficient $\beta_{\mathrm{ours}}$, and write $\boldsymbol{\mu} \coloneqq \Exp_{p}[\w]$.
If the policy does not depend on $\w$, its optimum is $\pi^*_{\boldsymbol{\mu},\beta_{\mathrm{ours}}} \propto \pi_{\mathrm{ref}}\,e^{\boldsymbol{\mu}^\top\r/\beta_{\mathrm{ours}}}$, and the coefficient multiplying $r^{(m)}$ in its log-ratio to the reference equals the baseline's, for $\mu_m > 0$, if and only if
\begin{equation}
    \beta_{\mathrm{ours}} = \beta_b\,\mu_m.
\end{equation}
In particular, if $p$ is symmetric over $\Delta^{M-1}$, then $\mu_m = 1/M$ and $\beta_{\mathrm{ours}} = \beta_b/M$ for every $m$.
\end{proposition}

\begin{proof}
Since the policy does not depend on $\w$, linearity gives $\Exp_{\w\sim p}\Exp_{\y\sim\pi}[\w^\top\r] = \Exp_{\y\sim\pi}[\boldsymbol{\mu}^\top\r]$, and the KL term does not depend on $\w$.
The objective is therefore the regularized objective with reward $\boldsymbol{\mu}^\top\r$, whose maximizer is~\eqref{eq:tilted} with $\w$ replaced by $\boldsymbol{\mu}$; its log-ratio to $\pi_{\mathrm{ref}}$ has coefficient $\mu_m/\beta_{\mathrm{ours}}$ on $r^{(m)}$, against $1/\beta_b$ for the baseline.
With the prefixed reference of \ours the KL term does depend on $\w$, but $\Exp_{\w\sim p}[\mathrm{KL}(\pi\,\|\,\pi_{\mathrm{ref}}(\cdot\mid\x,\w))] = \mathrm{KL}(\pi\,\|\,q) + \text{const}$ with $q$ the normalized geometric mean of $\pi_{\mathrm{ref}}(\cdot\mid\x,\w)$ over $p$, so the optimum is the same expression with $\pi_{\mathrm{ref}}$ replaced by $q$ and the coefficient on $r^{(m)}$ unchanged.
\end{proof}

\para{Scope.}
% The proposition is exact only for a policy that ignores the preference.
% At the level of gradients it describes the part of any conditioned policy's update that is shared across preferences: since $\nabla_\theta J_\w = \sum_m w_m \nabla_\theta J^{(m)}$, objective $m$ contributes $\mu_m\nabla_\theta J^{(m)}$ there, while the KL gradients of all preferences pull that part toward the reference without attenuation.
% The part that varies with $\w$ instead receives reward and KL in the same proportion, and under perfect conditioning the objective decouples across preferences, so $\beta_b$ is already the matching coefficient.
% The scaling is therefore a correction for incomplete conditioning, and a compromise rather than a bound: the matching coefficient is $\beta_b\,\mu_m/w_m$ against a policy trained at $\w$, so it differs across objectives, and the adversary moves $\boldsymbol{\mu}$ during training.
The proposition is exact only for a policy that ignores the preference.
{For a conditioned policy, writing $g_m(\w)\coloneqq\nabla_\theta J^{(m)}(\theta;\w)$ gives $\Exp_p[w_m g_m(\w)]=\mu_m\Exp_p[g_m(\w)]+\operatorname{Cov}_p(w_m,g_m(\w))$, so the factor $\mu_m$ requires the covariance term to vanish.}
The scaling is therefore a correction for incomplete conditioning, and a compromise rather than a bound: the matching coefficient is $\beta_b\,\mu_m/w_m$ against a policy trained at $\w$, {which needs $w_m>0$ and} differs across objectives, and the adversary moves $\boldsymbol{\mu}$ during training.
Empirically the gains in Figure~\ref{fig:kl_abl} are consistent with the conditioning being incomplete within our training budget, though those runs also raise the target KL, so part of the gain may come from the larger budget.

%% ============================================================================
\subsection{The Ambiguity Set and How Much It Smooths}
\label{apdx:layer2}

Two choices separate Algorithm~\ref{alg:maestro} from the objective~\eqref{eq:objective}: the ambiguity set the inner player ranges over, and the statistic its descent direction is computed from.
We state here what each of them costs, and what we do not claim.

Taking $\Ps = \Delta(\Ws)$ would be degenerate: the inner objective is linear in $p$, so its minimum is attained at a point mass, and the resulting envelope credits improvement only at the minimizing preference and {is nonsmooth where that preference switches}.
% Restricting $\Ps$ to Dirichlet distributions of bounded concentration removes both problems, at the price of a smoothed rather than pointwise worst case, since such a distribution cannot place all its mass at one preference.
{ Restricting $\Ps$ to Dirichlet distributions of bounded concentration prevents point-mass inner distributions and smooths their payoffs, at the price of replacing the pointwise worst case by a distributional one. The resulting outer envelope can nevertheless remain nonsmooth when the minimizing Dirichlet parameter changes.}
The next proposition makes the amount of smoothing explicit and shows that $\alpha_{\max}$ interpolates between the two.

\begin{proposition}[The Dirichlet family recovers the pointwise worst case as concentration grows]
\label{prop:dirichlet_tight}
Fix $\theta$ and suppose $\w \mapsto J_\w(\pi_\theta)$ is continuous on $\Ws$. Then
\begin{equation}
    \inf_{\alphaB \in \realnum_{>0}^M} \; \Exp_{\w\sim\mathrm{Dir}(\alphaB)}\bigl[J_\w(\pi_\theta)\bigr]
    \;=\; \min_{\w \in \Ws} J_\w(\pi_\theta).
\end{equation}
% If in addition $\w \mapsto J_\w(\pi_\theta)$ is $L$-Lipschitz with respect to $\|\cdot\|_1$, then restricting to $\As = [\alpha_{\min},\alpha_{\max}]^M$ costs at most
{ If, in addition, $\w \mapsto J_\w(\pi_\theta)$ is $L$-Lipschitz with respect to $\|\cdot\|_1$ and $\alpha_{\max}\ge M\alpha_{\min}$, then restricting to $\As = [\alpha_{\min},\alpha_{\max}]^M$ costs at most}
\begin{equation}
    \min_{\alphaB\in\As}\,\Exp_{\w\sim\mathrm{Dir}(\alphaB)}\bigl[J_\w(\pi_\theta)\bigr] - \min_{\w\in\Ws} J_\w(\pi_\theta)
    \;\le\; L\left( \sqrt{\frac{M}{\alpha_{\max}+1}} \;+\; \frac{2M\alpha_{\min}}{\alpha_{\max}} \right).
    \label{eq:dirichlet_gap}
\end{equation}
\end{proposition}

\begin{proof}
Write $\alphaB = \kappa\muB$ with $\muB = \alphaB/\|\alphaB\|_1 \in \Delta^{M-1}$ and $\kappa = \|\alphaB\|_1$.
Under $\mathrm{Dir}(\kappa\muB)$ we have $\Exp[w^{(i)}] = \mu_i$ and $\mathrm{Var}[w^{(i)}] = \mu_i(1-\mu_i)/(\kappa+1)$, so by Jensen and Cauchy--Schwarz
\begin{equation}
    \Exp\bigl[\|\w - \muB\|_1\bigr] \;\le\; \sum_{i=1}^M \sqrt{\tfrac{\mu_i(1-\mu_i)}{\kappa+1}} \;\le\; \sqrt{\frac{M}{\kappa+1}} .
    \label{eq:conc}
\end{equation}
Hence $\mathrm{Dir}(\kappa\muB) \Rightarrow \delta_{\muB}$ as $\kappa \to \infty$, and by bounded continuity $\Exp_{\w\sim\mathrm{Dir}(\kappa\muB)}[J_\w] \to J_{\muB}$.
{ If $\muB$ lies on the boundary, choose an interior sequence $\muB_n\to\muB$ and let the corresponding concentrations grow; continuity gives the same limiting value.}
Taking the infimum over $\muB$ bounds the left-hand side by $\inf_{\muB} J_{\muB} = \min_{\w} J_\w$; the reverse inequality holds since any average of $\{J_\w\}$ is at least $\min_\w J_\w$.
For the second claim, let $\w^\star \in \argmin_{\w} J_\w$ and $\delta \coloneqq \alpha_{\min}/\alpha_{\max}$, which satisfies $M\delta \le 1$ in the regime we use.
Project $\w^\star$ onto $\Delta_\delta \coloneqq \{\muB : \mu_i \ge \delta\}$ via $\muB \coloneqq (1-M\delta)\w^\star + \delta\mathbf{1}$, so $\|\muB - \w^\star\|_1 \le 2M\delta$ and $J_{\muB} \le J_{\w^\star} + 2LM\delta$.
The choice $\alphaB = \alpha_{\max}\muB$ lies in $\As$ since $\alpha_{\max}\mu_i \ge \alpha_{\min}$, and has $\kappa = \alpha_{\max}$; combining with~\eqref{eq:conc} and $L$-Lipschitzness gives~\eqref{eq:dirichlet_gap}.
\end{proof}
% \MJ{This proposition assumes continuity in the numerical weight vector. Independently rounded percentage prompts create a discontinuous, many-to-one map from $\w$ to text. State whether the theoretical $J_\w$ uses exact weights, or add a quantization error/continuity qualification connecting the implementation to the proposition.}

The two clamps therefore play different roles rather than trading off against each other.
The bound is dominated by its first term, so $\alpha_{\max}$ sets how robust the objective is, at rate $O(\alpha_{\max}^{-1/2})$; with $(\alpha_{\min},\alpha_{\max}) = (10^{-3},10^{2})$ and $M=2$ it reads $L\sqrt{2/101} \approx 0.14L$.
The clamp $\alpha_{\min}>0$ instead makes the adversary's estimator well-posed: the Dirichlet score $s_i = \log w^{(i)} - \psi(\alpha_i) + \psi(S)$ has $\Exp[s_i] = 0$ and $\Exp[s_i^2] = \psi'(\alpha_i) - \psi'(S) \le \psi'(\alpha_{\min})$, and $\psi'(\alpha) \sim \alpha^{-2}$ as $\alpha \downarrow 0$, so without it the estimator has unbounded second moment.

It is sometimes supposed that a worst-case objective will concentrate on a vertex of the simplex, \ie on a single objective.
% A vertex, however, evaluates to the best attainable value of a single objective, which is typically large, so there is no reason for the minimum to sit there.
A vertex, however, evaluates {the trained policy's return on a single objective, which need not be small}, so there is no reason for the minimum to sit there.
This is what we observe: the scalarized return at the two vertices is far above its value near the center of the simplex (Figure~\ref{fig:pareto_beaver}), so the adversary concentrates on genuinely conflicting interior trade-offs rather than collapsing onto one objective.

\subsection{What the Adversary Is Scored On}
\label{apdx:layer3}

The adversary's descent direction is computed from the surrogate \eqref{eq:soft_rank} rather than from $\w^\top\r$ directly, while the policy is still trained on $\w^\top\r$.
{The two ingredients play separate roles: UEG is a positivity-preserving update rule for the Dirichlet parameters, whereas~\eqref{eq:soft_rank} is a batch-relative, per-objective feedback signal chosen to reduce sensitivity to reward scales, and not the exact payoff of~\eqref{eq:minimax}.
% Because the comparison samples are treated as fixed when credit is assigned to each sampled preference, the update is a stochastic gradient heuristic rather than an unbiased gradient of~\eqref{eq:minimax}: the minimax formulation is what motivates it, and the sampler it adapts is evaluated empirically in Section~\ref{sec:exp}.}
{ Because the comparison samples are treated as fixed when credit is assigned to each sampled preference, the update is a stochastic semi-gradient heuristic rather than an unbiased gradient of~\eqref{eq:minimax}: the minimax formulation motivates it, and the adaptive sampler is evaluated empirically in Section~\ref{sec:exp}.}}
{The baseline plays no role in this: for independent returns the same-batch mean of Algorithm~\ref{alg:maestro} only rescales the estimate by $1-1/B_{\mathrm{adv}}$, which the step size absorbs, and a leave-one-out baseline would remove that rescaling but not the coupling above.}
The statistic interpolates between two regimes.
When $\tau\,|r^{(m)}_j - r^{(m)}_{j'}| \ll 1$ within the batch, $\sigma(\tau d) \approx \tfrac12 + \tfrac{\tau}{4}d$ and $\sum_{j'\neq j}(r^{(m)}_j - r^{(m)}_{j'}) = B_{\mathrm{adv}}(r^{(m)}_j - \bar r^{(m)})$ give $\tilde r^{(m)}_j \approx \tfrac12 + \tfrac{\tau}{4}\tfrac{B_{\mathrm{adv}}}{B_{\mathrm{adv}}-1}(r^{(m)}_j - \bar r^{(m)})$, with $\bar r^{(m)}$ the batch mean, and hence
\begin{equation}
    {R_j} \;\approx\; \tfrac12 + \tfrac{\tau}{4}\tfrac{B_{\mathrm{adv}}}{B_{\mathrm{adv}}-1}\bigl(\w_j^\top\r_j - \w_j^\top\bar{\r}\bigr):
    \label{eq:softrank_linear}
\end{equation}
% the slope is common to all objectives and each objective's level is removed, so the adversary targets the worst case of the centered return $J_\w - \w^\top\bar{\r}$.
the slope is common to all objectives and each objective's level is removed, so {the feedback in this regime reduces, up to an affine transformation that the baseline and the step size absorb, to the batch-centered scalarized reward $\w_j^\top(\r_j - \bar{\r})$.
The centering is by the batch mean and therefore moves with the sampled preferences, so this characterizes the signal rather than a fixed objective the adversary minimizes.}
When the spread is large compared with $1/\tau$, the statistic approaches the rank itself, which is invariant to any strictly increasing transformation of each objective, the property on which BOtied~\citep{park2024botied} builds its multi-objective criterion.

% Removed: the expected-gradient argument below does not apply to the batch-rank setting, and the claims
% that the gradient points along $\e_m$ and that a monotone transform preserves its direction are not correct.
% Why the rank is taken per objective is visible in the expected adversary gradient, whose $i$-th component is $\bar r_i/S - \alphaB^\top\bar{\r}/S^2$ with $S = \sum_i \alpha_i$: it points along $\e_m$ whenever one $\bar r_m$ dominates in magnitude, and a transform applied to $\w^\top\r$ after scalarization rescales $\gh$ without changing that direction.
% Outside the realizability of Proposition~\ref{prop:pareto_recovery}, where the surrogate leaves the solution unchanged, we therefore claim only that the adversary targets the worst case on this per-objective scale, and not that it coincides with the worst case of~\eqref{eq:objective}.
% {Under the simultaneous realizability of Proposition~\ref{prop:pareto_recovery} the target policy is optimal regardless of the sampling distribution, so this choice of feedback does not change what the policy converges to; outside it, we claim only that the adversary targets the worst case on this per-objective scale, and not that this coincides with the worst case of~\eqref{eq:objective}.}
{Under exact simultaneous realizability (Proposition~\ref{prop:pareto_recovery} with $\epsilon = 0$), a pointwise-optimal policy remains a population optimum under any fixed sampling distribution.
Outside this regime, we interpret~\eqref{eq:soft_rank} only as a heuristic that shifts sampling toward preferences receiving lower batch-relative feedback.}

\end{document}